\documentclass[opre,nonblindrev]{informs3opt}

\usepackage{endnotes}
\let\footnote=\endnote

\usepackage{natbib}
\usepackage{natbib}
\bibpunct[, ]{(}{)}{,}{a}{}{,}%
\def\bibfont{\small}%

\usepackage{titlesec,soul}
\usepackage[titletoc,toc,title]{appendix}
\usepackage{graphicx}
\usepackage{subfigure}
\usepackage{multirow}

\usepackage[hidelinks,bookmarksnumbered=true]{hyperref}
\usepackage{todonotes,bm}
\usepackage{tikz}
\usepackage{url}
\usepackage{amssymb}
\usepackage{psfrag}
\usepackage{amsmath}
\usepackage{setspace}
\newcommand{\exclude}[1]{}
\usepackage{bm}
\usepackage[flushleft]{threeparttable}
\usepackage{algorithm}
\usepackage{algpseudocode}
\algdef{SE}[DOWHILE]{Do}{doWhile}{\algorithmicdo}[1]{\algorithmicwhile\ #1}%
\algnewcommand{\Or}{\textbf{or}}
\algnewcommand{\And}{\textbf{and}}

\usepackage{thmtools} 
\usepackage{thm-restate}

\usepackage{cleveref}

\declaretheorem[name=Theorem]{theorem}
\declaretheorem[name=Proposition]{proposition}
\declaretheorem[name=Lemma]{lemma}
\declaretheorem[name=Claim]{claim}
\declaretheorem[name=Corollary]{corollary}

\declaretheorem[name=Observation]{observation}

\declaretheorem[name=Definition]{definition}
\declaretheorem[name=Example]{example}

\def\Re{\mathbb{R}}

\def\hat{\widehat}

\def\N{{\mathcal N}}

\def\H{{\mathcal H}}

\def\Re{{\mathbb R}}

\DeclareMathOperator{\supp}{supp}
\DeclareMathOperator{\diag}{diag}
\DeclareMathOperator{\Diag}{Diag}
\DeclareMathOperator{\conv}{conv}
\DeclareMathOperator{\tr}{tr}
\DeclareMathOperator{\col}{col}

\renewcommand*{\qed}{\hfill\ensuremath{\square}}
\newcommand*{\qedA}{\hfill\ensuremath{\diamond}}

\usepackage{enumerate}

\renewcommand{\proof}{\noindent{\em Proof. }}

\usepackage{natbib}
 \bibpunct[, ]{(}{)}{,}{a}{}{,}%
 \def\bibfont{\small}%

\allowdisplaybreaks

\begin{document}
	
	\RUNAUTHOR{Yongchun Li and Weijun Xie}
	
	\RUNTITLE{Approximation Algorithms for the Maximum Entropy Sampling Problem}
	
	\TITLE{Best Principal Submatrix Selection for the Maximum Entropy Sampling Problem:\\ Scalable Algorithms and Performance Guarantees}
	
	\ARTICLEAUTHORS{%
		\AUTHOR{Yongchun Li}
		\AFF{H. Milton Stewart School of Industrial and Systems Engineering, Georgia Institute of Technology, Atlanta, GA, USA, \EMAIL{ycli@gatech.edu}} %
		\AUTHOR{Weijun Xie}
		\AFF{H. Milton Stewart School of Industrial and Systems Engineering, Georgia Institute of Technology,  Atlanta, GA, USA, \EMAIL{wxie@gatech.edu}}
	} %
	
	\ABSTRACT{This paper studies a classic maximum entropy sampling problem (MESP), which aims to select the most informative principal submatrix of a prespecified size from a covariance matrix. %
By investigating its Lagrangian dual and primal characterization, we derive a novel convex integer program for MESP and show that its continuous relaxation yields a near-optimal solution.
 The results motivate us to develop a  sampling algorithm and derive its approximation bound for MESP, which improves the best-known bound in literature. 
We then provide an efficient deterministic implementation of the sampling algorithm with the same approximation bound. 
Besides, we investigate the widely-used local search algorithm and prove its first-known approximation bound for MESP. The proof techniques further inspire us  an efficient implementation of the local search algorithm. Our numerical experiments demonstrate that these approximation algorithms can efficiently solve medium-sized and large-scale instances to near-optimality. 
Finally, we extend the analyses to the A-Optimal MESP (A-MESP), where the objective is to minimize the trace of the inverse of the selected principal submatrix. 
}%

\KEYWORDS{Maximum Entropy Sampling Problem; Convex Integer Program; Sampling Algorithm; Local Search Algorithm; A-Optimality.}

	\maketitle
	\newpage
\section{Introduction}
The maximum entropy sampling problem (MESP) is a classic problem in statistics and information theory \citep{gilmore1996maximum,jaynes1957information,shewry1987maximum}, which aims to select a small number of random observations from a possibly large set of candidates to maximize the information obtained. The MESP has been widely applied to healthcare \citep{alarifi2019optimal}, power system \citep{li2012information}, manufacturing \citep{wang2019optimal}, data science \citep{charikar2000combinatorial,song2010new,zilly2017glaucoma}, among others.
Specifically, suppose that the $n$ random variables follow a Gaussian distribution and their covariance matrix $\bm{C}\in \Re^{n \times n }$ has a rank $d\leq n$. Then, the goal of MESP is to seek a size-$s$ ($s\le d$) principal submatrix of $\bm{C}$ with the largest logarithm of the determinant, i.e., MESP can be formulated as
\begin{align}
\text{(MESP)} \quad z^* :=  \max_{S} \left\{ \log \det (\bm{C}_{S,S}):  S \subseteq [n] , |S|=s \right\}, \label{mesp}
\end{align}
where $\bm{C}_{S,S}$ denotes an $s\times s$ principal submatrix of $\bm{C}$ with rows and columns from set $S$ and $[n]:=\{1, \cdots, n\}$. 
Note that (i) MESP \eqref{mesp} can be generalized to the case that the observations follow multivariate elliptical distributions (see, e.g., \citealt{arellano2013shannon});
(ii) if $s>d$, then the optimal value of MESP \eqref{mesp} is $-\infty$, which is not interesting. Thus, this paper focuses on the non-trivial setting $s\le d$; (iii) when the true covariance matrix $\bm C$ was not known, one would use the sample covariance matrix. We show that the theoretical absolute difference between the optimal value of the true MESP and that of the sample MESP is at most proportional to one over the square root of sample size, which decays polynomially as the sample size increases (see detailed derivations in Appendix \ref{app:estimated}); and (iv) if we only know the mean and the covariance of the random observations, then the formulation \eqref{mesp} is equivalent to the distributionally optimistic counterpart of the MESP. That is, 
the joint Gaussian distribution achieves the largest entropy among all the probability distributions with the same mean and covariance \citep{cover2012elements}. 
Thus, MESP \eqref{mesp} is indeed a very general model and covers many interesting cases.

\subsection{Relevant  Literature}
We review the relevant literature on three aspects: applications, relaxation bounds of MESP, and exact and approximation algorithms.\par

\noindent{\textbf{Applications:}} MESP dates back to \cite{shewry1987maximum} and has been applied to many different areas. One typical application of MESP is the sensor placement \citep{christodoulou2015smarting,bueso1998state}. Due to a limited budget, it is desirable to place a small number of sensors to effectively monitor spatial and temporal phenomena, including temperature, humidity, air pollution, etc. Recently, it has been applied to water quality monitoring \citep{o2010experiences}. 
MESP has also played an important role in machine learning and data science, such as feature selection \citep{charikar2000combinatorial,song2010new}, compressive sensing \citep{hoch2014nonuniform,schmieder1993application}, and image sampling \citep{rigau2003entropy,zilly2017glaucoma}.

\noindent{\textbf{Relaxation Bounds of MESP:}} It has been proven in \cite{ko1995exact} that solving MESP in general is NP-hard. 
Hence, many efforts have been made to explore its strong relaxation bounds (see, e.g., \citealt{anstreicher1996continuous,anstreicher1999using, anstreicher2004masked, anstreicher2020efficient,anstreicher2018maximum,ko1995exact,hoffman2001new,lee1998constrained,lee2003linear,anstreicher2018maximum}). For example, \cite{ko1995exact} used the eigenvalue interlacing property of symmetric matrices to derive an upper bound for MESP. Recent progress by \cite{anstreicher2020efficient} proposed a new upper bound, referred to as linx bound, and numerically showed that it dominated other bounds studied in the literature. In this paper, we derive a Lagrangian dual bound for MESP and also numerically demonstrate that 
this new upper bound can be stronger than the linx bound for some numerical cases (see our numerical study in Section \ref{sec:comp}).

\noindent{\textbf{Exact and Approximation Algorithms:}}
Besides providing stronger upper bounds, researchers have also attempted to propose exact or approximation algorithms to solve MESP \eqref{mesp}. \cite{ko1995exact} was one of the first works to develop a branch and bound (B\&B) algorithm for solving MESP to optimality. Similar works can be found in \cite{anstreicher1999using,anstreicher2020efficient,anstreicher2018maximum,burer2007solving} by integrating stronger upper bounds with the B\&B algorithm. In this paper, we provide an equivalent convex integer program for MESP, which is suitable for a branch and cut (B\&C) algorithm.

However, exact algorithms might not be able to solve very large-scale instances. It has been shown in \cite{anstreicher2020efficient} that solving MESP \eqref{mesp} on the instance of $n=90$ to optimality can take as long as several days. As alternative ones, approximation algorithms have also attracted much attention. Many approximation algorithms such as greedy and exchange (i.e., local search) heuristics have been used to generate high-quality solutions for MESP in literature \citep{ko1995exact,sharma2015greedy}. However, theoretical performance guarantees of these approximation algorithms are rarely known. Although the objective function of MESP \eqref{mesp} is submodular \citep{kelmans1983multiplicative}, it is neither monotonic nor always nonnegative. Thus, existing results on maximizing the nondecreasing and nonnegative submodular function over a cardinality constraint might not apply and thus require additional assumptions \citep{charikar2000combinatorial,sharma2015greedy}. Recently, \cite{nikolov2015randomized} studied a sampling algorithm for the maximum $s$-subdeterminant problem, which can be reduced to MESP \eqref{mesp}, and developed its approximation guarantee. The inapproximability of MESP \eqref{mesp} can be found in \cite{civril2013exponential,summa2014largest}, which shows that unless P=NP, it is impossible to approximate MESP within an additive error $s\log(c)$ for some constant $c>1$. This paper proposes a different sampling algorithm from the one in \cite{nikolov2015randomized} and improves its approximation bound. We also analyze the well-known local search algorithm and derive its first-known approximation guarantee. Both proposed algorithms yield better approximation bound than the greedy algorithm studied in \cite{ccivril2009selecting}. \Cref{table:mespbounds} summarizes the existing approximation bounds in literature and our proposed ones for MESP \eqref{mesp}. Note that approximation Bound is defined as the difference between the optimal value and the output value returned by an algorithm.

\begin{table}[h]
	\centering
	\caption{Summary of Approximation Algorithms for MESP} 
		\label{table:mespbounds}
		\setlength{\tabcolsep}{3pt}\renewcommand{\arraystretch}{1.6}
		\begin{tabular}{| c | c |c |}
			\hline
			&\multicolumn{1}{c|}{Approximation Algorithm} & \multicolumn{1}{c|}{Approximation Bound\tnote{1}}     \\  
			\hline
		\multirow{2}{4.5em}{\textbf{Literature}}&	Greedy \citep{ccivril2009selecting} & $2\log(s!)$  \\
		\cline{2-3}  
			&Samping \citep{nikolov2015randomized} & $s\log (s)-\log(s!) $  \\
			\hline
			\multirow{2}{5.5em}{\textbf{This paper}}	&Sampling  \Cref{alg_rand_round}& $s\log(s) + \log (\binom{n}{s}) - s\log (n) $  \\
			\cline{2-3}
			&Local Search \Cref{algo} & $s \min\left\{ \log(s), \log (n-s+2-n/s) \right\}$ \\
			\hline
		\end{tabular}%
\end{table}

\subsection{Summary of Contributions}
The objective of this paper is to develop a new convex integer program for MESP \eqref{mesp}, analyze  approximation algorithms, and develop their efficient implementations. Below is a summary of our main contributions:
\begin{enumerate}[(i)]
	\item Through the Lagrangian dual of MESP \eqref{mesp} and its primal characterization, we derive a convex integer program for MESP \eqref{mesp} and show that its continuous relaxation solution is near-optimal. In addition, we apply the efficient Frank-Wolfe algorithm to solving the continuous relaxation and derive its rate of convergence.%
	\item The continuous relaxation of the proposed convex integer program inspires us a  sampling algorithm and develop its approximation bound for MESP \eqref{mesp}, which improves the best-known bound in literature. We then provide an efficient deterministic implementation of the proposed sampling algorithm with the same approximation bound.
	\item %
	Using the weak duality between the proposed convex integer program and its  Lagrangian dual, 
	 we investigate the widely-used local search algorithm and prove its first-known approximation bound for MESP \eqref{mesp}. The proof techniques further motivate us to develop an efficient implementation of the local search algorithm.
	\item Our numerical experiments demonstrate that these approximation algorithms can efficiently solve medium-sized and large-scale instances to near-optimality. 
	\item Finally, we extend the analyses to the A-Optimal MESP (A-MESP), where its objective is to minimize the trace of the inverse of the selected principal submatrix. We propose a new convex integer program for A-MESP, study volume sampling and local search algorithms, and prove their approximation ratios.
\end{enumerate}

\noindent \textit{Organization:} The remainder of the paper is organized as follows. %
Section \ref{sec:cip} derives an equivalent convex integer program for MESP. Section \ref{sec:samp} develops the sampling algorithm and its deterministic implementation and also explores their approximation guarantees for MESP. Section \ref{sec:loc} investigates the local search algorithm and proves its approximation guarantee for MESP. Section \ref{sec:comp} conducts a numerical study to demonstrate the efficiency and the solution quality of our proposed approximation algorithms. Section \ref{sec:amesp} extends the analyses to A-MESP. Section \ref{sec:con} concludes the paper.

\noindent \textit{Notation:}	The following notation is used throughout this paper. We use bold lower-case letters (e.g., $\bm{x}$) and bold upper-case letters (e.g., $\bm{X}$) to denote vectors and matrices, respectively, and use corresponding non-bold letters (e.g., $x_i$) to denote their components. {We use $\bm{0}$ to denote the zero vector.} We let $\Re^n_+$ denote the set of all the $n$ dimensional nonnegative vectors and let $\Re^n_{++}$ denote the set of all the $n$ dimensional positive vectors. Given an integer $n$, we let $[n]:=\{1,2,\cdots, n\}$ and let $[s,n]:=\{s,s+1,\cdots, n\}$. We let $\bm{I}_n$ denote the $n \times n$ identity matrix and let $\bm{e}_i$ denote its $i$-th column. Given a set $S$ and an integer $k$, we let $|S|$ denote its cardinality and let $\binom{S}{k}$ denote the collection of all the size-$k$ subsets out of $S$. Given an $m \times n$ matrix $\bm{A}$ and two sets $S\in [m]$, $T\in [n]$, we let $\bm{A}_{S,T}$ denote a submatrix of $\bm{A}$ with rows and columns indexed by sets $S, T$, respectively, let $\bm{A}_S$ denote a submatrix of $\bm{A}$ with columns from the set $S$, and let $\col(\bm{A})$ denote its column space. Given a vector $\bm{x} \in \Re^n$, we let $\Diag(\bm{x})$ denote the diagonal matrix with diagonal elements $x_1,\cdots, x_n$, and let $\supp(\bm{x})$ denote the support of $\bm{x}$.
Given a  symmetric matrix $\bm{A}$, let $\diag(\bm{A})$ denote the vector of diagonal entries of $\bm{A}$, let $\bm{A}^{\dag}$ denote its pseudo inverse, let $\det (\bm{A})$ denote its determinant, let $\tr (\bm{A})$ denote its trace, and let $\lambda_{\min}(\bm{A}),\lambda_{\max}(\bm{A})$ denote the smallest and largest eigenvalues of $\bm{A}$, respectively. Given a convex set $\mathbb{D}$, we use $\textrm{relint}(\mathbb{D})$ to denote its relative interior. Additional notation will be introduced later as needed.

	\section{Convex Integer Programming Formulation} \label{sec:cip}
	\noindent In this section, we derive the Lagrangian dual (LD) of MESP \eqref{mesp} and its primal characterization (PC), where the latter inspires us a new convex integer programming formulation of MESP \eqref{mesp} by enforcing its variables to be binary. 
	
	\subsection{Mixed Integer Nonlinear Program of MESP}
	To begin with, we first observe that MESP \eqref{mesp} has an equivalent {mixed integer} nonlinear programming formulation using the Cholesky factorization. To do so, for matrix $\bm{C} \succeq 0$, let $\bm{C}=\bm{V}^{\top}\bm{V}$ denote its Cholesky factorization, where $\bm{V} \in \Re^{d\times n}$ and let $\bm{v}_i \in \Re^d$ denote the $i$-th column vector of matrix $\bm{V}$ for each $i \in [n]$. Also, let us define the following two functions, which correspond to the objective function of an alternative reformulation of MESP \eqref{mesp}, and the objective function of its Lagrangian dual.
		\begin{definition} \label{def:det}
			For a $d\times d$ matrix $\bm{X} \succeq 0$ of its eigenvalues $\lambda_1 \ge \cdots \ge \lambda_d \ge 0$, we define
			\begin{enumerate}[(i)]
				\item $\det\limits^s (\bm{X}) := \prod_{i\in [s]}\lambda_{i}$,  				
				\item $\det\limits_s (\bm{X}) := \prod_{i\in [d-s+1,d]}\lambda_{i}$.
			\end{enumerate}
		\end{definition}
Note that for any matrix $\bm{X}$, $\det\limits^s (\bm{X})$ denotes the product of the $s$ largest eigenvalues and $\det\limits_s (\bm{X})$ denotes the product of the $s$ smallest eigenvalues. In fact, the following observation shows that the objective function of MESP \eqref{mesp} can be represented by the function $\det\limits^s (\cdot)$
	\begin{observation}\label{claim:cholesky}
		$\det\left(\bm{C}_{S,S}\right) = \det\limits^s \left (\sum_{i \in S} \bm{v}_i \bm{v}_i^{\top}\right)$.
	\end{observation}
	\begin{proof}
	Note that $\bm{C}_{S,S}=\bm{V}_S^{\top} \bm{V}_S$.
			Suppose matrix $\bm{V}_S^{\top} \bm{V}_S$ has eigenvalues $\lambda_1\ge \cdots \ge \lambda_s \ge 0$, which correspond to the $s$ largest eigenvalues of $\bm{V}_S \bm{V}_S^{\top}$. Therefore, we must have
		\begin{align*}
		& \det\left(\bm{C}_{S,S}\right) =	\det \left (\bm{V}_S^{\top} \bm{V}_S \right) = \prod_{i \in [s]}\lambda_i=\det\limits^s \left (\bm{V}_S \bm{V}_S^{\top} \right) =\det\limits^s \bigg (\sum_{i \in S} \bm{v}_i \bm{v}_i^{\top}\bigg) . 
		\end{align*}
		\qed
	\end{proof}

	Let us introduce the binary variables $\bm{x} \in \{0,1\}^n$ where for each $i \in [n]$, $x_i=1$ if the $i$-th column vector $\bm{v}_i$ is chosen, and 0 otherwise. Then according to \Cref{claim:cholesky},  MESP \eqref{mesp} can be reformulated as
	\begin{align}\label{eq_obj}
	\textrm{(MESP)} \quad z^{*} := \max_{\bm{x}} \Bigg \{ \log \det^s \bigg(\sum_{i\in[n]} x_i \bm{v}_i\bm{v}_i^\top \bigg ):  \sum_{i\in [n]} x_i =s,
	\bm{x} \in \{0,1\}^n \Bigg \} .
	\end{align}	
	
	Note that in this paper, we assume $s\leq d\leq n$. However, it is worth mentioning that when $d \le s \le n$, MESP becomes the well-known D-Optimal design problem, a classic problem in statistics \citep{de1995d,pukelsheim2006optimal}.
	
	The following proposition summarizes the properties of the objective function in MESP \eqref{eq_obj}.
	\begin{restatable}{proposition}{proobj} \label{pro:obj}
		The objective function of MESP \eqref{eq_obj} is (i) discrete-submodular, (ii) non-monotonic, (iii) neither concave nor convex, and (iv) not always nonnegative.
	\end{restatable}
	\begin{proof}
		See Appendix \ref{proof_pro_obj}. \qed
	\end{proof}
	
	The non-monotonicity and possible-negativity of the objective function in \eqref{eq_obj}  imply that the existing approximation results for maximizing monotonic or nonnegative submodular problems \citep{charikar2000combinatorial,sharma2015greedy} are not directly applicable to MESP. The non-concavity motivates us to explore a new equivalent convex integer program of MESP.
	
	\subsection{Lagrangian Dual (LD) of MESP}
	\label{LD problem}
	
	In this subsection, we develop the Lagrangian dual (LD) of MESP \eqref{eq_obj}.
	First, let us introduce an auxiliary matrix $\bm{X}\in \Re^{d\times d}$ and reformulate MESP \eqref{eq_obj} as 
	\begin{align} \label{eq_obje}
	\textrm{(MESP)} \quad z^{*} := \max_{\bm{x}, \bm{X} {\succeq} 0} \Bigg \{ \log \det^s (\bm{X} ):  \sum_{i \in [n]} x_i \bm{v}_i \bm{v}_i^{\top} \succeq \bm{X}, \sum_{i\in [n]} x_i =s,
	\bm{x} \in \{0,1\}^n \Bigg \} .
	\end{align}
	By dualizing the constraint $\sum_{i \in [n]} x_i \bm{v}_i \bm{v}_i^{\top} \succeq \bm{X}$, we can obtain the LD of MESP \eqref{eq_obje}. Before deriving the LD formulation, we would like to establish the convex conjugate of the objective function in MESP \eqref{eq_obje}.  	
	\begin{restatable}{lemma}{lemldmax} \label{lem:ldmax} 
		For a $d \times d$ matrix $\bm{\Lambda} \succ 0$, we have
		\begin{align} \label{ldmax}
		\max_{\bm{X}\succeq 0} \left \{ \log \det^s (\bm{X}) -\tr (\bm{X} \bm{\Lambda}) \right \} = -\log \det_s (\bm{\Lambda}) -s,
		\end{align}
		where function $\underset{s}{\det}(\cdot)$ is defined in \Cref{def:det}.
	\end{restatable}
	\begin{proof}
		See Appendix \ref{proof_lem_ldmax}. \qed
	\end{proof}
	
 Using the result in \Cref{lem:ldmax}, we are able to {show the Lagrangian dual formulation} of MESP.
	\begin{restatable}{theorem}{lagrange} \label{lagrange}
		The optimization problem below is the Lagrangian dual of MESP \eqref{eq_obje}
		\begin{align}
		\textrm{(LD)} \quad z^{LD}:= \min_{\bm{\Lambda}\succeq 0, \nu, \bm{\mu} \in \Re_{+}^n} \bigg \{ -\log \det_s (\bm{\Lambda}) +s\nu+\sum_{i \in [n]} \mu_i-s :  \nu+\mu_i \ge \bm{v}_i^{\top} \bm{\Lambda} \bm{v}_i , \forall i \in [n] \bigg \},\label{eq_dual}
		\end{align}	
		and its optimal value provides an upper bound of MESP, i.e., $z^{LD} \ge z^*$.
		
	\end{restatable}
	\begin{proof}We let $\bm{\Lambda} \succ 0$ denote the Lagrange multiplier associated with the constraint $\sum_{i \in [n]} x_i \bm{v}_i \bm{v}_i^{\top} \succeq \bm{X}$ in MESP \eqref{eq_obje}.
		Thus, the resulting dual problem is
		\begin{align}
		z^{LD}:=\min_{\bm{\Lambda}\succ 0}\Bigg\{ \max_{\bm{x}, \bm{X}\succeq 0} \bigg \{ \log \det^s \left (\bm{X} \right ) -\tr(\bm{X}\bm{\Lambda}) + \sum_{i \in [n]} x_i \bm{v}_i^{\top} \bm{\Lambda}\bm{v}_i:  \sum_{i\in [n]} x_i = s, \bm{x} \in \{0,1\}^n \bigg \}\Bigg\}. \label{eq_duale}
		\end{align}
		Note that the inner maximization problem above can be separated into two parts: (i) maximization over {$\bm{X}$} and (ii) maximization over $\bm{x}$.
		\begin{enumerate}[(i)]
			\item For the maximization over $\bm{X} $, applying the identity in \Cref{lem:ldmax}, we have
			\begin{align*}
			\max_{ \bm{X}\succeq 0} \left \{ \log \det^s \left (\bm{X} \right ) -\tr(\bm{X}\bm{\Lambda}) \right \}= -\log \det_s (\bm{\Lambda}) -s .
			\end{align*}
			
			\item For the maximization over $\bm{x}$, it is known that optimizing a linear function over a cardinality constraint is equivalent to its continuous relaxation, which leads to that
			\begin{align*}
			\max_{\bm{x}} \bigg \{ \sum_{i \in [n]} x_i \bm{v}_i^{\top} \bm{\Lambda}\bm{v}_i:  \sum_{i\in [n]} x_i = s, \bm{x} \in \{0,1\}^n \bigg \}=\min_{\nu, \bm{\mu}\in \Re^{n}_+} \bigg \{ s\nu+ \sum_{i \in [n]} \mu_i: \nu+\mu_i \ge \bm{v}_i^{\top} \bm{\Lambda} \bm{v}_i, \forall i \in [n] \bigg \},
			\end{align*}
			where the right-hand side is the dual of the continuous relaxation of the left-hand side.
		\end{enumerate}
		Plugging the above results (i.e., Parts (i) and (ii)) into the dual problem \eqref{eq_duale} and combining the minimization problems over $(\bm{\Lambda}, \nu, \bm{\mu})$ together, we arrive at \eqref{eq_dual}.
		
		Further, the inequality $z^*\leq z^{LD}$ holds due to the weak duality. \qed
		
	\end{proof}
	
	\subsection{Primal Characterization (PC) of LD and Convex Integer Program of MESP}
	In this subsection, we  show the primal characterization (PC) of LD \eqref{eq_dual}, which inspires us an equivalent convex integer program of MESP \eqref{eq_obj}.
	
	According to the standard result (see, e.g., \citealt{bertsekas1982constrained,lemarechal2001geometric}) on a primal characterization of the Lagrangian dual, we have
	\begin{align*} %
	\textrm{(PC)} \quad z^{LD} := \max_{w,\bm{x}, \bm{X} \succ 0} \Bigg \{w: &\sum_{i \in [n]} x_i \bm{v}_i \bm{v}_i^{\top} \succeq \bm{X}, \notag\\
	&(w,\bm{x}, \bm{X})\in \conv\bigg\{(w,\bm{x}, \bm{X}): w\leq \log \det^s (\bm{X} ),  \sum_{i\in [n]} x_i =s,
	\bm{x} \in \{0,1\}^n\bigg\} \Bigg \} .
	\end{align*}
In general, the convex hull is difficult to obtain, and thus alternatively, we derive the primal characterization through the dual formulation of LD \eqref{eq_dual}.
	
	The primal characterization relies on the following results. {First, for any given $\bm{\lambda}\in \Re^d$, let us define a unique integer $k$ based on its sorted elements as below.}
	
	\begin{lemma}[lemma 14, \citealt{nikolov2015randomized}] \label{lem_k}
		Given a vector $\bm{\lambda} \in \Re^d$ with its elements sorted by $\lambda_1 \ge \cdots \ge \lambda_d$ and an integer $s \in [d]$, there exists a unique integer $0\le k < s$ such that $\lambda_{k} > \frac{1}{s-k} \sum_{i\in[k+1,d]} \lambda_{i} \ge \lambda_{k+1}$, where by convention $\lambda_{0}=\infty$.
		\label{lem:kappa}
	\end{lemma}

	Throughout this paper, we use $k$ to denote the unique integer in \Cref{lem_k}. Next, we define the objective function of the primal characterization below, which can be also found in \cite{nikolov2015randomized}.
	
\begin{definition}\label{def:objPC}
	For a $d\times d$ matrix $\bm{X} \succeq 0$ with its eigenvalues $\lambda_1 \ge \cdots \ge \lambda_d \ge 0$, let us denote
	$$\Gamma_s(\bm{X}) := \log \bigg (\prod_{i\in [k]} \lambda_{i}\bigg) + (s-k) \log \bigg(\frac{1}{s-k} \sum_{i\in [k+1,d]} \lambda_{i} \bigg), $$ 
where the unique integer $k$ is defined in Lemma~\ref{lem:kappa}.
\end{definition}

We are now ready to derive the convex conjugate of the objective function in LD \eqref{eq_dual}.
		\begin{restatable}{lemma}{lemldmaxtwo} \label{lem:ldmax2} 
			Given a $d \times d$ matrix $\bm{X} \succeq 0$ with rank $r\in [s,d]$, suppose that the eigenvalues of $\bm{X}$ are $\lambda_1 \ge \cdots \ge \lambda_r > \lambda_{r+1}=\cdots = \lambda_d =0$ and $\bm{X} =\bm{Q}\Diag(\bm{\lambda})\bm{Q}^{\top}$ with an orthonormal matrix $\bm{Q}$. Then
			\begin{enumerate}[(i)]
				\item \begin{align} 
				\min_{\bm{\Lambda} \succ 0} \left \{-\log \det_s (\bm{\Lambda}) + \tr (\bm{X} \bm{\Lambda})\right \} = \min_{\begin{subarray}{c}\bm{\beta}\in \Re_{+}^d,\\
					0<\beta_1 \le \cdots \le \beta_d
					\end{subarray}}\Bigg\{ -\sum_{i\in [s]}\log (\beta_{i}) + \sum_{i\in [d]} \lambda_{i} \beta_{i}  \Bigg \}, \label{ldmin}
				\end{align}
				\item \begin{align}
				\min_{\begin{subarray}{c}\bm{\beta}\in \Re_{+}^d,\\
					0<\beta_1 \le \cdots \le \beta_d
					\end{subarray}}\Bigg\{ -\sum_{i\in [s]}\log (\beta_{i}) + \sum_{i\in [d]} \lambda_{i} \beta_{i}  \Bigg\} = \Gamma_s (\bm{X})+s . \label{ldmin1}
				\end{align}
			\end{enumerate}
		\end{restatable}
	\begin{proof}
		See Appendix \ref{proof_lem_ldmax2}.\qed
	\end{proof}
	
With the convex conjugate of the objective function in LD \eqref{eq_dual}, using the Lagrangian dual method, we are able to derive its dual problem and also show the primal characterization below.
	\begin{restatable}{theorem}{primal} \label{primal_cont}
		LD \eqref{eq_dual} has the following primal characterization, i.e., 
		\begin{align}\label{eq_pcont}
		\textrm{(PC)} \quad z^{LD} := \max_{\bm{x}} \Bigg\{ \Gamma_s\bigg(\sum_{i\in [n]}x_i \bm{v}_i\bm{v}_i^{\top}\bigg ):  \sum_{i\in [n]} x_i =s,
		\bm{x} \in [0,1]^n \Bigg \},
		\end{align}	
		where function $\Gamma_s(\cdot)$ can be found in \Cref{def:objPC}.
	\end{restatable}
	
	\begin{proof}
	In LD \eqref{eq_dual}, let us introduce Lagrangian multiplies $\bm{x}$ associated with the constraints. Since $z^{LD}\geq z^*$ and the constraint system of LD \eqref{eq_dual} satisfies the relaxed Slater condition, according to theorem 3.2.2 in \cite{ben2012optimization}, the strong duality holds, i.e.,
		\begin{align*}
		z^{LD} := \max_{\bm{x}\in \Re_+^n} \Bigg \{ \min_{\bm{\Lambda}\succ 0, \nu , \bm{\mu} \in \Re_{+}^n } \bigg \{-\log \det_s (\bm{\Lambda}) +s\nu+\sum_{i \in [n]}\mu_i-s + \sum_{i \in [n]} x_i (\bm{v}_i^{\top} \bm{\Lambda} \bm{v}_i- \nu-\mu_i ) \bigg \}\Bigg\}.
		\end{align*}
		The inner minimization above can be separated into two parts: (i) minimization over $\bm{\Lambda} $ and (ii) minimization over $(\nu,\bm{\mu})$, which are discussed below.
		\begin{enumerate}[(i)]
			\item Let $\bm{X}= \sum_{i \in [n]} x_i \bm{v}_i \bm{v}_i^{\top}$. For the minimization over $\bm{\Lambda} $, applying the identities \eqref{ldmin} and \eqref{ldmin1} in Lemma \ref{lem:ldmax2} and using the fact that $\sum_{i \in [n]} x_i \bm{v}_i^{\top} \bm{\Lambda} \bm{v}_i=\tr (\bm{X} \bm{\Lambda})$, we have
			\begin{align*}
			\min_{\bm{\Lambda} \succ 0} \left \{-\log \det_s (\bm{\Lambda}) + \tr (\bm{X} \bm{\Lambda})\right \}-s = \Gamma_s (\bm{X}).
			\end{align*}

			\item For the minimization over $(\nu,\bm{\mu})$, we have
			\begin{align*}
			\min_{\nu , \bm{\mu} \in \Re_+^n } \bigg \{s\nu+\sum_{i \in [n]}\mu_i + \sum_{i \in [n]} x_i (- \nu-\mu_i )\bigg\}=\begin{cases}
			0,&\textrm{ if }\sum_{i \in [n]}x_i=s, x_i \leq 1, \forall i\in [n];\\
			-\infty, &\textrm{ otherwise}.
			\end{cases}.
			\end{align*}
		\end{enumerate}
		Putting the above two pieces together, we arrive at \eqref{eq_pcont}. \qed
	\end{proof}

	We remark that PC \eqref{eq_pcont} has the same objective function as another convex relaxation proposed by \cite{nikolov2015randomized}, but we distinguish our formulation from \cite{nikolov2015randomized}'s in the following three aspects: (i) We derive the primal characterization from a Lagrangian dual perspective, which is also applicable to the A-Optimality (see Section \ref{sec:amesp}) and enables us to derive supdifferentials of the objective function; (ii) Our PC \eqref{eq_pcont} can be stronger than the one in \cite{nikolov2015randomized} due to the extra constraints $x_i\leq 1$ for each $i\in [n]$; and (iii) LD \eqref{eq_dual} and PC \eqref{eq_pcont} together are critical to the analysis of the local search algorithm in Section \ref{sec:loc}.

	The PC \eqref{eq_pcont} is a concave maximization problem and is efficiently solvable. In the next subsection, we introduce the Frank-Wolfe algorithm to solve it. However, according to \Cref{def:objPC}, the objective function $\Gamma_s(\cdot)$ might not be differentiable. Fortunately, the following result shows how to derive its supdifferentials.
	\begin{restatable}{proposition}{subdiff} \label{supdiff}
		Given a $d \times d$ matrix $\bm{X} \succeq 0$ with rank $r\in [s,d]$, suppose that its eigenvalues are $\lambda_1 \ge \cdots \ge \lambda_r > \lambda_{r+1}=\cdots = \lambda_d =0$ and ${\bm{X}} = \bm{Q} \Diag(\bm{\lambda}) \bm{Q}^{\top}$ with an orthonormal matrix $\bm{Q}$. Then the supdifferential of the function $\Gamma_s(\cdot)$ at $\bm{X}$ that is denoted by $\partial \Gamma_s(\bm{X})$ is 
		\begin{align*}
		&\partial \Gamma_s(\bm{X}) =\Bigg\{\bm{Q} \Diag(\bm{\beta}) \bm{Q}^{\top}: \bm{X} = \bm{Q} \Diag(\bm{\lambda}) \bm{Q}^{\top}, \bm{Q} \textrm{\rm\ is orthonormal}, \lambda_1 \ge \cdots \ge \lambda_d,\\
		&\bm{\beta} \in \conv \bigg \{\bm{\beta}: \beta_i = \frac{1}{\lambda_i}, \forall i \in [k], \beta_i = \frac{s-k}{\sum_{i\in [k+1,d]} \lambda_{i}}, \forall i \in [k+1, r],\beta_i \ge \beta_r, \forall i \in [r+1,d] \bigg \} \Bigg\},
		\end{align*}
		where the unique integer $k$ follows from \Cref{lem:kappa}.
		Note that the function $\Gamma_s(\cdot)$ is differentiable whenever $\bm{X}$ is a positive-definite matrix and the unique supgradient becomes the gradient.
	\end{restatable}
	\begin{proof}
First, let us define a function $\gamma_s(\cdot) $ as below
		\begin{align}
		\gamma_s(\bm{\lambda}) :
		=	\min_{\begin{subarray}{c}\bm{\beta}\in \Re_{+}^d,\\
			0<\beta_1 \le \cdots \le \beta_d
			\end{subarray}} \Bigg\{ -\sum_{i\in [s]}\log (\beta_{i}) + \sum_{i\in [d]} \lambda_{i} \beta_{i}  \Bigg \}= \Gamma_s(\bm{X})+s,\label{ldmineq}
		\end{align}
		where the equation stems from the identity \eqref{ldmin1} in Lemma~\ref{lem:ldmax2}.
		
		Since function $\Gamma_s(\bm{X})$ is invariant under all the permutations of its eigenvalues, according to corollary 2.5 in \cite{lewis1995convex}, we have that
		\begin{align*}
		\partial \Gamma_s(\bm{X})=\left\{\bm{Q} \Diag(\bm{\beta}) \bm{Q}^{\top}: \bm{X} = \bm{Q} \Diag(\bm{\lambda}) \bm{Q}^{\top}, \bm{Q} \textrm{\rm\ is orthonormal}, \bm{\beta} \in \partial \gamma_s(\bm{\lambda})\right\}.
		\end{align*}
		
		Further, by corollary 23.5.3 in \cite{rockafellar1970convex}, the supdifferential of the concave function $\gamma_s(\bm{\lambda})$ is the convex hull of all the optimal solutions $\bm{\beta}^*$ of the minimization problem in \eqref{ldmineq}. From the proof of Lemma \ref{lem:ldmax2}, any optimal solution $\bm{\beta}^*$ satisfies
		\begin{align*}
		\beta_{i}^* = \frac{1}{\lambda_{i}}, \forall i \in [k], \beta_{i}^* = \frac{s-k}{\sum_{i\in [k+1,d]} \lambda_{i}}, \forall i\in [k+1,r], \beta_{i}^* \ge \beta_{r}^*, \forall i\in [r+1,d].
		\end{align*}	
		
		Hence, the supdifferential of function $\gamma_s(\bm{\lambda})$ at $\bm{\lambda}$ is
		\begin{align*}
		\partial \gamma_s(\bm{\lambda}) = \conv \bigg\{\bm{\beta}: \beta_{i} = \frac{1}{\lambda_{i}}, \forall i \in [k],  \beta_{i} = \frac{s-k}{\sum_{i\in [k+1,d]} \lambda_{i}},\forall i\in [k+1,r],  \beta_{i} \ge \beta_r,\forall i \in [r+1, d] \bigg\}.
		\end{align*}
		This completes the proof.\qed
	\end{proof}
	
	As a side product of PC \eqref{eq_pcont}, we observe that if we enforce its variables $\bm{x}$ to be binary, we can arrive at an equivalent convex integer program for MESP.

	\begin{restatable}{theorem}{primal} \label{primal}
		MESP can be formulated as the following convex integer program
		\begin{align}\label{eq_p}
		\textrm{(MESP)} \quad z^{*} := \max_{\bm{x}} \Bigg \{ \Gamma_s\bigg(\sum_{i\in [n]}x_i \bm{v}_i\bm{v}_i^{\top}\bigg):  \sum_{i\in [n]} x_i =s,
		\bm{x} \in \{0,1\}^n \Bigg \}.
		\end{align}
	\end{restatable}
	\begin{proof}
		See Appendix \ref{proof_primal}. \qed
	\end{proof}
	
	We close this subsection by showing that under three special cases, the optimal value of PC \eqref{eq_pcont} is equal to that of MESP, i.e.,  $z^{LD} = z^*$.
	
	\begin{restatable}{proposition}{propmesp}\label{prop:mesp}
		The optimal value of PC \eqref{eq_pcont} is equal to $z^*$, i.e., $z^{LD}=z^*$ provided the following three special cases: (i) $\bm{C}$ is diagonal; (ii) $s=1$; and (iii) $s=n$. 
	\end{restatable}
	
	\proof
	See Appendix \ref{proof_prop_mesp}. \qed
	\endproof
	The results above demonstrate that the optimal value of the proposed PC \eqref{eq_pcont} can be close to that of MESP. We  further numerically verify this property of PC \eqref{eq_pcont} in Section \ref{sec:comp}.
	
	\section{Frank-Wolfe Algorithm, Sampling Algorithm, and its Deterministic Implementation}
	\label{sec:samp}
	In this section, we apply the Frank-Wolfe algorithm to solving PC \eqref{eq_pcont} and derive its convergence rate. We also study a randomized sampling algorithm for MESP and prove its approximation bound, which admits a deterministic implementation  with the same performance guarantee.

	\subsection{Solving PC \eqref{eq_pcont} using Frank-Wolfe Algorithm} \label{sec3.2:fw}
	In this subsection, we  investigate the Frank-Wolfe algorithm for solving PC \eqref{eq_pcont}. We define a feasible solution $\bm{\hat x}$ to be an $\alpha$-optimal solution to PC \eqref{eq_pcont} if the inequality $\Gamma_s(\sum_{i \in [n]} \hat{x}_i \bm{v}_i \bm{v}_i^{\top} ) \ge z^{LD} - \alpha$ with $\alpha \in (0, \infty)$. Given a target accuracy $\alpha$, our proposed Frank-Wolfe algorithm returns an $\alpha$-optimal solution to PC \eqref{eq_pcont}.
	
	The proposed Frank-Wolfe algorithm proceeds as follows. We denote PC \eqref{eq_pcont} to be the primal problem and LD \eqref{eq_dual} to be the dual problem. At each iteration $t$, we set the step size $\epsilon_t := \frac{2}{t+2}$. For the current feasible primal solution $\bm{x}^t$, we let $\bm{X}^t = \sum_{i\in [n]} {x}^t_i \bm{v}_i\bm{v}_i^{\top}$ and then compute the eigendecomposition of matrix $\bm{X}^t$ with eigenvalues $\lambda_1 \ge \cdots \ge \lambda_d$ and an orthonormal matrix $ \bm{Q}$ such that $\bm{X}^t = \bm{Q}
	\Diag(\bm{\lambda}) \bm{Q}^{\top}$. Next, we compute the integer $k$ according to \Cref{lem:kappa} and construct a new vector $\bm\beta^t\in \Re_+^d$ as
	\[\beta_{i}^t = \frac{1}{\lambda_{i}}, \forall i \in [k],  \beta_{i}^t = \frac{s-k}{\sum_{i\in [k+1,d]} \lambda_{i}}, \forall i\in [k+1, d].\]
	Thus, let us denote the dual variable by $\bm{\Lambda}^t= \bm{Q}	\Diag(\bm{\beta}^t) \bm{Q}^{\top}$, which is also a supgradient of function $\Gamma_s(\cdot)$ at $\bm{X}^t$ according to \Cref{supdiff}. Then we obtain the other two dual variables $(\nu^t, \bm{\mu}^t)$ of LD \eqref{eq_dual} by solving the following minimization problem with a closed-form optimal solution:
	\[(\nu^t, \bm{\mu}^t):= \argmin_{ \nu, \bm{\mu} \in \Re_{+}^n} \bigg\{s\nu+\sum_{i \in [n]} \mu_i-s :  \nu+\mu_i \ge \bm{v}_i^{\top}\bm{\Lambda}^t \bm{v}_i , \forall i \in [n]  \bigg\},\]
	i.e., suppose that $\bm\sigma$ is a permutation of $[n]$ such that $\bm{v}_{\sigma(1)}^{\top}\bm{\Lambda}^t \bm{v}_{\sigma(1)}\geq \cdots\geq \bm{v}_{\sigma(n)}^{\top}\bm{\Lambda}^t \bm{v}_{\sigma(n)}$, then
	\[\nu^t=\bm{v}_{\sigma(s)}^{\top}\bm{\Lambda}^t \bm{v}_{\sigma(s)}, \mu_{\sigma(i)}^t=\begin{cases}
	\bm{v}_{\sigma(i)}^{\top}\bm{\Lambda}^t \bm{v}_{\sigma(i)}-\bm{v}_{\sigma(s)}^{\top}\bm{\Lambda}^t \bm{v}_{\sigma(s)}, &\textrm{  }\forall i\in [s];\\
	0, &\textrm{ }\forall i\in [s+1,n].
	\end{cases}.\]%
	According to \Cref{lem:ldmax2}, the construction of $\bm{\Lambda}^t$ implies that  $\Gamma_s(\bm{X}^t )=-\log\underset{s}{\det} (\bm{\Lambda}^t)$. Thus, the duality gap at current iteration only relies on $s\nu^t+\sum_{i \in [n]} \mu_i^t-s.$ We check if the smallest duality gap is less than the threshold $\alpha$ or not. If ``Yes", then we terminate the algorithm. Otherwise, we keep on running the algorithm by: (i) deriving the supgradient of PC \eqref{eq_pcont} at the current solution $\bm{x}^t$, which is $\bm{g}^t := (\bm{v}_1^{\top} \bm{\Lambda}^t \bm{v}_1, \cdots, \bm{v}_n^{\top} \bm{\Lambda}^t \bm{v}_n )^{\top}$; (ii) computing the incumbent solution $\bm{\hat x}^t := \argmax_{\bm{x}} \{(\bm{g}^t)^{\top} \bm{x}: \sum_{i \in [n]} x_i=s, \bm{x} \in [0,1]^n  \}$, i.e.,
	\[\hat{x}_{\sigma(i)}^t=\begin{cases}
	1, &\textrm{ } \forall i\in [s];\\
	0, &\textrm{ } \forall i\in [s+1,n].
	\end{cases};\]
	and (iii) updating the solution $\bm{x}^{t+1} := \epsilon_t \bm{\hat x}^t + (1-\epsilon_t) \bm{x}^t$. The detailed implementation can be found in Algorithm \ref{algo:fw}. %
	
	\begin{algorithm}[ht]
		\caption{Frank-Wolfe Algorithm} \label{algo:fw}
		\begin{algorithmic}[1]
			\State \textbf{Input:} $n\times n$ matrix $\bm{C} \succeq 0$ of rank $d$, integer $s \in [d]$, and target accuracy $\alpha \in (0, \infty)$
			\State Let $\bm{C}=\bm{V}^{\top}\bm{V}$ denote its Cholesky factorization where $\bm{V} \in \Re^{d\times n}$
			\State Let $\bm{v}_i \in \Re^d$ denote the $i$-th column vector of $\bm{V}$ for each $i \in [n]$
			\State Initialize a feasible solution $\bm{x}^0$ of PC \eqref{eq_pcont}, the number of steps $t=0$, and the duality gap $\Delta= \infty$
			\Do		
			\State Let $\epsilon_t := \frac{2}{t+2}$
			\State Let $\bm{X}^t = \sum_{i\in [n]} {x}^t_i \bm{v}_i\bm{v}_i^{\top}$ with eigenvalues $\lambda_1 \ge \cdots \ge \lambda_d$ and compute $\bm{X}^t= \bm{Q}
			\Diag(\bm{\lambda}^t) \bm{Q}^{\top}$
			\State Compute $k$ according to \Cref{lem:kappa}
			\State Compute the new vector $\bm{\beta}$: $\beta_i^t= \frac{1}{\lambda_i}$ for each $i \in [k]$ and $\frac{s-k}{\sum_{i\in [k+1,d]} \lambda_i} $, otherwise
			
			\State Let $\bm{\Lambda}^t= \bm{Q}	\Diag(\bm{\beta}) \bm{Q}^{\top}$ 
			\State Let $\bm\sigma$ be a permutation of $[n]$ such that $\bm{v}_{\sigma(1)}^{\top}\bm{\Lambda}^t \bm{v}_{\sigma(1)}\geq \cdots\geq \bm{v}_{\sigma(n)}^{\top}\bm{\Lambda}^t \bm{v}_{\sigma(n)}$
			\State Let $\nu^t=\bm{v}_{\sigma(s)}^{\top}\bm{\Lambda}^t \bm{v}_{\sigma(s)}, \mu_{\sigma(i)}^t= 		\bm{v}_{\sigma(i)}^{\top}\bm{\Lambda}^t \bm{v}_{\sigma(i)}-\nu^t$ for each $i \in [s]$ and 0, otherwise
			
			\State Let $\hat{x}_{\sigma(i)}^t=1$ for all $i \in [s]$ and $0$, otherwise
			\State Update $\bm{x}^{t+1} := \epsilon_t \bm{\hat x}^t + (1-\epsilon_t) \bm{x}^t$, $\Delta :=  \min \{\Delta, s\nu^t+\sum_{i \in [n]} \mu_i^t-s \}$ and $t:= t+1$
			\doWhile{$\Delta \ge  \alpha$}
			\State \textbf{Output:} $\bm{x}^t$
		\end{algorithmic}
	\end{algorithm}
	
	Compared to the other first-order methods, the Frank-Wolfe Algorithm \ref{algo:fw} is known to deliver a sparse incumbent solution at each iteration \citep{freund2016new}, which allows us to study the size of the support of its output. To begin with, let us introduce the following key lemma.
	
	\begin{restatable}{lemma}{lemhessian}\label{lem_hessian}
		Suppose that for any size-$s$ subset $S\subseteq [n]$, the columns $\{\bm{v}_i\}_{i\in S}$ are linearly independent. %
		Let $\mathbb{D}:=\{\bm{x} \in \Re^n: \sum_{i \in [n]}x_i =s, \bm{x}\in [0,1]^n \}$. Then for any $\bm{x} \in \textrm{relint}(\mathbb{D})$, we have
\begin{align}
		\nabla^2 \Gamma_s \bigg(\sum_{i \in [n]} {x}_i \bm{v}_i \bm{v}_i^{\top} \bigg)\succeq - \frac{ \lambda^2_{\max}(\bm{C})}{\delta^2} \bm{I}_n,
		\label{lips}
		\end{align}
		where the constant $\delta :=\min_{S\subseteq [n], |S|=s} \lambda_{\min}(\bm{C}_{S,S})$.
		
	\end{restatable}
	\begin{proof}
		See Appendix \ref{proof_lem_hessian}. \qed
	\end{proof}

		In \Cref{lem_hessian}, the constant $\delta$ should be positive, which is a mild assumption and could be easily satisfied due to the fact  $s\le d$. Besides, this assumption (i.e., $\delta>0$) is only useful to prove the convergence rate of Frank-Wolfe \Cref{algo:fw}. Therefore, even when $\delta=0$, the proposed Frank-Wolfe \Cref{algo:fw} would still work and our analyses of the proposed approximation algorithms would still follow. In practice, when running the Frank-Wolfe \Cref{algo:fw}, one may want to add a small perturbation (e.g., $\epsilon\bm{I}_n$ with a small but positive $\epsilon$) to the covariance matrix $\bm{C}$ to remedy the singularity.
The inequality in \Cref{lem_hessian} implies that the Hessian of the objective function $\Gamma_s(\cdot)$ of PC \eqref{eq_pcont} is lower bounded. 
	Based upon this result, we are able to derive the rate of convergence of the proposed Frank-Wolfe Algorithm \ref{algo:fw}.
	\begin{theorem}\label{pro:fw}
		Let $\bm{\hat x}$ denote the output of Frank-Wolfe Algorithm \ref{algo:fw}. Suppose that  for any subset $S\subseteq [n]$ with $|S|=s$, the columns $\{\bm{v}_i\}_{i\in S}$ are linearly independent, and  $\bm{\hat x}$ is an $\alpha$-optimal solution of PC \eqref{eq_pcont} for some $\alpha\in (0,\infty)$. Then
\begin{enumerate}[(i)]
			\item The number of iterations is bounded by $t\leq 4\alpha^{-1}L\min\{s,n-s\}$, where $L:=\delta^{-2} \lambda^2_{\max}(\bm{C})$,
			\item The size of support of $\bm{\hat x}$ satisfies $|\supp (\bm{\hat x})| \leq 4\alpha^{-1}Ls\min\{s,n-s\}$.
		\end{enumerate} %
	\end{theorem}
	\begin{proof}
		\noindent{\textbf{Part (i).}} Let $\mathbb{D}:=\{\bm{x}:\sum_{i\in [n]} x_i =s,
		\bm{x} \in [0,1]^n \}$. Since $\Gamma_s(\cdot)$ is continuous in $\mathbb{D}$, thus
		\[z^{LD} := \max_{\bm{x}\in \mathbb{D}} \Bigg\{ \Gamma_s\bigg(\sum_{i\in [n]}x_i \bm{v}_i\bm{v}_i^{\top}\bigg )\Bigg \}:=-\inf_{\bm{x}\in \textrm{relint}(\mathbb{D})} \Bigg\{ -\Gamma_s\bigg(\sum_{i\in [n]}x_i \bm{v}_i\bm{v}_i^{\top}\bigg )\Bigg \}.\]
		Thus, it is equivalent to analyzing the Frank-Wolfe Algorithm \ref{algo:fw} on solving the right-hand side problem. 
		{The inequality \eqref{lips} in \Cref{lem_hessian} indicates that for any $\bm{x}\in \textrm{relint}(\mathbb{D})$, the largest eigenvalue of the Hessian of the convex function $-\Gamma_s(\sum_{i\in [n]}x_i \bm{v}_i\bm{v}_i^{\top} )$ is bounded by $L$.}
		Therefore, the smoothness coefficient of $-\Gamma_s(\sum_{i\in [n]}x_i \bm{v}_i\bm{v}_i^{\top})$ in $ \textrm{relint}(\mathbb{D})$ is at most $L$.
		Given the $L$-smoothness, for Frank-Wolfe Algorithm \ref{algo:fw}, after iteration $t$,  \cite{pedregosa2018step}[theorem 2] showed that the duality gap is bounded by
		\[\frac{2\sup_{\bm{x},\bm{y}\in \textrm{relint}(\mathbb{D})}\|\bm{x}-\bm{y}\|_2^2L}{t+1}=\frac{4L\min\{s,n-s\}}{t+1}.\]
		Given the target of the duality gap to be $\alpha$, it follows that
		\[t\leq 4\alpha^{-1}L\min\{s,n-s\}.\]
		
		\noindent{\textbf{Part (ii).}} Since each iteration of Algorithm \ref{algo:fw} increases at most $s$ nonzero entries for the current solution,  the size of the support of the output solution $\bm{\hat x}$ is bounded by
		\[|\supp (\bm{\hat x})| \leq st\leq   4\alpha^{-1}Ls\min\{s,n-s\}.\]\qed
	\end{proof}
	
	\subsection{Sampling Algorithm}
	In this subsection, we  introduce and analyze a randomized sampling algorithm for MESP. Given an $\alpha$-optimal solution $\bm{\hat x}$ of PC \eqref{eq_pcont} with $\alpha \in (0, \infty)$, our proposed sampling algorithm is to sample a size-$s$ subset $S \subseteq [n]$ with probability
	\begin{align}
	\mathbb{P}[\tilde{S}=S] := \frac{\prod_{i \in S} \hat{x}_i}{ \sum_{\bar{S} \in \binom{[n]}{s} }  \prod_{i \in \bar{S}} \hat{x}_i  } . \label{samprob}
	\end{align}

	\begin{algorithm}[ht]
		\caption{Efficient Implementation of Sampling Procedure \eqref{samprob}}
		\label{alg_rand_round}
		\begin{algorithmic}[1]
			\State \textbf{Input:} $n\times n$ matrix $\bm{C} \succeq 0$ of rank $d$ and integer $s \in [d]$
			\State Let $\bm{\hat x}$ be an $\alpha$-optimal solution of PC \eqref{eq_pcont} with $\alpha \in (0, \infty)$
			\State Initialize chosen set $\tilde{S}=\emptyset$ and unchosen set $T= \emptyset$
			\State Two factors: $A_1=\sum_{S\in {[n]\choose s}}\prod_{i\in {S}}\hat{x}_i,A_2=0$
			\For{$j=1,\cdots,n$}
			\State Let $A_2=\sum_{S\in {[n]\setminus (\tilde S\cup T)\choose s-1-|\tilde{S}|}}\prod_{\tau\in S}\hat{x}_{\tau}$
			\State Sample a $(0,1)$ uniform random variable $U$
			\If{$\hat{x}_jA_2/A_1 \geq U$}
			\State Add $j$ to set $\tilde{S}$
			\State $A_1=A_2$
			\Else
			\State Add $j$ to set $T$
			\State $A_1=A_1-\hat{x}_jA_2$
			\EndIf
			\EndFor
			\State \textbf{Output:} $\tilde{S}$%
		\end{algorithmic}
	\end{algorithm}
	
	The detailed implementation can be found in Algorithm \ref{alg_rand_round}. This sampling procedure is similar to algorithm 1 in \cite{singh2018approximate}, which has been proved to be computationally efficient with running time complexity $O(n \log n)$. The following result helps establish a relationship between the expected objective value using our sampling procedure and the optimal value of PC \eqref{eq_pcont}.

	\begin{lemma}\label{bound}
		Given an $n \times n$ matrix $\bm{X} \succeq 0$ of rank $d$ such that $\bm{X}=\bm{V}^{\top}\bm{V}$ with $\bm{V} \in \Re^{d \times n}$ and a vector $\hat{\bm{x}}\in \Re_+^n$, then we have
		\[\sum_{S \in \binom{[n]}{s} }  \prod_{i \in {S}} \hat{x}_i   \det^s (\bm{V}_{S} \bm{V}_{S}^{\top} )\geq \exp\bigg[\Gamma_s\bigg(\sum_{i \in [n]}  \hat{x}_i \bm{v}_i \bm{v}_i^{\top}\bigg)\bigg]  .\]
	\end{lemma}
	\begin{proof}
		The proof follows from theorem 18 in \cite{nikolov2015randomized} and is thus omitted here. \qed
	\end{proof}

Now we are ready to show the approximation bound of the proposed sampling Algorithm \ref{alg_rand_round}.
	
	\begin{theorem} 		\label{thm:samp}
		Given an $\alpha$-optimal solution $\bm{\hat x}$ of PC \eqref{eq_pcont} with $\alpha \in (0, \infty)$, the random set generated by the sampling Algorithm \ref{alg_rand_round} returns a $(s\log(s) + \log (\binom{n}{s}) - s\log (n)+\alpha)$-approximation bound for MESP {\eqref{eq_obj}}, i.e., suppose the output of Algorithm \ref{alg_rand_round} is the random set $\tilde{S}$, then  
		\begin{align*}
		\log \mathbb{E} \bigg[ \det^s \bigg(\sum_{i \in \tilde{S}}\bm{v}_i \bm{v}_i^{\top} \bigg) \bigg] \ge z^*-s\log(s) - \log \bigg(\binom{n}{s}\bigg) + s\log (n) - \alpha.
		\end{align*} 
	\end{theorem} 

	\begin{proof}
		Given the random set $\tilde{S}$ and its sampling probability \eqref{samprob}, the expected exponential of the objective value of MESP \eqref{eq_obj} is equal to
		\begin{align*}
		\mathbb{E} \bigg[ \det^s \bigg(\sum_{i \in \tilde{S}}\bm{v}_i \bm{v}_i^{\top} \bigg) \bigg] &=\sum_{S \in \binom{[n]}{s} } \mathbb{P}[\tilde{S}=S] \det^s (\bm{V}_{S} \bm{V}_{S}^{\top} ) 
		= \sum_{S \in \binom{[n]}{s} } \frac{\prod_{i \in S} \hat{x}_i}{ \sum_{\bar{S} \in \binom{[n]}{s} }  \prod_{i \in \bar{S}} \hat{x}_i  } \det^s (\bm{V}_{S} \bm{V}_{S}^{\top} ) \\
		&\geq  \frac{\exp\bigg[\Gamma_s\bigg(\sum_{i \in [n]}  \hat{x}_i \bm{v}_i \bm{v}_i^{\top}\bigg)\bigg] }{ \sum_{\bar{S} \in \binom{[n]}{s} }  \prod_{i \in \bar{S}} \hat{x}_i  } \ge \left(\left(\frac{s}{n}\right)^{s} {\binom{n}{s}}\right)^{-1} \exp\bigg[\Gamma_s\bigg(\sum_{i \in [n]}  \hat{x}_i \bm{v}_i \bm{v}_i^{\top}\bigg)\bigg] \\
		&\ge \left(\left(\frac{s}{n}\right)^{s} {\binom{n}{s}}\right)^{-1} \exp \left (  z^{*} -\alpha\right ),
		\end{align*}
		where the first inequality is due to \Cref{bound}, the second one is from Maclaurin’s inequality \citep{lin1994some}, and the last one is due to the $\alpha$-optimality of the solution $\hat{\bm{x}}$ and the weak duality $z^{LD}\geq z^*$. The conclusion follows by taking logarithm on both sides of the above inequalities. \qed

	\end{proof}
	
	We make the following remarks about the result in Theorem~\ref{thm:samp}.
	\begin{enumerate}[(i)] \setlength{\itemsep}{0pt}
		\item This approximation bound of sampling \Cref{alg_rand_round} improves the one studied in \cite{nikolov2015randomized} using a different sampling scheme, where the existing approximation bound is $\log\left(s^s/s!\right)+ \alpha$ (see Figure \ref{cfig} for illustrations). To show this fact, it suffices to prove that
		\begin{align*}
		\left ( \left(\frac{s}{n}\right)^{s} {\binom{n}{s}} \right)^{-1} \ge \frac{s!}{s^s},
		\end{align*}
		i.e.,
		\begin{align*}
		\left ( \left(\frac{s}{n}\right)^{s} {\binom{n}{s}} \right)^{-1} \frac{s^s}{s!} = \frac{n^s}{n\cdots (n-s+1)} \ge 1,
		\end{align*}
		where the inequality relies on the fact that $n\geq n-j+1$ for each $j\in [s]$.
		
		\begin{figure}[h]
			\subfigure[{$n$=100}] { \label{cfig:a}
				\includegraphics[width=7.1cm,height=5cm]{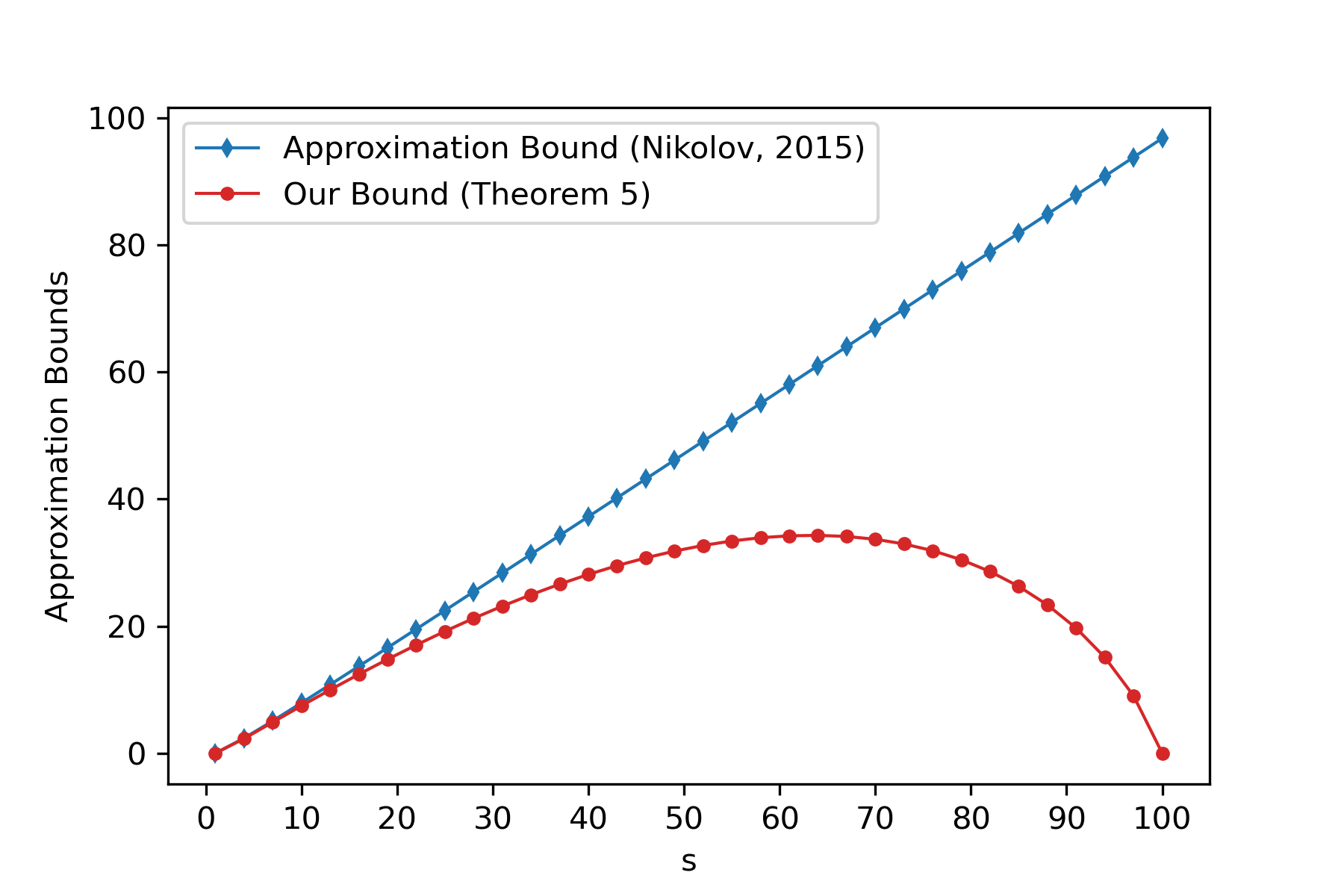}
			}
			\hspace{2em}
			\subfigure[{$n$=1000}] { \label{cfig:b}
				\includegraphics[width=7.1cm,height=5cm]{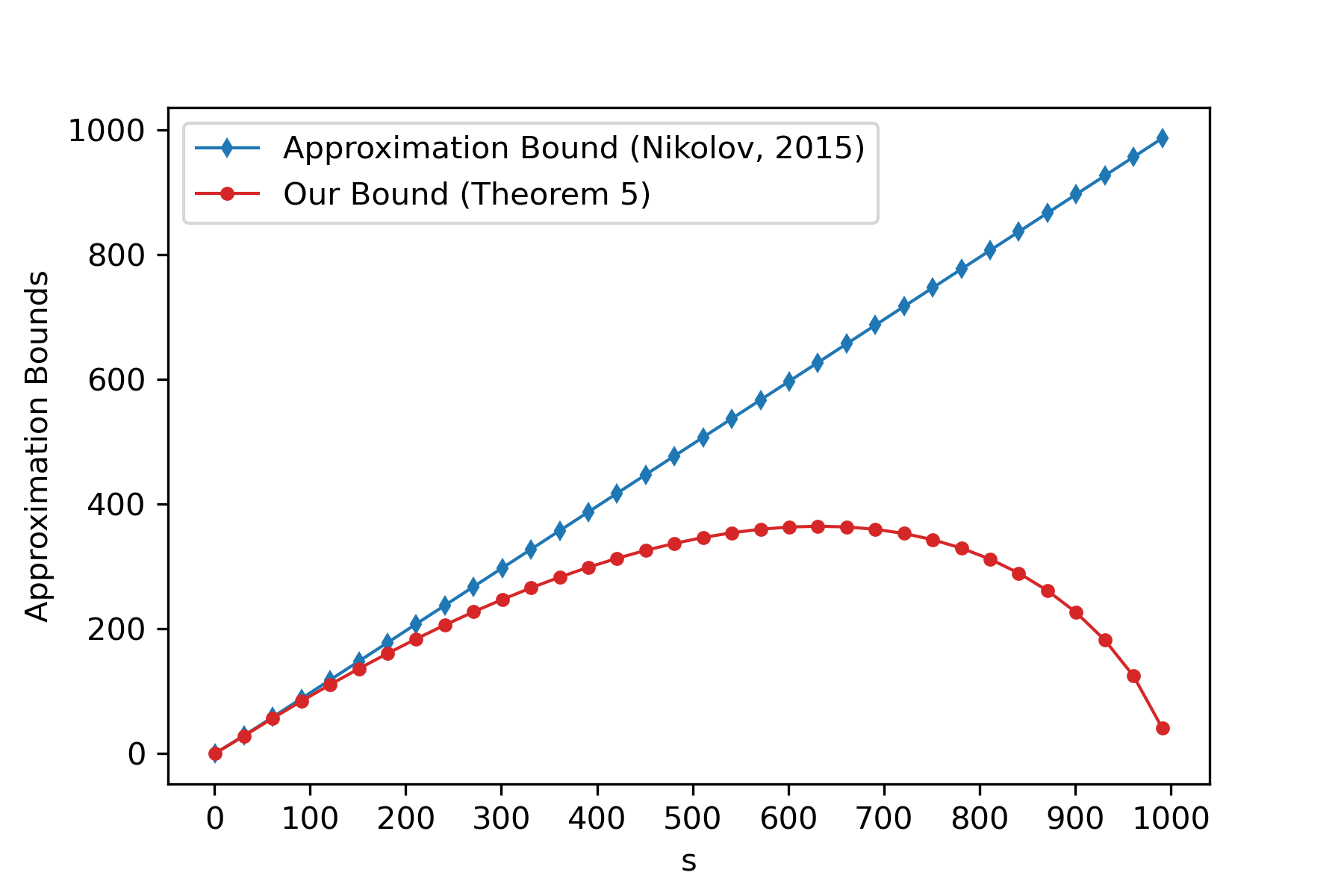}
			}
			\caption{Approximation bounds comparison of our sampling \Cref{alg_rand_round} and \cite{nikolov2015randomized} with $\alpha=0$.}
			\label{cfig}
		\end{figure}
		\item The approximation bound attains zero when $s=1$ and $s=n$.
		\item The proof in Theorem \ref{thm:samp} indicates that the approximation bound depends on the sparsity of the $\alpha$-optimal solution $\hat{\bm{x}}$  to PC \eqref{eq_pcont}. Indeed, if we  consider the sampling probability as
		\[\mathbb{P}[\tilde{S}=S] = \frac{\prod_{i \in S} \hat{x}_i}{ \sum_{\bar{S} \in \binom{\supp(\hat{\bm{x}})}{s} }  \prod_{i \in \bar{S}} \hat{x}_i  }, \]
		for any size-$s$ subset $S\subseteq \supp(\hat{\bm{x}})$. Then the approximation bound can be further improved as $(s\log(s) + \log (\binom{\hat n}{s}) - s\log (\hat n)+\alpha)$, where $\hat{n}=|\supp(\hat{\bm{x}})|$. This bound can be much smaller than the one in Theorem~\ref{thm:samp} if $\hat{n}\ll n$.
	\end{enumerate} 
	Another observation is that the optimal value of the continuous relaxation of MESP \eqref{eq_p} (i.e., PC \eqref{eq_pcont}) is not too faraway from the optimal value $z^*$.
	\begin{corollary}\label{cor_pc_bound} The optimal value of PC \eqref{eq_pcont} is bounded by $z^*+s\log(s) + \log (\binom{n}{s}) - s\log (n)$, i.e.
		\[z^*\leq z^{LD} \leq z^*+s\log(s) + \log \bigg(\binom{n}{s}\bigg) - s\log (n).  \]
	\end{corollary}
	
	\begin{proof}The proof follows from that in Theorem~\ref{thm:samp} by observing that $z^*\geq \log \mathbb{E} [ \overset{s}{\det} (\sum_{i \in \tilde{S}}\bm{v}_i \bm{v}_i^{\top} ) ]$ and $\alpha$ can be arbitrarily positive.
		\qed
	\end{proof}
	
	The following instance illustrates the tightness of our analysis for the sampling Algorithm \ref{alg_rand_round}.
	\begin{proposition}\label{prop:samp}
		Given the sampling probability in \eqref{samprob}, there exists an instance such that
		\[ 	\log \mathbb{E} \bigg[ \det^s \bigg(\sum_{i \in \tilde{S}}\bm{v}_i \bm{v}_i^{\top} \bigg) \bigg] = z^*-s\log(s) - \log \bigg(\binom{n}{s}\bigg) +s\log (n).\]
	\end{proposition}

	\begin{proof}
		Let us consider the following example. 
		\begin{example}\label{examp3} Suppose that $d=s,n=\ell s$ with some positive integer $\ell$, and $\bm{v}_{s\times (t-1)+i} = \bm{e}_i$ for all $(i,t) \in [s] \times [\ell]$.  %
		\end{example}
		Clearly, in Example~\ref{examp3}, we have $z^*=z^{LD}=0$, and 
		one optimal solution to PC \eqref{eq_pcont} is ${\hat{x}}_i = \frac{s}{n}=\frac{1}{\ell}$ for all $i \in [n]$. If we use $\bm{\hat x}$ as the input of the sampling Algorithm \ref{alg_rand_round}, then the expected exponential of the output objective value is
		\begin{align*}
		\mathbb{E} \bigg[ \det^s \bigg(\sum_{i \in \tilde{S}}\bm{v}_i \bm{v}_i^{\top} \bigg) \bigg] & 
		= \sum_{S \in \binom{[n]}{s} } \frac{\prod_{i \in S} \hat{x}_i}{ \sum_{\bar{S} \in \binom{[n]}{s} }  \prod_{i \in \bar{S}} \hat{x}_i  } \det^s \bigg(\sum_{i \in {S}}\bm{v}_i \bm{v}_i^{\top} \bigg) =  \left(\left(\frac{s}{n}\right)^{s} {\binom{n}{s}}\right)^{-1} \exp(z^*) .
		\end{align*} 
		\qed
	\end{proof}

	\subsection{Deterministic Implementation}
	To overcome the issue of randomness from the sampling algorithms, it is common to derive their corresponding polynomial-time deterministic implementation \citep{nikolov2015randomized,singh2020approximation,nikolov2019proportional}. In this subsection, we also develop the deterministic implementation of the proposed sampling Algorithm \ref{alg_rand_round} with the same approximation bound, which is presented in Algorithm~\ref{algo:deter}. 
	The key idea of derandomization is to apply the method of {conditional expectation}  \citep{alon2016probabilistic},  which requires  an auxiliary function regarding the conditional expected value of the function $\overset{s}{\det}(\cdot)$. 

First, for notational convenience, let us introduce the elementary symmetric polynomials.
\begin{definition} \label{def:elem}
For any vector $\bm{x} \in \Re^n$ and a positive integer $\ell\in [n]$, we define the elementary symmetric polynomial of degree $\ell$ as
$$E_{\ell}(\bm{x}) := \sum_{S \in \binom{[n]}{\ell} }  \prod_{i \in S} x_i.$$ 
\end{definition}	
	 In the deterministic Algorithm~\ref{algo:deter}, given an $\alpha$-optimal solution to PC \eqref{eq_pcont} and a subset $T\subseteq [n]$
	such that $|T|=t\leq s$, according to the sampling probability \eqref{samprob}, {the conditional expected exponential of the objective value of MESP is equal to}
		\begin{align}
	\H (T)&=\mathbb{E} \bigg[\det^s \bigg(\sum_{i \in \tilde{S}} \bm{v}_i \bm{v}_i^{\top} \bigg) | T \subseteq \tilde S\bigg]
	= \sum_{ \substack{S \in \binom{[n]}{s}}} P(S|T \subseteq S ) \det^s \bigg(\sum_{i \in {S}} \bm{v}_i \bm{v}_i^{\top} \bigg) \notag \\
	&= \sum_{ \substack{S \in \binom{[n]}{s}}}\frac{ {\mathbb{I}_{\{T \subseteq S\}}}  \prod_{i \in {S} \backslash T} \hat{x}_i}{ \sum\limits_{ \substack{\bar{S} \in \binom{[n]}{s}}} \mathbb{I}_{\{T \subseteq \bar{S} \}} \prod_{i \in \bar{S} \backslash T} \hat{x}_i}  \det(\bm{C}_{S,S})= \frac{E_{s-|T|}(\bm\lambda(T))}{ \sum\limits_{ \substack{\bar{S} \in \binom{[n]}{s}}}  {\mathbb{I}_{\{T \subseteq \bar{S} \}}}  \prod_{i \in \bar{S} \backslash T} \hat{x}_i} \det(\bm{C}_{T, T}),\label{eq:def_H}
	\end{align}
	where $\mathbb{I}_{(\cdot)}$ denotes the indicator function, $\bm\lambda(T)$ denotes the vector of eigenvalues of  $\left(\bm{C}^{1/2} \bm{V}^{\top} ({\bm I_d} -(\bm{V}_T\bm{V}_T^{\top})^\dag\bm{V}_T\bm{V}_T^{\top}) \bm{V} \bm{C}^{1/2} \right)_{[n]\backslash T, [n]\backslash T}$, and the last equality is according to theorem 19 in \cite{nikolov2015randomized}. Note that the denominator in \eqref{eq:def_H} can be computed efficiently according to observation 1 in \cite{singh2020approximation} with running time complexity $O(n\log n)$. The numerator can be also computed efficiently according to the remark after theorem 19 in \cite{nikolov2015randomized}, which requires to compute the characteristic function of a matrix (e.g., Faddeev-LeVerrier algorithm in \citealt{hou1998classroom}) with time complexity $O(n^4)$.

	Algorithm \ref{algo:deter} proceeds as follows. We start with an empty subset $S$, then for each $j\notin S$, we compute the the conditional expected exponential of the objective value of MESP, provided that the $j$-th column $\bm{v}_j$ will be chosen, i.e., $\H(S\cup \{j\})$. We add $j^*$ to $S$, where $j^*\in \arg\max_{j\in [n]\setminus S}\H(S\cup \{j\})$ and then go to next iteration. This procedure  terminates until $|S|=s$. Besides, Algorithm~\ref{algo:deter} requires $O(ns)$ evaluations of function $\H(\cdot)$; hence,  the corresponding time complexity is $O(n^5s)$. Therefore,  we recommend Algorithm \ref{alg_rand_round} due to its simplicity and shorter running time.
	
	The performance guarantee for Algorithm~\ref{algo:deter} is identical to Theorem~\ref{thm:samp}, as summarized below.
	
	\begin{algorithm}[htbp]
		\caption{Deterministic Implementation} \label{algo:deter}
		\begin{algorithmic}[1]
			\State \textbf{Input:} $n\times n$ matrix $\bm{C} \succeq 0$ of rank $d$ and integer $s \in [d]$
			\State Let $\bm{C}=\bm{V}^{\top}\bm{V}$ denote its Cholesky factorization where $\bm{V} \in \Re^{d\times n}$
			\State Let $\bm{v}_i \in \Re^d$ denote the $i$-th column vector of matrix $\bm{V}$ for each $i \in [n]$
			\State Let $\bm{\hat x}$ be an $\alpha$-optimal solution $\bm{\hat x}$ of PC \eqref{eq_pcont} with $\alpha \in (0, \infty)$
			\State Let set ${\hat{S}} := \emptyset$ denote the chosen set	
			\For{$i = 1, \cdots, s$}
			\State Let $j^*\in \arg\max_{j\in [n]\setminus {\hat{S}}}\H({\hat{S}}\cup \{j\})$
			\State Add $j^*$ to the set $\hat{S}$
			\EndFor
			\State \textbf{Output:} $\hat S$%
		\end{algorithmic}
	\end{algorithm}

	\begin{theorem}	\label{thm_deter}
		The deterministic \Cref{algo:deter} yields the same approximation bound for MESP as the sampling \Cref{alg_rand_round} , i.e, suppose that the output of Algorithm \ref{algo:deter}  is $\hat{S}$, then
		\begin{align*}
		\log \det^s \bigg(\sum_{i \in \hat{S}} \bm{v}_i \bm{v}_i^{\top} \bigg)\ge z^*-s\log(s) - \log \bigg(\binom{n}{s}\bigg) +s\log (n) -  \alpha.
		\end{align*}
	\end{theorem}

	\section{Local Search Algorithm and its Approximation Guarantees} 
	\label{sec:loc}
	In this section, we investigate the widely-used local search algorithm (see, e.g., \citealt{hazimeh2020fast, madan2019combinatorial}) on solving MESP and prove its performance guarantee. The local search algorithm runs as follows: (i) first, we initialize a size-$s$ subset $\hat{S}\subseteq [n]$; (ii) next, we swap one element from the set $\hat{S}$ with one from the unchosen set $[n]\setminus\hat{S}$, and we update the chosen set if such a movement strictly increases the objective value; and (iii) the algorithm terminates until no improvement can be found. %
	The detailed implementation can be found in Algorithm \ref{algo}.

	\begin{algorithm}[htbp]
		\caption{Local Search Algorithm} \label{algo}
		\begin{algorithmic}[1]
			\State \textbf{Input:} $n\times n$ matrix $\bm{C} \succeq 0$ {of rank $d$ and integer $s \in [d]$}
			\State Let $\bm{C}=\bm{V}^{\top}\bm{V}$ denote its Cholesky factorization where $\bm{V} \in \Re^{d\times n}$
			\State Let $\bm{v}_i \in \Re^d$ denote the $i$-th column vector of matrix $\bm{V}$ for each $i \in [n]$
			\State Initial subset $\hat{S} \subseteq [n]$ of size $s$ such that $\{\bm{v}_i\}_{i\in \hat{S}}$ are linearly independent
			\Do
			\For{each pair {$(i,j) \in \hat{S} \times ([n]\setminus \hat{S})$}}
			\If{$\log \overset{s}{\det} \left(\sum_{\ell \in \hat{S} \cup \{j\} \setminus \{i\}} \bm{v}_\ell \bm{v}_\ell^{\top}\right) > \log \overset{s}{\det} \left(\sum_{\ell \in \hat{S}} \bm{v}_\ell \bm{v}_\ell^{\top}\right)$}
			\State Update $\hat{S} := \hat{S} \cup \{j\} \setminus \{i\}$
			\EndIf
			\EndFor
			\doWhile{there is still an improvement}
			\State \textbf{Output:} $\hat S$%
		\end{algorithmic}
	\end{algorithm}
	Let us first derive the following technical results on the rank-one update of singular matrices, which are essential to the analysis of the local search Algorithm \ref{algo}.
	\begin{restatable}{lemma}{lemrankupd} \label{lem:rankupd}
		Consider a size-$\tau$ subset $\hat{S}\subseteq [n]$ with $\tau\in [d]$ such that $\{\bm{v}_i\}_{i\in \hat{S}}$ are linearly independent. {Let $\bm{X} = \sum_{i \in \hat{S}} \bm{v}_{i}\bm{v}_{i}^{\top}$, and for each $i \in \hat{S}$, let $\bm{X}_{-i}=\bm{X}-\bm{v}_{i}\bm{v}_{i}^{\top}$.} Then for each $(i,j)\in \hat{S}\times ([n]\setminus \hat{S})$,  we have the followings
		\begin{enumerate}[(i)]
			\item $	\overset{\tau}{\det} (\bm{X})=
			\overset{\tau-1}{\det} (\bm{X}_{-i}) \bm{v}_i^{\top}(\bm{I}_d-
			\bm{X}_{-i}^{\dag}\bm{X}_{-i})\bm{v}_i,$
			
			\item 
			$ 
			\begin{cases}
			\overset{\tau}{\det} (\bm{X}_{-i}+\bm{v}_j\bm{v}_j^{\top})=
			\overset{\tau-1}{\det} (\bm{X}_{-i}) \bm{v}_j^{\top}(\bm{I}_d-
			\bm{X}_{-i}^{\dag}\bm{X}_{-i})\bm{v}_j , & \textrm{ if } \bm{v}_j \notin \col(\bm{X}_{-i}),\\
			\overset{\tau-1}{\det} (\bm{X}_{-i}+\bm{v}_j\bm{v}_j^{\top})=
			\overset{\tau-1}{\det} (\bm{X}_{-i}) (1+\bm{v}_j^{\top}\bm{X}_{-i}^{\dag}\bm{v}_j),& \textrm{otherwise},
			\end{cases} $
			
			\item $\bm{X}^{\dag} =\bm{X}_{-i}^{\dag} - 
			\frac{\bm{X}_{-i}^{\dag} \bm{v}_i \bm{v}_i^{\top}(\bm{I}_d-\bm{X}_{-i}^{\dag}\bm{X}_{-i}) }{\|(\bm{I}_d-\bm{X}_{-i}^{\dag}\bm{X}_{-i}) \bm{v}_i \|_{2}^{2}} -
			\frac{(\bm{I}_d-\bm{X}_{-i}^{\dag}\bm{X}_{-i}) \bm{v}_i \bm{v}_i^{\top} \bm{X}_{-i}^{\dag} }{\|(\bm{I}_d-\bm{X}_{-i}^{\dag}\bm{X}_{-i}) \bm{v}_i \|_{2}^{2}}  + \frac{(1+\bm{v}_i^{\top} \bm{X}_{-i}^{\dag} \bm{v}_i)(\bm{I}_d-\bm{X}_{-i}^{\dag}\bm{X}_{-i}) \bm{v}_i \bm{v}_i^{\top} (\bm{I}_d-\bm{X}_{-i}^{\dag}\bm{X}_{-i}) }{\|(\bm{I}_d-\bm{X}_{-i}^{\dag}\bm{X}_{-i}) \bm{v}_i \|_{2}^{4}},$
			\item $\bm{X}_{-i}^{\dag} =\bm{X}^{\dag} - \frac{\bm{X}^{\dag}\bm{v}_i \bm{v}_i^{\top} \bm{X}^{\dag}\bm{X}^{\dag}}{\|\bm{X}^{\dag} \bm{v}_i\|_2^2}- \frac{\bm{X}^{\dag}\bm{X}^{\dag}\bm{v}_i \bm{v}_i^{\top} \bm{X}^{\dag}}{\|\bm{X}^{\dag} \bm{v}_i\|_2^2} +
			\frac{\bm{v}_i^{\top} (\bm{X}^{\dag})^3 \bm{v}_i\bm{X}^{\dag}\bm{v}_i \bm{v}_i^{\top} \bm{X}^{\dag}}{\|\bm{X}^{\dag} \bm{v}_i\|_2^4},$
			\item  $\bm{v}_i^{\top} \bm{X}^{\dag} \bm{v}_i=1,$
			\item $\bm{v}_i^{\top} (\bm{I}_d - \bm{X}^{\dag} \bm{X}) = \bm{0},$
			\item 
			$\bm{v}_i^{\top}(\bm{I}_d-\bm{X}_{-i}^{\dag}\bm{X}_{-i})\bm{v}_i =\frac{1}{\|\bm{ X}^{\dag}\bm{v}_i\|^2_2},$
			\item 
			$ \bm{v}_j^{\top}(\bm{I}_d-\bm{X}_{-i}^{\dag}\bm{X}_{-i})\bm{v}_j=
			\begin{cases}
			\bm{v}_j^{\top} (\bm{I}_d-\bm{ X}^{\dag} \bm{X}) \bm{v}_j+\frac{(\bm{v}_j^{\top} \bm{ X}^{\dag} \bm{v}_i)^2}{\|\bm{ X}^{\dag}\bm{v}_i\|^2_2}, & \textrm{ if } \bm{v}_j \notin \col(\bm{X}_{-i});\\
			0,& \textrm{otherwise}.
			\end{cases} $.
		\end{enumerate}
	\end{restatable}
	
	\begin{proof} 
		See Appendix \ref{proof_lem_rankupd}.
		\qed
	\end{proof}
	
	Lemma \ref{lem:rankupd} helps establish the local optimality condition (i.e., stopping criterion) of the local search Algorithm \ref{algo}. That is, we first rewrite the local optimality condition as
\[\log \overset{s}{\det} \bigg(\sum_{\ell \in \hat{S} \cup \{j\} \setminus \{i\}} \bm{v}_{\ell}\bm{v}_{\ell}^{\top}\bigg) - \log \overset{s-1}{\det} \bigg(\sum_{\ell \in \hat{S}  \setminus \{i\}} \bm{v}_{\ell}\bm{v}_{\ell}^{\top}\bigg) \leq \log \overset{s}{\det} \bigg(\sum_{\ell \in \hat{S}} \bm{v}_\ell \bm{v}_\ell^{\top}\bigg)- \log \overset{s-1}{\det} \bigg(\sum_{\ell \in \hat{S} \setminus \{i\}} \bm{v}_{\ell}\bm{v}_{\ell}^{\top}\bigg), \]
for all $i\in \hat{S}$ and $j\in [n]\setminus \hat{S}$, and then use  the results in Lemma \ref{lem:rankupd}  to simplify the both differences. 
	\begin{restatable}{lemma}{lemrankupdtwo} \label{lem:rankupd2}
		Let $\hat{S}$ denote the output of the local search Algorithm \ref{algo} and let $\bm{X} = \sum_{i \in \hat{S}} \bm{v}_i\bm{v}_i^{\top}$. Then for each pair $(i,j)\in \hat{S} \times ([n]\setminus \hat{S})$, the following inequality holds
		$$1 \ge \left (\bm{v}_i^{\top}\bm{X}^{\dag} \bm{X}^{\dag}\bm{v}_i\right) \bm{v}_j^{\top}  (\bm{I}_d - \bm{X}^{\dag} \bm{X}) \bm{v}_j + \bm{v}_j^{\top} \bm{X}^{\dag} \bm{v}_i\bm{v}_i^{\top} \bm{X}^{\dag} \bm{v}_j. $$

	\end{restatable}
	\begin{proof}
		See Appendix \ref{proof_lem_rankupdtwo}. \qed
	\end{proof}
	
	\subsection{Analysis of Local Search Algorithm \ref{algo}}
	Now we are ready to analyze the local search Algorithm \ref{algo}. The main proof idea is two-fold: (i) using the output of the local search Algorithm \ref{algo} and its local optimality condition in \Cref{lem:rankupd2}, we construct a dual feasible solution to LD \eqref{eq_dual}, and (ii) we show that the objective value of this dual feasible solution can be bounded by $z^*$ with some extra constant.
	\begin{restatable}{theorem}{themlocal}\label{them1}
		Let $\hat{S}$ denote the output of the local search Algorithm \ref{algo}, then the set $\hat{S}$ yields a $s\min\{\log(s), \log (n-s-n/s+2)\} $-approximation bound for MESP \eqref{eq_obj}, i.e., 
		\begin{align*}
		\log \det^s \bigg(\sum_{i \in \hat{S}} \bm{v}_i\bm{v}_i^{\top}\bigg) \ge z^* -s \min\left\{ \log(s), \log \left(n-s-\frac{n}{s}+2\right) \right\}  . 
		\end{align*}
	\end{restatable}
	
	\begin{proof}
		See Appendix \ref{proof_them1}. \qed
	\end{proof}
	
	We make the following remarks about Theorem~\ref{them1}.
	\begin{enumerate}[(i)] \setlength{\itemsep}{0pt}
		\item To the best of our knowledge, it is the first-known approximation bound of the local search Algorithm \ref{algo} for MESP.
		\item The approximation bound attains the maximum when $s=\frac{n}{2}$ and is equal to zero when $s=1$ or $s=n$.
		\item The approximation bound is weaker than that of the sampling \Cref{alg_rand_round} in \Cref{thm:samp} if the continuous relaxation can be solved to optimality or very close to optimality. That is, if $\alpha\rightarrow 0$, then we have
		\[s\log(s) + \log \bigg(\binom{n}{s}\bigg) - s\log (n) \leq s\min\left\{ \log (s), \log \left(n-s-\frac{n}{s}+2\right) \right\}.\]
		However, as we can see from the numerical study, the local search Algorithm \ref{algo} in practice is more capable to find high-quality solutions than the sampling \Cref{thm:samp}.
		\item The proof also relies on the sparsity of the optimal solution to PC \eqref{eq_pcont}. In fact, if there exists a sparse optimal solution $\bm{x}^*$ to PC \eqref{eq_pcont} (i.e., $|\supp(\bm{x}^*)|\ll n$), then according to KKT conditions, we can drop the redundant dual constraints $\bm{v}_i^{\top} \bm{ \Lambda} \bm{v}_i \leq  \nu+\mu_i$ for each $i\in [n]\setminus \supp(\bm{x}^*)$ in LD \eqref{eq_dual}. Therefore, following the same proof in Theorem~\ref{them1}, the approximation bound can be further improved as $s\min\{\log (s), \log (\hat{n}-s-\hat{n}/s+2) \}$, where $\hat{n} = |\supp(\bm{x}^*)|$.
	\end{enumerate} 
	
The following instance shows that the proof of Theorem \ref{them1} is tight. That is, the approximation bound cannot be improved if we construct a feasible $\bm{\Lambda}$ to LD \eqref{eq_dual} as
	\begin{align}
	\bm{ \Lambda} = \frac{1}{t}\left[\tr(\bm{X}^{\dag})(\bm{I}_d - \bm{ X}^{\dag} \bm{ X}) + \bm{X}^{\dag}\right],%
	\label{feas_lambda}
	\end{align}
	where for the output $\hat{ S}$ of the local search \Cref{algo}, we let $\bm{X}=\sum_{i \in {\hat S}} \bm{v}_i\bm{v}_i^{\top}$ and let $t>0$ be a positive scaling factor.

		\begin{restatable}{proposition}{propdeg}\label{prop:degenerate}
		If one follows the construction of a feasible solution $\bm{\Lambda}$ in \eqref{feas_lambda} to LD \eqref{eq_dual}, then even with the best choice of $(\nu,\bm\mu)$, there exists an instance such that
		\[-\log \det_s (\bm{\Lambda}) +s\nu+\sum_{i \in [n]} \mu_i-s=z^*+ s\min\left\{ \log (s), \log \left(n-s-n/s+2\right) \right\}.\]
	\end{restatable}	
	
\begin{proof}	
	See Appendix \ref{proof_prop_deg}. \qed
\end{proof}
	
	The above proposition shows the tightness of the analysis of Theorem \ref{them1}. Thus, to improve the analysis of the local search Algorithm~\ref{algo}, one might need different ways to construct dual feasible solutions to LD \eqref{eq_dual}. In fact, we show that under a certain assumption,  the approximation bound of the local search Algorithm~\ref{algo} can be improved.

	\begin{restatable}{proposition}{propls}\label{prop_ls}
		Let $\hat{S}$ denote the output of the local search Algorithm \ref{algo}.	Suppose that $\bm{v}_i^{\top} \bm{v}_j = 0$ for each pair $ (i,j) \in \hat{S}\times ([n] \setminus \hat{S})$, then we have
		\begin{align*}
	\log \det^s \bigg(\sum_{i \in \hat{S}} \bm{v}_i\bm{v}_i^{\top}\bigg)  \ge z^*- s\min\left\{ \log\left(\frac{\lambda_{\max}(\bm C)}{\delta}\right),  \log \left( \frac{\lambda_{\max}(\bm C)}{s\delta }(n-s) -\frac{n}{s} + 2\right) \right\} ,
		\end{align*}
		where the constant $\delta$ is defined in \Cref{lem_hessian}.
	\end{restatable}
	
	\begin{proof}	
		See Appendix \ref{proof_prop_ls}. \qed
	\end{proof}

	Compared with the bound $O(s\log s)$ in Theorem \ref{them1}, the approximation bound in \Cref{prop_ls} is $O(s)$, which matches the order of the bound derived for the sampling \Cref{alg_rand_round}.
	
	\subsection{Efficient Implementation of the Local Search Algorithm}
	In this subsection, we discuss how to efficiently implement the local search Algorithm \ref{algo} using the results in \Cref{lem:rankupd} and develop its corresponding time complexity. 
	
	Similar to many improving heuristics, the performance of the local search Algorithm \ref{algo} highly depends on the choice of the initial subset. In practice, we employ the greedy approach to find an initial solution. The greedy approach begins with an empty set $\hat{S} = \emptyset$, then at each iteration, we select one element from the unchosen set $[n]\setminus {\hat{S}}$ that maximizes the marginal increment of the objective value until $|\hat{S}|=s$. That is, at current iteration $\ell\in [s]$, suppose that $\bm{X}=\sum_{i \in \hat{S}}\bm{v}_i \bm{v}_i^{\top}$ and $|\hat{S}|=\ell<s$. Then by Part (ii) in Lemma \ref{lem:rankupd}, the next element that will be chosen is computed by
	\[j^*\in  \arg\max_{j\in [n]\setminus \hat{S}} \left(\log \det^{\ell+1}(\bm{X}+ \bm{v}_j \bm{v}_j^{\top})-\log\det^{\ell}(\bm{X})\right) =\arg\max_{j\in [n]\setminus \hat{S}}  \bm{v}_j^{\top} (\bm{I}_d-\bm{X}\bm{X}^{\dag}) \bm{v}_j.\]
	The detailed implementation of the greedy approach can be found in \Cref{algo:effls} at Steps 4-10. Using the equation above and Part (iii) in Lemma \ref{lem:rankupd}, the greedy approach has a running time complexity of $O(s(n-s)d^2)$. Furthermore, we show that the rank-one update techniques for the singular matrices in Lemma \ref{lem:rankupd} can also improve the implementation of the local search \Cref{algo}.
	
	\begin{algorithm}[htbp]
		\caption{Efficient Implementation of Local Search Algorithm \ref{algo} Initialized by Greedy Solution} \label{algo:effls}
		\begin{algorithmic}[1]
			\State \textbf{Input:} $n\times n$ matrix $\bm{C} \succeq 0$ of rank $d$ and integer $s \in [d]$
			\State Let $\bm{C}=\bm{V}^{\top}\bm{V}$ denote its Cholesky factorization where $\bm{V} \in \Re^{d\times n}$
			\State Let $\bm{v}_i \in \Re^d$ denote the $i$-th column vector of matrix $\bm{V}$ for each $i \in [n]$
			\vspace{.5em}
			
			\noindent\textbf{(a) Greedy Selection}
			
			\State Let set $\hat{S} := \emptyset$ denote the chosen set, $\bm{X} := \emptyset $ and $\bm{X}^{\dag} := \emptyset$
			\For{$\ell = 1, \cdots, s$}
			\State Let $j^*\in \arg\max_{j\in [n]\setminus \hat{S}} \bm{v}_j^{\top} (\bm{I}_d-\bm{X}\bm{X}^{\dag}) \bm{v}_j$
			\State Add $j^*$ to the set $\hat{S}$
			\State Update $\bm{X}^{\dag} :=\bm{X}^{\dag} - 
			\frac{\bm{X}^{\dag} \bm{v}_{j^*}  \bm{v}_{j^*} ^{\top}(\bm{I}_d-\bm{X}^{\dag}\bm{X}) }{\|(\bm{I}_d-\bm{X}^{\dag}\bm{X}) \bm{v}_{j^*}  \|_{2}^{2}} -
			\frac{(\bm{I}_d-\bm{X}^{\dag}\bm{X}) \bm{v}_{j^*}  \bm{v}_{j^*}^{\top} \bm{X}^{\dag} }{\|(\bm{I}_d-\bm{X}^{\dag}\bm{X}) \bm{v}_{j^*}  \|_{2}^{2}}  + \frac{(1+\bm{v}_{j^*}^{\top} \bm{X}^{\dag} \bm{v}_i)(\bm{I}_d-\bm{X}^{\dag}\bm{X}) \bm{v}_{j^*}  \bm{v}_{j^*}^{\top} (\bm{I}_d-\bm{X}^{\dag}\bm{X}) }{\|(\bm{I}_d-\bm{X}^{\dag}\bm{X}) \bm{v}_{j^*}  \|_{2}^{4}}$
			\State Update $\bm{X}:= \bm{X}+\bm{v}_{j^*} \bm{v}_{j^*}^{\top}$
			
			\EndFor
			
			\vspace{.5em}
			
			\noindent\textbf{(b) Swapping Procedure}
			\State Let $\theta$ denote a positive constant
			\Do
			\For{each {$i \in \hat{S}$}}
			\State Compute $\bm{X}_{-i} = \bm{X}-\bm{v}_i \bm{v}_i^{\top}$, $ \bm{X}_{-i}^{\dag} =\bm{X}^{\dag} - \frac{\bm{X}^{\dag}\bm{v}_i \bm{v}_i^{\top} \bm{X}^{\dag}\bm{X}^{\dag}}{\|\bm{X}^{\dag} \bm{v}_i\|_2^2}- \frac{\bm{X}^{\dag}\bm{X}^{\dag}\bm{v}_i \bm{v}_i^{\top} \bm{X}^{\dag}}{\|\bm{X}^{\dag} \bm{v}_i\|_2^2} +
			\frac{\bm{v}_i^{\top} (\bm{X}^{\dag})^3 \bm{v}_i\bm{X}^{\dag}\bm{v}_i \bm{v}_i^{\top} \bm{X}^{\dag}}{\|\bm{X}^{\dag} \bm{v}_i\|_2^4}$
			\State Let $j^*\in \arg\max_{j\in [n]\setminus \hat{S}} \bm{v}_j^{\top} (\bm{I}_d-\bm{X}_{-i}\bm{X}_{-i}^{\dag}) \bm{v}_j$
			\If{$\bm{v}_{j^*}^{\top} (\bm{I}_d-\bm{X}_{-i}\bm{X}_{-i}^{\dag}) \bm{v}_{j^*} > (1+\theta) \bm{v}_i^{\top} (\bm{I}_d-\bm{X}_{-i}\bm{X}_{-i}^{\dag}) \bm{v}_i$}
			\State Update $\hat{S} := \hat{S} \cup \{j\} \setminus \{i\}$, $\bm{X}:=\bm{X}_{-i}+\bm{v}_{j^*}\bm{v}_{j^*}^{\top}$ and
			$\bm{X}^{\dag} :=\bm{X}_{-i}^{\dag} - 
			\frac{\bm{X}_{-i}^{\dag} \bm{v}_{j^*} \bm{v}_{j^*}^{\top}(\bm{I}_d-\bm{X}_{-i}^{\dag}\bm{X}_{-i}) }{\|(\bm{I}_d-\bm{X}_{-i}^{\dag}\bm{X}_{-i}) \bm{v}_{j^*}\|_{2}^{2}} -
			\frac{(\bm{I}_d-\bm{X}_{-i}^{\dag}\bm{X}_{-i}) \bm{v}_{j^*}\bm{v}_{j^*}^{\top} \bm{X}_{-i}^{\dag} }{\|(\bm{I}_d-\bm{X}_{-i}^{\dag}\bm{X}_{-i}) \bm{v}_{j^*} \|_{2}^{2}}  + \frac{(1+\bm{v}_{j^*}^{\top} \bm{X}_{-i}^{\dag} \bm{v}_{j^*})(\bm{I}_d-\bm{X}_{-i}^{\dag}\bm{X}_{-i}) \bm{v}_{j^*} \bm{v}_{j^*}^{\top} (\bm{I}_d-\bm{X}_{-i}^{\dag}\bm{X}_{-i}) }{\|(\bm{I}_d-\bm{X}_{-i}^{\dag}\bm{X}_{-i}) \bm{v}_{j^*}\|_{2}^{4}};$
			\EndIf
			\EndFor
			\doWhile{there is still an update}
			\State \textbf{Output:} $\hat S$%
			
		\end{algorithmic}
	\end{algorithm}

	One key component of the local search Algorithm \ref{algo} is the swapping procedure (i.e., Steps 6-9), which might cause the running time to be exponential in the size of the input. To avoid this, we can restrict the number of swapping iterations by simply introducing a small positive constant $\theta>0$ and replacing the condition at Step 8 of Algorithm \ref{algo} by 
	$$\overset{s}{\det} \bigg(\sum_{\ell \in \hat{S} \cup \{j\} \setminus \{i\}} \bm{v}_\ell \bm{v}_\ell^{\top}\bigg) > (1+\theta)\overset{s}{\det} \bigg(\sum_{\ell \in \hat{S}} \bm{v}_\ell \bm{v}_\ell^{\top}\bigg).$$
Then, following from the similar arguments in \cite{madan2019combinatorial}, the number of swapping iterations is at most $O(Ld^3 \theta^{-1}\log(s))$, where $L$ is the encoding length of the matrix $\bm{V}$. Note that by doing so, the approximation bound in \Cref{them1}  becomes $s \min\{ \log(s(1+\theta)), \log ((n-s)(1+\theta)-n/s+2)\}$.
	
	On the other hand, we can use Parts (ii) and (iv) in Lemma \ref{lem:rankupd} to complete the swapping and use Part (iii) in Lemma \ref{lem:rankupd} to update matrix $\bm{X}^{\dag}$. Hence, it takes $O(s(n-s)d^2)$ for each swapping. Thus, the local search Algorithm \ref{algo:effls} has a polynomial-time complexity of $O(Ld^3\theta^{-1} \log(s)s(n-s)d^2)$. These results are summarized below.
	\begin{corollary}\label{cor:localsearch}
		The running time complexity of the local search Algorithm \ref{algo:effls} is $O(Ld^3\theta^{-1} \log(s)s(n-s)d^2)$, where $L$ denotes the encoding length of the matrix $\bm{V}$. In addition, the local search Algorithm \ref{algo:effls} yields a $s\min\{\log(s(1+\theta)), \log ((n-s)(1+\theta)-n/s+2)\} $-approximation bound for MESP.
	\end{corollary}

\section{Numerical Illustrations} \label{sec:comp}
In this section, we present numerical experiments on two medium-sized instances in \cite{hoffman2001new} and \cite{anstreicher2020efficient}, which were provided by Prof. Anstreicher, and one large-scale instance in \cite{dey2022using} to demonstrate the solution quality and computational efficiency of our proposed Frank-Wolfe Algorithm \ref{algo:fw}, sampling Algorithm \ref{alg_rand_round}, and local search Algorithm \ref{algo} for solving MESP. All the algorithms are coded in Python 3.6 with calls to Gurobi 7.5 on a PC with 2.3 GHz Intel Core i5 processor and 8G of memory. The codes for these three algorithms are available at \url{https://github.com/yongchunli-13/Approximation-Algorithms-for-MESP}.

\subsection{Numerical Experiments on Two Medium-sized Instances}
In this subsection, we test the proposed algorithms on two commonly-used benchmark instances of MESP in literature and present their computational performance. In particular, the first instance has a covariance matrix of size $90\times90$ built on a temperature monitoring problem introduced in \cite{anstreicher2020efficient}, denoted by $n=90$ instance, and the second one is based on a covariance matrix of size $124\times 124$ introduced by \cite{hoffman2001new}, denoted by $n=124$ instance. Please note that these two covariance matrices are non-singular, i.e., $n=d$. For the $n=90$ instance, we test $8$ cases with $s\in \{10,20,\ldots,80\}$, while for the $n=124$ instance, we test $9$ cases with $s\in \{20, 30,\ldots,100\}$. The computational results are displayed in \Cref{data90} and \Cref{data124}, where we let \textbf{B\&B}, \textbf{Frank-Wolfe}, \textbf{Sampling}, \textbf{Local Search}, and \textbf{Samp+LS} denote the Branch and Bound algorithm used in \cite{anstreicher2020efficient},  the Frank-Wolfe Algorithm \ref{algo:fw}, the sampling Algorithm \ref{alg_rand_round}, the local search Algorithm \ref{algo}, and the combination of sampling Algorithm \ref{alg_rand_round} and local search Algorithm \ref{algo}, respectively. We also use \textbf{S-FW}  to denote the size of the support of the continuous relaxation solution from the Frank-Wolfe Algorithm \ref{algo:fw}, 
use \textbf{time} to denote the total time in seconds of an algorithm spent on a case, and use \textbf{gap} to denote the  absolute   optimality gaps of algorithms, computed as the absolute difference between the output value of an algorithm and the optimal value or the best upper bound of MESP, where only if the optimal value is not available, we use the upper bound to calculate the gap instead.
Note that due to the randomness, we repeat the sampling Algorithm \ref{alg_rand_round} one thousand times for each case and choose the best output, and its running time includes the time spent on the repetitions as well as that on running the Frank-Wolfe Algorithm \ref{algo:fw}.
The column ``Samp + LS" in \Cref{data90} and \Cref{data124}, denotes the integrated sampling \Cref{alg_rand_round} and local search \Cref{algo}. 
	Particularly, 
in  the integrated algorithm, we  consider one hundred random solutions of sampling \Cref{alg_rand_round} as the initial solutions of local search \Cref{algo} and then output the best solution for each testing case.

\Cref{data90} and \Cref{data124} present the numerical results. 
From \Cref{data90} and \Cref{data124}, we can see that it can take more than two days to solve some cases to optimality using the B\&B algorithm, indicating that the optimal value of MESP is in general difficult to obtain. Note that in the $n=124$ instance, the optimal value $z^*$ decreases when $s$ increases from 80 to 100, which demonstrates that the objective of MESP may not be monotonic with $s$. 
For both instances, the local search \Cref{algo} works quite well, where its absolute optimality gap is always within $0.096$, and its running time is less than a second. The sampling Algorithm \ref{alg_rand_round} is often worse than the local search \Cref{algo} in terms of optimality gap and computational time. 
The proposed integrated algorithm is able to find
an optimal solution for each testing case, however, takes a longer time.
It is  seen that the Frank-Wolfe \Cref{algo:fw} is quite effective, and its output can be indeed very sparse, especially when $s$ is small. 

Next, we compare two solution algorithms with the heuristic used in \cite{anstreicher2020efficient} and the results are illustrated in Figures \ref{fig_comp_appa} and \ref{fig_comp_appb}. Clearly, the proposed local search \Cref{algo} performs the best among these methods. Finally, Figure~\ref{fig_comp_upper} compares our Lagrangian dual bound $z^{LD}$ with the best linx bound found in \cite{anstreicher2020efficient}, where the latter has been shown to be superior to the other existing upper bounds of MESP on these two instances. In general, these two bounds are not comparable. We see that our dual bound outperforms the linx bound in some cases, especially when $s$ is small.

\begin{table}[h] 
	\centering
	\caption{Computational results of MESP on the $n=90$ instance} 
	\begin{threeparttable}
		\setlength{\tabcolsep}{3.5pt}\renewcommand{\arraystretch}{1.3}
		\begin{tabular}{c|r r|r r r |r r | r r | r r}
			\hline  
			\multicolumn{1}{c|}{$n$=90}   & \multicolumn{2}{c|}{B\&B\tnote{1}} & \multicolumn{3}{c|}{Frank-Wolfe } & \multicolumn{2}{c|}{Sampling} & \multicolumn{2}{c|}{Local Search} & \multicolumn{2}{c}{Samp + LS}
			\\ \cline{1-12}  
			$s$	& \multicolumn{1}{c}{$z^*$}  & \multicolumn{1}{c|}{time\tnote{2} }    &\multicolumn{1}{c}{gap} & \multicolumn{1}{c}{S-FW}&\multicolumn{1}{c|}{time} &\multicolumn{1}{c}{gap}  & \multicolumn{1}{c|}{time } &\multicolumn{1}{c}{gap} &   \multicolumn{1}{c|}{time} &\multicolumn{1}{c}{gap} &   \multicolumn{1}{c}{time} \\  
			\hline
			10 & 58.532   & 2088 & {0.382}&  23 & \textless1& {0.011} & 18  &{0.000} & {\textless 1\tnote{3 } } & 0.000 & 4   \\
			20 & 111.482  & 95976 &{0.645} &  42 &  \textless1 &{0.275}  & 20   &{0.000} & \textless 1 &{0.000}   & 9 \\
			30 & 161.539  & 167796 & {0.853}  &60   &\textless 1  &{0.655}  & 20   &{0.000} & \textless 1   &{0.000}   & 19 \\
			40 & 209.969  &187344 & {0.961} &80  & \textless1  &{1.212} & 19   &{0.011} &\textless 1  &{0.000} & 44 \\
			50 & 257.160  & 87912 &{0.955}&84  & \textless 1   &{1.424} & 19   &{0.006} & \textless 1  &{0.000}  & 68  \\
			60 & 303.019  & 12420 &{0.893}&87  & \textless 1 &{1.545} & 19  &{0.011} & \textless 1 &{0.000}   &88     \\
			70 & 347.471  & 1044 &{0.721}&89  & \textless 1   &{1.610}  & 19 &{0.018}& \textless 1   &{0.000}  &86   \\
			80 & 389.997  & 36 & {0.385} &89 & \textless 1&{0.995} & 19  &{0.000} & \textless 1   &{0.000}   & 92   \\
			\hline
		\end{tabular}%
		\begin{tablenotes}
			\item[1] The optimal value and running time of B\&B algorithm are from \cite{anstreicher2020efficient}
			\item[2] Time is in seconds
			\item[3] The running time is less than a second
		\end{tablenotes}    
	\end{threeparttable}
	\label{data90}
\end{table}

\begin{table}[h] 
	\centering
	\caption{Computational results of MESP on the $n=124$ instance}
	\begin{threeparttable}
		\setlength{\tabcolsep}{3pt}\renewcommand{\arraystretch}{1.3}
		\begin{tabular}{c|r r|r r r |r r | r r | r r}
	\hline  
	\multicolumn{1}{c|}{$n$=124}   & \multicolumn{2}{c|}{B\&B\tnote{1}} & \multicolumn{3}{c|}{Frank-Wolfe } & \multicolumn{2}{c|}{Sampling} & \multicolumn{2}{c|}{Local Search} & \multicolumn{2}{c}{Samp + LS}
	\\ \cline{1-12}  
	$s$	& \multicolumn{1}{c}{$z^*$}  & \multicolumn{1}{c|}{time\tnote{2} }    &\multicolumn{1}{c}{gap} & \multicolumn{1}{c}{S-FW}&\multicolumn{1}{c|}{time} &\multicolumn{1}{c}{gap}  & \multicolumn{1}{c|}{time } &\multicolumn{1}{c}{gap} &   \multicolumn{1}{c|}{time} &\multicolumn{1}{c}{gap} &   \multicolumn{1}{c}{time} \\  
			\hline
			20 & 77.827   & 756   & {0.510}& 40& 1 &{0.101} & 35  &{0.001} & {\textless 1\tnote{3 }  }  &{0.000} &   11  \\
			30 & 106.700   & 1692 & {1.285} &60 & 2  & {0.857}  & 37  &{0.000} & \textless1    &{0.000} &  15    \\
			40 & 131.055   & 8712& {2.246}& 80 & 3   &{2.067}  & 39  &{0.000} & \textless1   &{0.000} &  26    \\
			50 & 149.498   & 186516 & {3.857} & 98& 5  &{3.667} & 44   &{0.000} & \textless 1    &{0.000} &   37  \\
			60 & 164.012   & 241236 &{4.910}& 106& 6 &{6.057} & 41  &{0.096} & \textless1      &{0.000} &  57    \\
			70 & 172.528   & 136548 &{5.493}& 115& 5 &{6.712} & 41   &{0.000} & \textless 1    &{0.000} &     52   \\
			80 & 175.091   & 45756 &{5.529}& 122& 4 &{7.193}  & 40  &{0.000}& \textless 1     &{0.000} &   69    \\
			90 & 171.262   & 17352 & {5.790} & 124& 3  &{10.837} & 43 &{0.000} & \textless 1    &{0.000} &  77       \\
			100 & 162.865   & 4140 &{4.891} & 124& 3  &{7.273} & 39  &{0.000} & \textless1     &{0.000} &    74     \\
			\hline
		\end{tabular}%
		\begin{tablenotes}
			\item[1] The optimal value and running time of B\&B algorithm are from \cite{anstreicher2020efficient}
			\item[2] Time is in seconds
			\item[3] The running time is less than a second
		\end{tablenotes} 
	\end{threeparttable}
	\label{data124}
\end{table}

\begin{figure}[h]
	\centering
	\subfigure[{$n$=90}] {\label{fig_comp_appa}
		\includegraphics[width=7.2cm,height=5.5cm]{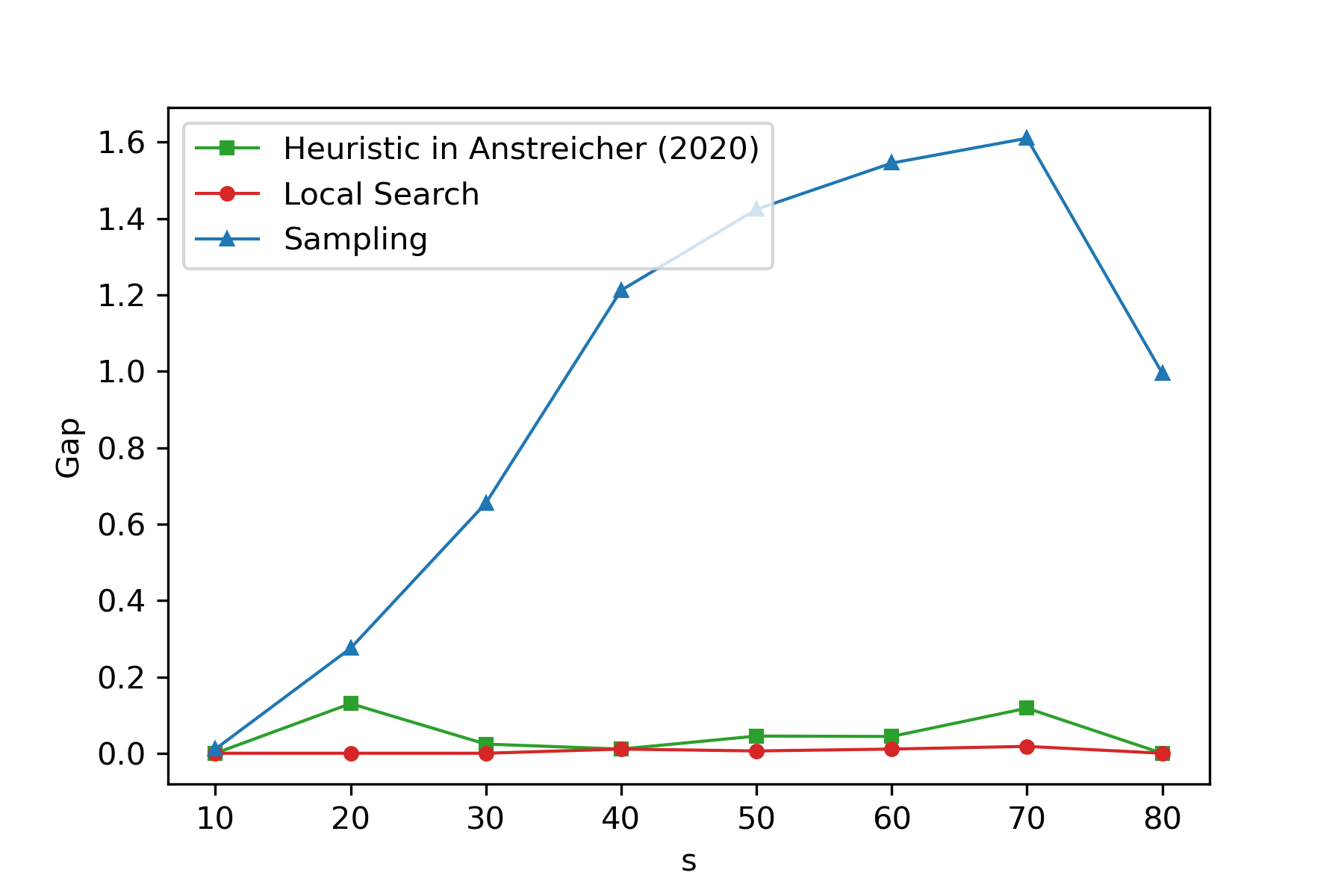}
	}
	\hspace{1.5em}
	\subfigure[{$n$=124}] {\label{fig_comp_appb}
		\centering
		\includegraphics[width=7.2cm,height=5.5cm]{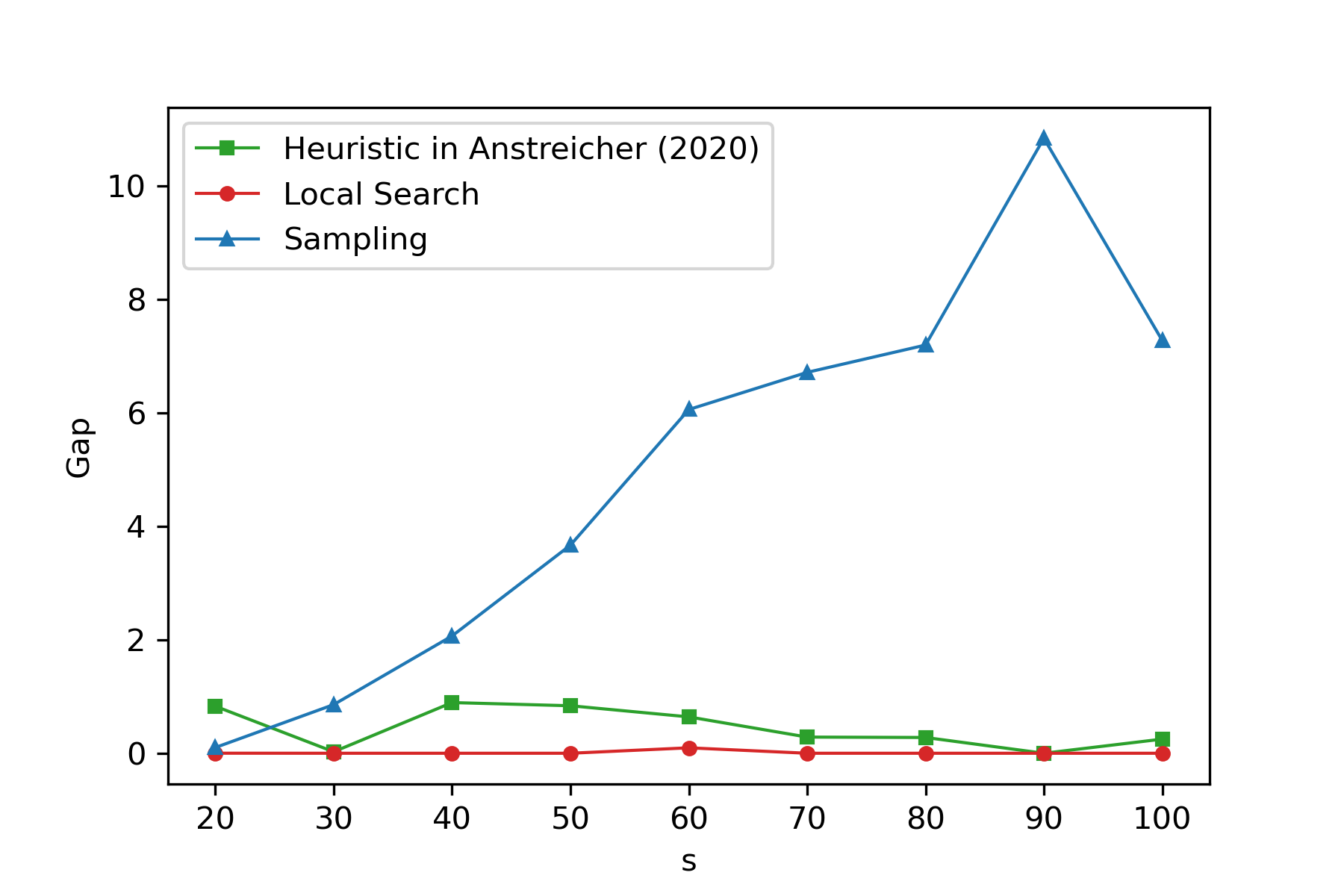}
	}
	
	\caption{Absolute optimality gap comparison of the sampling Algorithm \ref{alg_rand_round}, the local search \Cref{algo}, and the best heuristic in \cite{anstreicher2020efficient}.}
	\label{fig_comp_app}
\end{figure}

\begin{figure}[hbtp]
	\centering
	
	\subfigure[{$n$=90}] {\label{fig_comp_uppera}
		\includegraphics[width=7.2cm,height=5.3cm]{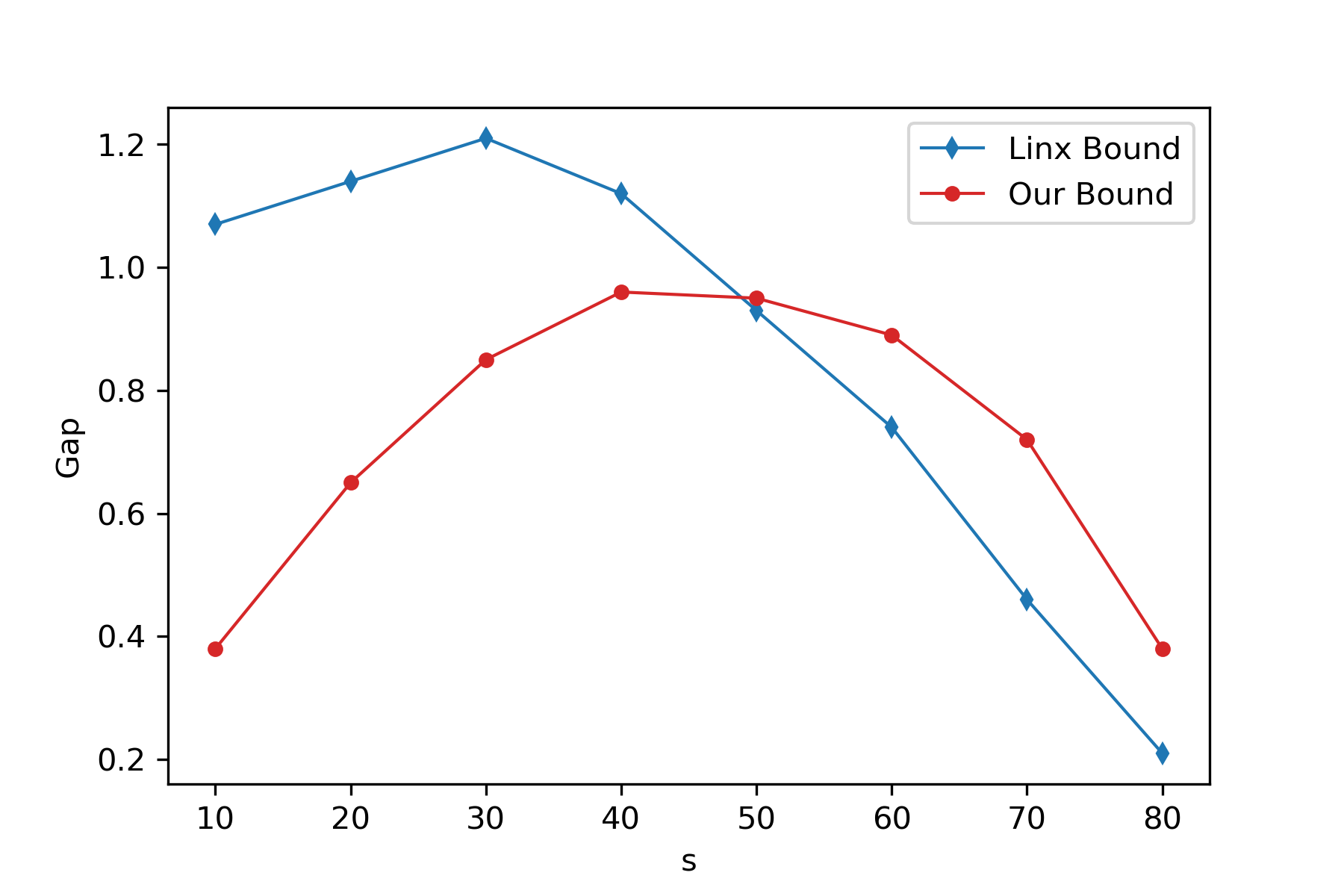}
	}
	\hspace{1.5em}
	\subfigure[{$n$=124}] {\label{fig_comp_upperb}
		\centering
		\includegraphics[width=7.2cm,height=5.3cm]{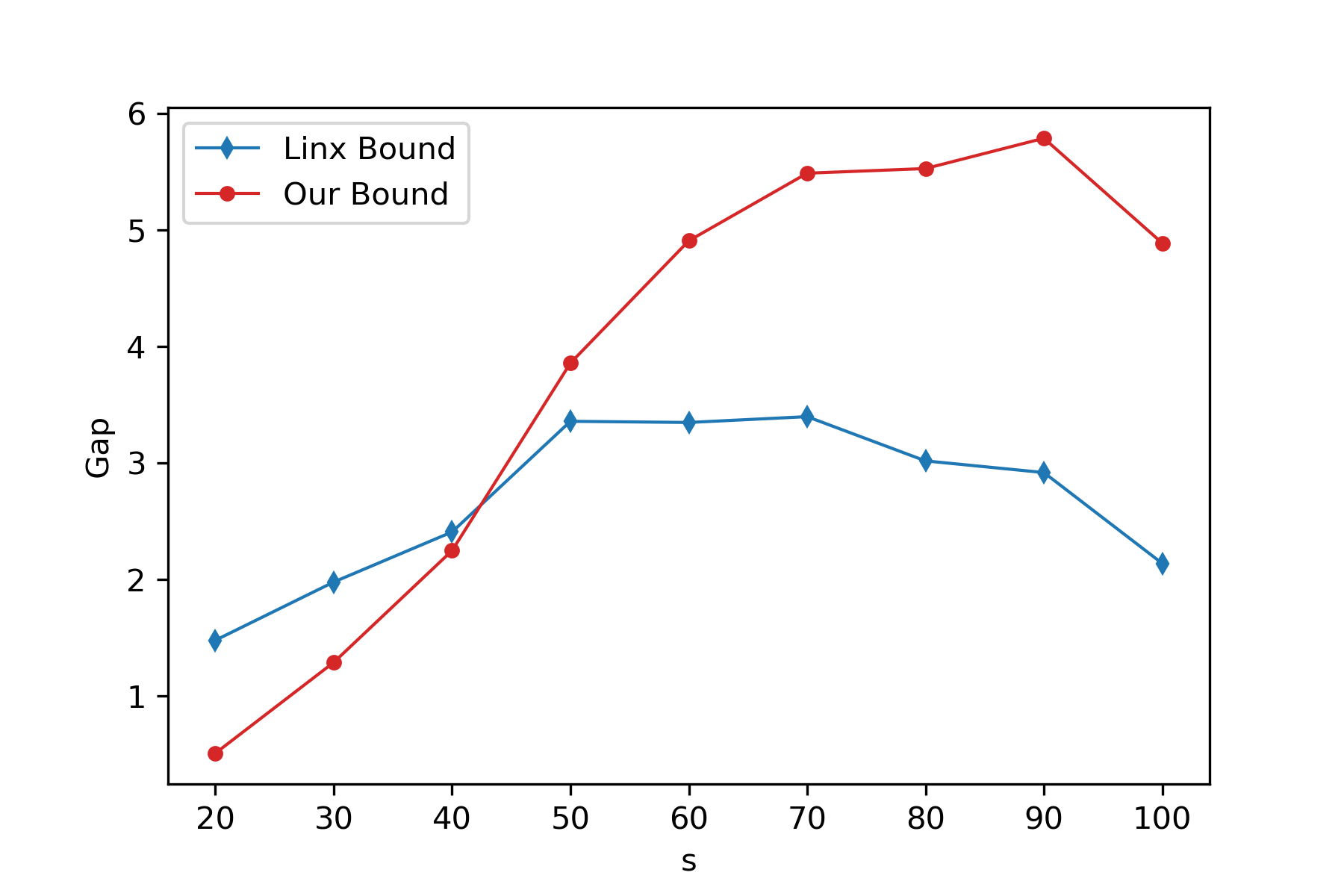}
	}
	\caption{Absolute optimality gap comparison of $z^{LD}$ and the linx bound in \cite{anstreicher2020efficient}.}\label{fig_comp_upper}
\end{figure}

\subsection{Numerical Experiments on a Large-scale Instance}
In this subsection, we test the proposed algorithms on a large-scale instance with a $2000\times2000$ covariance matrix $\bm{C}$ based upon Reddit data from \cite{dey2022using}. Note that for this instance, the matrix $\bm{C}$ is singular, and its rank is equal to 949, i.e., $d=949<n=2000$. The computational results are displayed in \Cref{data124}, where we use \textbf{B\&C} to denote the branch and cut algorithm, use \textbf{UB} to denote the best upper bound output from B\&C algorithm, and use UB to compute the absolute optimality gaps for the sampling Algorithm \ref{alg_rand_round}, the local search \Cref{algo}, and their combination. The lower bound of the B\&C algorithm is always inferior to the one found by the local search \Cref{algo} and is thus not reported.

We make the following remarks of the implementation of B\&C: (i) we use the warm start, i.e., we solve the continuous relaxation of MESP \eqref{eq_p} using the cutting-plane method (i.e., at each iteration, we add a supgradient inequality) and add all the cuts into the root node, (ii) if we encounter a solution $\hat{\bm x}$ with support $\hat{S}$ such that its corresponding columns $\{\bm{v}_i\}_{i\in\hat{S}}$ are not linearly independent, then the supgradient according to \Cref{supdiff} is not well-defined, and thus we add no-good cut to cut it off, which is in the form of $1\le \sum_{i \in \hat{S}} (1-{x}_i) + \sum_{i \in [n]\backslash\hat{S}} x_i,$
and (iii) we set the time limit to be $3,600$ seconds.

In \Cref{data2000}, it is expected that the B\&C algorithm has difficulty in solving MESP to optimality; however, it produces a better upper bound than $z^{LD}$. {Note that in the sampling algorithm, we only sample from the support of the output solution from the Frank-Wolfe \Cref{algo:fw} for the sake of computational efficiency.} 
For the proposed integrated ``Samp + LS'' algorithm in \Cref{data2000}, the running time is limited to be 3,600 seconds for each case.  Since we use UB to compute the optimality gaps of the approximation algorithms, their true optimality gaps can be even smaller. We also observe that the solution output from the Frank-Wolfe \Cref{algo:fw} is very sparse. The computational time of the Frank-Wolfe \Cref{algo:fw} is longer because at each iteration, one has to compute the eigendecomposition in order to obtain the supgradient, which can be time-consuming. It is seen that the local search \Cref{algo} outperforms the sampling \Cref{alg_rand_round} and the integrated algorithm  in both time and solution quality. In particular, ``-" in the last row of ``Samp + LS" column means infeasible output, i.e., the selected vectors by the integrated algorithm are linearly dependent with the output objective value being $-\infty$, which is possibly because the original matrix is rank-deficient and the Frank-Wolfe \Cref{algo:fw} selects many linearly dependent vectors.
Thus, we recommend using the vanilla local search \Cref{algo} to solve large-scale problems, with more stable output and lower computational cost. 

\begin{table}[h] 
	\centering
	\caption{Computational results of MESP on the $n=2000$ instance} 
	\begin{threeparttable}
		\setlength{\tabcolsep}{2.5pt}\renewcommand{\arraystretch}{1.3}
		\begin{tabular}{c|r r|r r r|r r | r r | r r }			
			\hline  
			\multicolumn{1}{c|}{ $n$=2000}   & \multicolumn{2}{c|}{B\&C} & \multicolumn{3}{c|}{Frank-Wolfe } & \multicolumn{2}{c|}{Sampling} & \multicolumn{2}{c|}{Local Search} & \multicolumn{2}{c}{Samp + LS}
			\\ \cline{2-12} 
			$s$		& \multicolumn{1}{c}{UB}  &\multicolumn{1}{c|}{{time\tnote1 } }& \multicolumn{1}{c}{$z^{LD}$} & \multicolumn{1}{c}{S-FW}  & \multicolumn{1}{c|}{ time }&  \multicolumn{1}{c}{gap} & \multicolumn{1}{c|}{time}  &\multicolumn{1}{c}{gap}  & \multicolumn{1}{c|}{ time }  &\multicolumn{1}{c}{gap}  & \multicolumn{1}{c}{ time }   \\  
			\hline
			20 &102.939  &3600 &103.007  & 30&119   & 0.331 & 232    & {0.037}  & 21 &{0.037}  &   1506    \\
			40 & 185.327 & 3600& 185.332   & 61 & 257 & 0.915  & 359  &{0.233}  & 23   & 0.233 & 2852      \\
			60 & 256.584  & 3600  &256.589& 93 & 321  & 2.415 &  463  &{0.303}  & 33 &0.303 &3600       \\
			80 & 320.812 & 3600&320.817 &160  & 833    & 4.384 & 950 &{0.612}  &41  & 0.612 & 3600     \\
			100 &  380.298 &3600  &380.307 &214  &1466  & 9.570 &1333  &{1.217}  & 52    & 1.217 & 3600      \\
			120 & 436.336 & 3600&436.350 &268  & 1935 & 18.478 &1973 &{1.850}  & 72  & - & 3600        \\
			\hline
		\end{tabular}%
		\label{data2000}
		\begin{tablenotes}
			\item[1] Time is in seconds
		\end{tablenotes} 
	\end{threeparttable}
\end{table}

\subsection{Stability of MESP}
The MESP \eqref{mesp}, selecting optimal $s$ random observations out of $n$ candidates,  depends on the knowledge of the covariance matrix $\bm C$. When the true covariance matrix is not known, we propose to use the sample covariance matrix whose accuracy is highly influenced by the sample size and noise level. In this subsection, we test the stability of the MESP \eqref{mesp} using the sample covariance matrix instead of the true one for the same benchmark instance as that in \Cref{data124}. Particularly, given the true covariance matrix $\bm C$ (i.e., the one used in \Cref{data124}), we generate $N$ i.i.d. samples following the Gaussian distribution with the corrupted covariance matrix, i.e., $\N(\bm 0, \bm C + \omega \bm \Sigma)$, where $\bm \Sigma\succeq 0$ is the corruption part of the covariance matrix and $\omega \ge 0$ is the corruption scalar. For the notational convenience, let us denote the sample covariance matrix built on $N$ i.i.d. samples as $\hat{\bm C}(N,  \omega)$. Let $S^*, \hat{S}(N,  \omega)$ denote the optimal solutions  of MESP \eqref{mesp} using $\bm C$ and $\hat{\bm C}(N,  \omega)$, respectively. Let us compute the false alarm rate of the optimal solution using the sample covariance as
$s^{-1}|S^* \setminus \hat{S}(N,  \omega)|$
 and its absolute gap of the optimal value as
 $|\log \det \left(\bm C_{S^*, S^*}\right) - \log \det ((\hat{\bm C}(N,  \omega))_{\hat{S}(N,  \omega), \hat{S}(N,  \omega)})|.$

\Cref{fig:fr1} presents the 95\% confidence intervals of false alarm rate and absolute gap for the case with $n=124, s=50$, which are computed by repeating the sampling procedure one hundred times. We see that as expected, the false alarm rate and absolute gap reduce to zero as sample size $N$ grows when there is no corruption (i.e., $\omega=0$), implying that the optimal solution and optimal value of MESP \eqref{mesp} using the sample covariance are closer to the true optimal ones as sample size increases. However, when there is a corruption (i.e., $\omega>0$),  the sample covariance matrix  converges to the corrupted covariance matrix, i.e., $\bm C +\omega \bm \Sigma$. Hence, its corresponding optimal solution and optimal value are close to the corrupted ones instead of true optimality. Therefore, in \Cref{fig:fr1}, it is expected that the false alarm rate or absolute gap does not vanish to zero as the sample size increases. Nevertheless, the obtained solutions based on the corrupted covariance matrix are still quite close to the optimal one of the true MESP \eqref{mesp} as shown in \Cref{fig:fra}.

\begin{figure}[t]
	\centering
	\subfigure[False alarm rate]{
		\includegraphics[width=.46\textwidth]{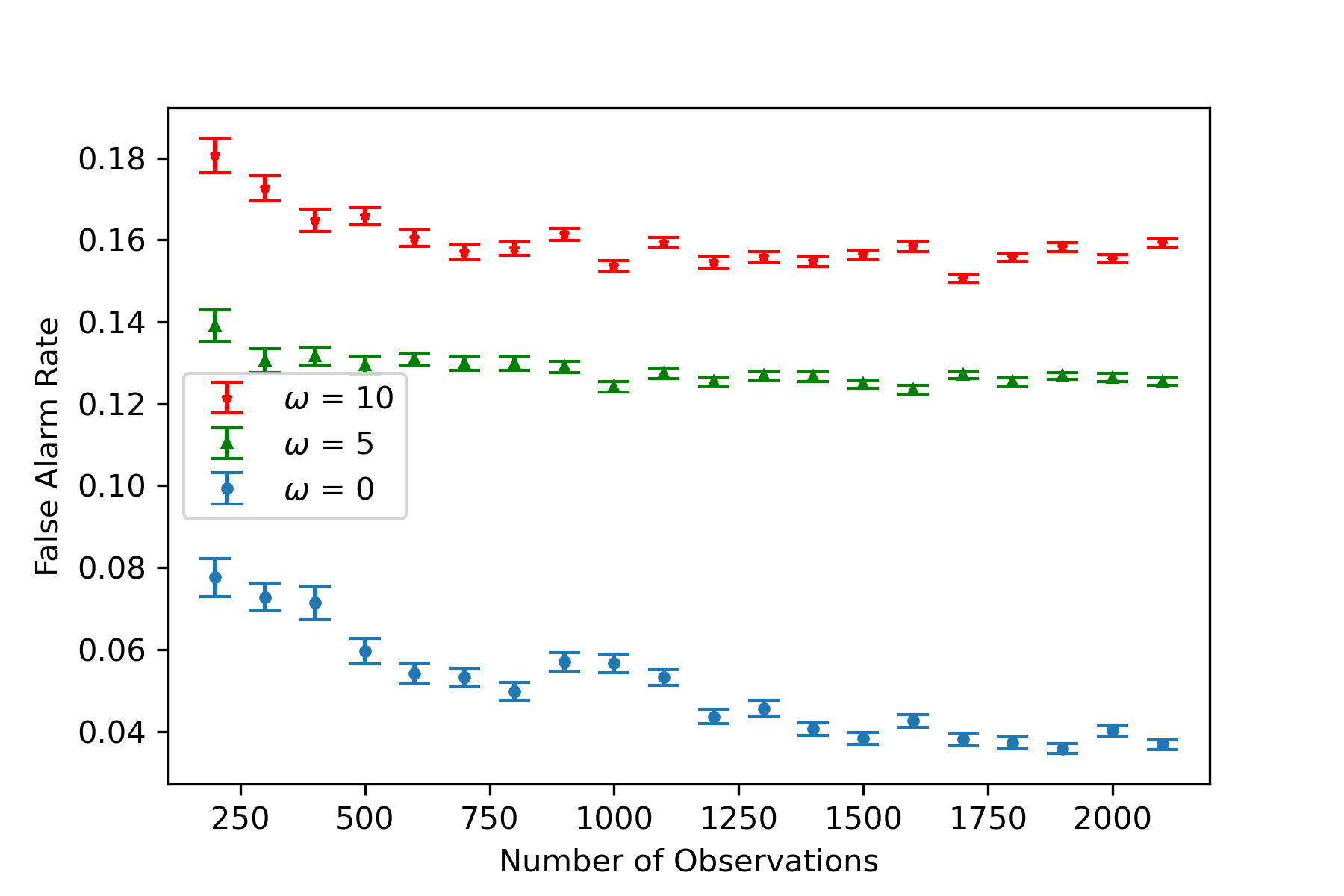} \label{fig:fra}
	}
	\subfigure[Absolute gap of objective values]{
		\includegraphics[width=.45\textwidth]{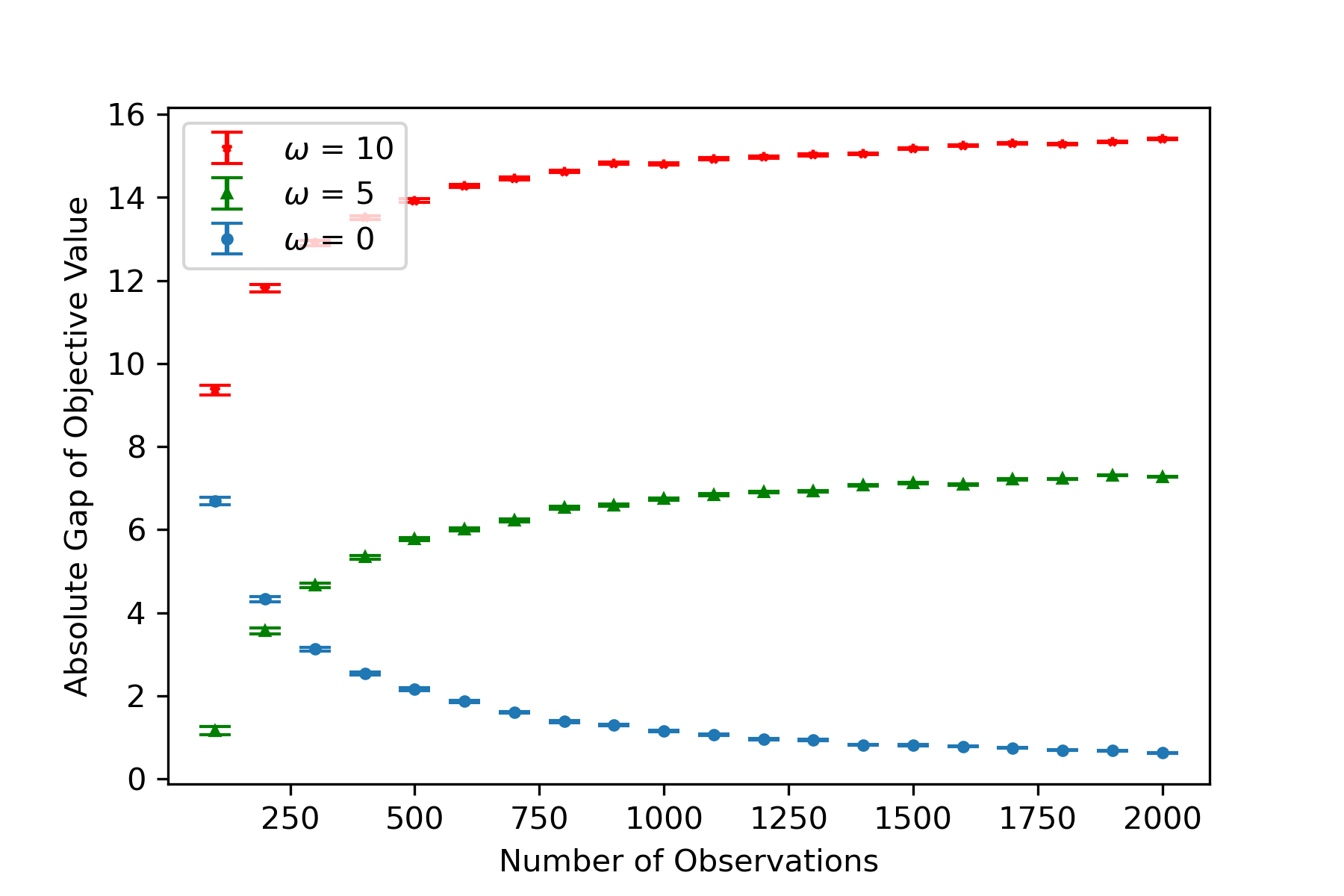}
	}
	\caption{ 95\% confidence intervals of false alarm rate and absolute gap for  $ s=50$ case  on $n=124$ instance}
	\label{fig:fr1}
\end{figure}

\section{Extension to the A-Optimal MESP (A-MESP)} \label{sec:amesp}
In the section, we extend the analyses to the A-Optimal MESP (A-MESP), which instead, minimizes the trace of the inverse of matrix $\bm{C}_{S, S}$. The A-Optimality, as an alternative measurement of information, has been widely used in the fields of experimental design \citep{madan2019combinatorial,nikolov2019proportional}, subdata selection \citep{yao2019optimal}, and sensor placement \citep{moreno2013sensor,xu2017optimal}. Formally, A-MESP is formulated as
\begin{align}
\textrm{(A-MESP)} \quad z^*_A := \min_{S}  \left\{ \tr \left(\bm{C}_{S,S}^{-1}\right):  \ \ S \subseteq [n] , |S|=s \right\}. \label{a-mesp}
\end{align}
By default, if $\bm{C}_{S,S}$ is singular, then $\tr \left(\bm{C}_{S,S}^{-1}\right)=\infty$.
\subsection{Convex Integer Programming Formulation}
This subsection derives an equivalent convex integer program for A-MESP \eqref{a-mesp}. 

First, we introduce the following three functions, corresponding to the objective function of another exact formulation for A-MESP \eqref{a-mesp}, the objective function of the Lagrangian dual, and the objective function of the primal characterization, respectively.
\begin{definition} \label{def:amesp}
	For a $d\times d$ matrix $\bm{X} \succeq 0$ of its eigenvalues $\lambda_1 \ge \cdots \ge \lambda_d \ge 0$, let us denote
	\begin{enumerate}[(i)]
		\item $\overset{s}{\tr} (\bm{X}^{\dag}) := \sum_{i\in [s]}\frac{1}{\lambda_{i}}$,
		\item $ \underset{s}{\tr} (\bm{X}) := \sum_{i\in [d-s+1,d]}{\lambda_{i}}$,
		\item $\Phi_s(\bm{X}) := \sum_{i\in [k]} \frac{1}{\lambda_{i}}+(s-k) \frac{s-k}{\sum_{i\in [k+1,d]} \lambda_{i}} $, where the unique integer $k$ is defined in Lemma~\ref{lem:kappa}.
	\end{enumerate}
	
\end{definition}

Similar to \Cref{claim:cholesky}, it is straightforward to show that 
$\tr\left ( \bm{C}_{S,S}^{-1}\right )= \overset{s}{\tr} [(\sum_{i \in S} \bm{v}_i \bm{v}_i^{\top})^{\dag} ].$
Thus, A-MESP \eqref{a-mesp} can be reformulated as
\begin{align}
\textrm{(A-MESP)} \quad z^*_{A} := \min_{\bm{x}}\Bigg \{ \overset{s}{{\tr}} \bigg[ \bigg(\sum_{i \in [n]}  x_i \bm{v}_i\bm{v}_i^\top \bigg)^{\dag} \bigg]: \sum_{i \in [n]} x_i = s, \bm{x} \in \{0,1\}^{n} \Bigg \}, \label{a_eq_obj}
\end{align}
which reduces to the conventional A-Optimal design problem \citep{madan2019combinatorial,nikolov2019proportional} if $d \leq s\leq n$. The following proposition summarizes the properties of the objective function of A-MESP \eqref{a_eq_obj}.
\begin{restatable}{proposition}{propaobj} \label{prop:aobj}
	The objective function of A-MESP \eqref{a_eq_obj} is (i) monotonic non-decreasing, (ii) neither discrete-supermodular nor discrete-submodular,  and (iii) neither convex nor concave. 
\end{restatable}
\begin{proof}
	See Appendix \ref{proof_prop_aobj}. \qed
\end{proof}

To derive an equivalent convex integer program, we introduce a matrix variable  $\bm{X} \in \Re^{d \times d}$ and reformulate A-MESP \eqref{a_eq_obj} as 
\begin{align}
\textrm{(A-MESP)} \quad z^*_{A} := \min_{\bm{x}, \bm{X} \succeq 0}\Bigg \{ \overset{s}{{\tr}} (\bm{X}^{\dag} ): \sum_{i \in [n]}x_i \bm{v}_i \bm{v}_i^{\top} \succeq \bm{X},\sum_{i \in [n]} x_i = s, \bm{x} \in \{0,1\}^{n} \Bigg \}. \label{a_eq_obje}
\end{align}
The key idea of deriving the convex integer program is summarized as follows: (i) obtain Lagrangian dual of A-MESP \eqref{a_eq_obje} by dualizing the constraint $\sum_{i \in [n]}x_i \bm{v}_i \bm{v}_i^{\top} \succeq \bm{X}$; (ii) characterize the primal formulation of the Lagrangian dual; and (iii) enforce the continuous variables in the primal characterization to be binary.
To begin with, we introduce the following lemma, which is essential to derive the Lagrangian dual of A-MESP. 
\begin{restatable}{lemma}{lemaldmax} \label{lem:aldmax}
	For a $d \times d$ matrix $\bm{\Lambda} \succeq 0$, we have
	\begin{align}
	\min_{\bm{X} \succeq 0} \left\{\overset{s}{\tr}(\bm{X}^{\dag})+\tr(\bm{X}\bm{\Lambda}) \right \} = 2 \underset{s}{\tr}\left(\bm{\Lambda}^{\frac{1}{2}}\right). \label{eq_relaxed_A1}
	\end{align}	
\end{restatable}
\begin{proof}
	See Appendix \ref{proof_lem_aldmax}. \qed
\end{proof}
Next, we are going to show the Lagrangian dual of A-MESP \eqref{a_eq_obje}, denoted by A-LD.
\begin{restatable}{theorem}{thmald} \label{them:ald}
	The Lagrangian dual of A-MESP \eqref{a_eq_obj} is 
	\begin{align}
	\textrm{(A-LD)} \quad z^{LD}_A := \max_{\bm{\Lambda} \succeq 0, \nu, \bm{\mu} \in \Re_{+}^n } \bigg\{
	2 \underset{s}{\tr}\left(\bm{\Lambda}^{\frac{1}{2}}\right) - s \nu - \sum_{i \in [n]} {\mu_i}: \nu + \mu_i \ge \bm{v}_i^\top \bm{\Lambda} \bm{v}_i, i \in [n] \bigg\}, \label{a_eq_dual}
	\end{align}
	and its optimal value is a lower bound of A-MESP, i.e., $z^{LD}_A \le z^*_{A}$.
\end{restatable}

\begin{proof}
	By dualizing the first constraint of A-MESP  \eqref{a_eq_obje}, we can formulate the dual problem as
	\begin{align*}
	z^{LD}_{A} := \max_{\bm{\Lambda} \succeq 0}\Bigg\{\min_{\bm{x}, \bm{X} \succeq 0} \bigg\{ \overset{s}{\tr}\left(\bm{X}^{\dag}\right)+\tr(\bm{X}\bm{\Lambda}) - \sum_{i \in [n]}x_i \bm{v}_i^{\top} \bm{\Lambda}\bm{v}_i: \sum_{i \in [n]}x_i =s, \bm{x} \in \{0,1\}^n \bigg\}  \Bigg\}.
	\end{align*}
	Applying Lemma \ref{lem:aldmax} to the inner minimization problem over $\bm{X}$, the dual problem becomes
	\begin{align*}
	z^{LD}_{A} := &\max_{\bm{\Lambda} \succeq 0}\Bigg\{\min_{\bm{x}} \bigg\{ 2 \underset{s}{\tr}\left(\bm{\Lambda}^{\frac{1}{2}}\right) - \sum_{i \in [n]}x_i \bm{v}_i^{\top} \bm{\Lambda}\bm{v}_i: \sum_{i \in [n]}x_i =s, \bm{x} \in \{0,1\}^n \bigg\}  \Bigg\}.
	\end{align*}
	Similarly, we derive the dual of minimization problem over $\bm{x}$ and combine the dual with the maximization over $\bm{\Lambda}$, which obtains A-LD problem. Apparently, $z_{A}^{LD} \le z_{A}^*$ by weak duality. \qed
\end{proof}

In addition, A-LD \eqref{a_eq_dual} has an equivalent primal characterization.
\begin{restatable}{theorem}{thmapcont} \label{thm:apcont}
	The primal characterization of A-LD \eqref{a_eq_dual}, referred to as (A-PC), is
	\begin{align}
	\textrm{(A-PC)} \quad z^{LD}_A := \min_{\bm{x}} \Bigg \{ \Phi_s \bigg(\sum_{i\in [n]} x_i \bm{v}_i \bm{v}_i^{\top}\bigg):
	\sum_{i \in [n]} x_i = s,
	\bm{x} \in [0,1]^n \Bigg \}. \label{a_eq_pcont}
	\end{align}
\end{restatable}
\begin{proof}
	See Appendix \ref{proof_them_apcont}. \qed
\end{proof}

As a side product of Theorem \ref{thm:apcont}, we can obtain the subdifferentials of the convex but non-smooth objective function $\Phi_s(\cdot)$ for A-PC \eqref{a_eq_pcont}.
\begin{restatable}{proposition}{subdiff} \label{subdiff}
	Given a $d \times d$ matrix $\bm{X} \succeq 0$ with rank $r\geq s$, suppose the vector of eigenvalues of $\bm{X}$ is $\bm{\lambda}$ such that $\lambda_1 \ge \cdots \ge \lambda_r > \lambda_{r+1}=\cdots = \lambda_d =0$ and ${\bm{X}} = \bm{Q} \Diag(\bm{\lambda}) \bm{Q}^{\top}$ with an orthonormal matrix $\bm{Q}$. Then the subdifferential of function $\Phi_s(\cdot)$ at $\bm{X}$ that is denoted by $\partial \Phi_s(\bm{X})$ is 
	\begin{align*}
	&\partial \Phi_s(\bm{X}) =\Bigg\{\bm{Q} \Diag(\bm{\beta}) \bm{Q}^{\top}: \bm{X} = \bm{Q} \Diag(\bm{\lambda}) \bm{Q}^{\top}, \bm{Q} \textrm{\rm\ is orthonormal},\\
	&\bm{\beta} \in \conv \bigg \{\bm{\beta}: \beta_i = \frac{1}{\lambda_i}, \forall i \in [k], \beta_i = \frac{s-k}{\sum_{i\in [k+1,d]} \lambda_{i}},\forall i \in [k+1, r],\beta_i \ge \beta_r,\forall i \in [r+1,d] \bigg \} \Bigg\}.
	\end{align*}
	Note that the subdifferential of $\Phi_s(\cdot)$ above is unique and becomes the gradient when $\bm{X} \succ 0$ is non-singular.
\end{restatable}
\begin{proof}
	The proof is similar to that of Proposition \ref{supdiff} and is thus omitted here.\qed
\end{proof}

Another side product is that we obtain an equivalent convex integer program of A-MESP by enforcing the variables $\bm{x}$ in A-PC \eqref{a_eq_pcont} to be binary.
\begin{restatable}{theorem}{thm:apc}
	The A-MESP is equivalent to the following convex integer program
	\begin{align}
	\textrm{(A-MESP)} \quad z^*_{A}:= \min_{\bm{x}} \Bigg \{ \Phi_s \bigg(\sum_{i\in [n]} x_i \bm{v}_i \bm{v}_i^{\top}\bigg):
	\sum_{i \in [n]} x_i = s,
	\bm{x} \in \{0,1\}^n \Bigg \}. \label{a_eq_p}
	\end{align}
\end{restatable}

\begin{proof}
	The proof is similar to that of Theorem \ref{primal} and is thus omitted. 
	\qed
\end{proof}

\subsection{Volume Sampling Algorithm} 
In this subsection, we present a polynomial-time volume sampling algorithm for A-MESP, which has been applied to the generalized A-Optimal design \citep{derezinski2017unbiased,nikolov2019proportional}. 
A size-$s$ subset $S\subseteq [n]$ is sampled with the probability
$$\mathbb{P}[\tilde{S}= S] := \frac{\prod_{i \in S}\hat{x}_i \overset{s}{\det}(\sum_{i \in S} \bm{v}_i \bm{v}_i^{\top})}{\sum_{\bar{S}\in \binom{[n]}{s}} \prod_{i \in \bar{S}} \hat{x}_i \overset{s}{\det}(\sum_{i \in \bar{S}} \bm{v}_i \bm{v}_i^{\top})}.$$
Different from the sampling Algorithm \ref{alg_rand_round}, this probability formula, known as volume sampling, delivers the proportional volume spanned by the selected vectors. Algorithm \ref{alg:volsamp} describes an efficient implementation of this volume sampling algorithm, with running time complexity $O(n^5)$. 

Next, we analyze the approximation ratio of the volume sampling Algorithm \ref{alg:volsamp}. We start with the following observation.
\begin{algorithm}[ht]
	\caption{Efficient Implementation of Volume Sampling Procedure }
	\label{alg:volsamp}
	\begin{algorithmic}[1]
		\State \textbf{Input:} $n\times n$ matrix $\bm{C} \succeq 0$ of rank $d$ and integer $s \in [d]$
		\State Let $\bm{\hat x}$ is an optimal solution of A-PC 
		\State Initialize chosen set $\tilde{S}=\emptyset$ and unchosen set $T= \emptyset$
		\State Two factors: $A_1=\sum_{{S}\in \binom{[n]}{s}} \left(\prod_{i \in {S}} \hat{x}_i\right) {\det}\left(\bm{V}_{S}^{\top} \bm{V}_S\right),A_2=0$
		\For{$j=1,\cdots,n$}
		\State Let $A_2=\sum_{{S}\in \binom{[n]}{s}, \tilde{S} \subseteq S, T\cap S = \emptyset} \left(\prod_{i \in {S}} \hat{x}_i\right) {\det}\left(\bm{V}_{S}^{\top} \bm{V}_S\right)$
		\State Sample a $(0,1)$ uniform random variable $U$
		\If{$A_2/A_1 \geq U$}
		\State Add $j$ to set $\tilde{S}$
		\State $A_1=A_2$
		\Else
		\State Add $j$ to set $T$
		\State $A_1=A_1-A_2$
		\EndIf
		\EndFor
		\State Output $\tilde{S}$%
	\end{algorithmic}
\end{algorithm}

\begin{restatable}{lemma}{lemlowbnd} \label{lem:lowbnd}
	For any feasible solution $\bm{x} $ to A-PC \eqref{a_eq_pcont}, let $\bm{\lambda} \in \Re_{+}^d$ denote the vector of eigenvalues of matrix $\sum_{i \in [n]} x_i \bm{v}_i \bm{v}_i^{\top}$, then we have
	\begin{align}
	\Phi_s \bigg(\sum_{i \in [n]} x_i \bm{v}_i \bm{v}_i^{\top}\bigg) \ge \frac{E_{s-1}(\bm{\lambda})}{E_{s}(\bm{\lambda})}, \label{eq_relaxation_A_PC}
	\end{align}
	where function $E_{s}(\cdot)$ is introduced in \Cref{def:elem}.
\end{restatable}
\begin{proof}
	See Appendix \ref{proof_lem_lowbnd}. \qed
\end{proof}

Observe that the right-hand side of the inequality \eqref{eq_relaxation_A_PC} is equivalent to the relaxation bound of A-MESP proposed by \cite{nikolov2019proportional}. Hence, Lemma \ref{lem:lowbnd} also indicates that our proposed bound is stronger than the existing one. The following theorem shows that we further improve the approximation ratio of the volume sampling Algorithm \ref{alg:volsamp}.
\begin{restatable}{theorem}{themasampling} \label{them:asamp}
	Given an optimal solution $\bm{\hat x}$ to A-PC, the volume sampling Algorithm \ref{alg:volsamp}
	yields a $\min(s, n-s+1)$-approximation ratio of A-MESP, i.e.,
	$$\mathbb{E} \bigg[ \overset{s}{{\tr}} \bigg[ \bigg(\sum_{i \in \tilde{S}}  \bm{v}_i\bm{v}_i^\top \bigg)^{\dag} \bigg] \bigg] \le \min(s,n-s+1) z_{A}^*.$$
\end{restatable}
\begin{proof}
	See Appendix \ref{proof_them_asampling}. \qed
\end{proof}

Note that this approximation ratio improves the one stated in theorem A.3 \citep{nikolov2019proportional}, in particular, if $s\geq \frac{n+1}{2}$, our approximation ratio is strictly better. Since we use the same volume sampling procedure, its deterministic implementation follows exactly from appendix B in \cite{nikolov2019proportional} and is thus omitted here.

\subsection{Local Search Algorithm for A-MESP}
This subsection analyzes the local search algorithm to solve A-MESP, which is presented in Algorithm~\ref{algo:aloc}. The efficient implementation straightforwardly follows from the local search Algorithm \ref{algo:effls} in Section \ref{sec:loc} and is thus omitted. Therefore, we mainly focus on deriving the approximation ratio of the local search Algorithm~\ref{algo:aloc}.
\begin{algorithm}[htbp]
	\caption{Local Search Algorithm} \label{algo:aloc}
	\begin{algorithmic}[1]
		\State \textbf{Input:} $n\times n$ matrix $\bm{C} \succeq 0$ {of rank $d$ and integer $s \in [d]$}
		\State Let $\bm{C}=\bm{V}^{\top}\bm{V}$ denote its Cholesky factorization where $\bm{V} \in \Re^{d\times n}$
		\State Let $\bm{v}_i \in \Re^d$ denote the $i$-th column vector of matrix $\bm{V}$ for each $i \in [n]$
		\State Initial subset $\hat{S} \subseteq [n]$ of size $s$ such that $\{\bm{v}_i\}_{i\in \hat{S}}$ are linearly independent
		\Do
		\For{each pair {$(i,j) \in \hat{S} \times ([n]\setminus \hat{S})$}}
		\If{$\overset{s}{\tr} \left(\sum_{i \in \hat{S} \cup \{j\} \setminus \{i\}} \bm{v}_i\bm{v}_i^{\top}\right) <  \overset{s}{\tr} \left(\sum_{i \in \hat{S}} \bm{v}_i\bm{v}_i^{\top}\right)$}
		\State Update $\hat{S} := \hat{S} \cup \{j\} \setminus \{i\}$
		\EndIf
		\EndFor
		\doWhile{there is still an improvement}
		\State \textbf{Output:} $\hat S$%
	\end{algorithmic}
\end{algorithm}

Let us begin with the following local optimality condition for the \Cref{algo:aloc}.
\begin{restatable}{lemma}{lemarankupd} \label{lem:arankupd}
	Suppose that $\hat{S}$ is the output of the local search Algorithm \ref{algo:aloc} and $\bm{X}=\sum_{i \in \hat{S}}\bm{v}_i\bm{v}_i^{\top}$, for each pair $(i,j)\in \hat{S} \times ([n]\setminus \hat{S})$, the following inequality always holds
	\begin{align*}
	\bm{v}_i^{\top} (\bm{X}^{\dag})^3 \bm{v}_i \bm{v}_j^{\top} (\bm{I}_n- \bm{X}^{\dag}\bm{X})\bm{v}_j 
	\le \bm{v}_i^{\top} (\bm{X}^{\dag})^2 \bm{v}_i + \bm{v}_i^{\top} (\bm{X}^{\dag})^2 \bm{v}_i\bm{v}_j^{\top} \bm{X}^{\dag} \bm{v}_j- 2\bm{v}_i^{\top} (\bm{X}^{\dag})^2 \bm{v}_j\bm{v}_i^{\top} \bm{X}^{\dag} \bm{v}_j.
	\end{align*}
\end{restatable}
\begin{proof}
	See Appendix \ref{proof_lem_arankupd}. \qed
\end{proof}
The local optimality condition inspires us a construction of a feasible solution to A-LD \eqref{a_eq_dual}, which allows the weak duality to bound the output value from the local search \Cref{algo:aloc}.
\begin{restatable}{theorem}{themalocs} \label{them:alocs}
	The local search Algorithm \ref{algo:aloc} yields a $s/2 + {\delta}^{-1}\min\left\{ {\lambda_{\max}(\bm{C})},  n{\delta}+{(n-s)\lambda_{\max}(\bm{C})} \right\}$-approximation ratio for A-MESP, i.e, 
	$$\overset{s}{\tr}\bigg(\sum_{i \in \hat{S}} \bm{v}_i \bm{v}_i^{\top} \bigg) \le \min \left\{\frac{s}{2} \left(1+\frac{\lambda_{\max}(\bm{C})}{\delta}\right), \frac{1}{2} \left(n+s+\frac{(n-s)\lambda_{\max}(\bm{C})}{\delta}\right) \right\} z^*_A, $$
	where $\hat{S}$ is the set produced by Algorithm \ref{algo:aloc}, and $\delta$ is defined in \Cref{lem_hessian}.
\end{restatable}
\begin{proof}
	See Appendix \ref{proof_them_alocs}. \qed
\end{proof}

We remark that the result in \Cref{them:alocs} is the first-known approximation ratio of the local search Algorithm \ref{algo:aloc} for A-MESP. 
Finally, \Cref{table:amespbounds} summarizes the existing and our developed approximation ratios for A-MESP.

\begin{table}[htbp]
	\centering
	\caption{Summary of Approximation Algorithms for A-MESP} 
	\label{table:amespbounds}
		\setlength{\tabcolsep}{1.5pt}\renewcommand{\arraystretch}{1.6}
		\begin{tabular}{| c | c |c |}
			\hline
			&\multicolumn{1}{c|}{Algorithm} & \multicolumn{1}{c|}{Approximation Ratio}      \\  
			\hline
			{\textbf{Literature}}&Volume Sampling \citep{nikolov2019proportional}& $s$\\
			\hline
			\multirow{2}{5.5em}{\textbf{This paper}}	& Volume Sampling \Cref{alg:volsamp} & $\min\{s, n-s+1\} $  \\
			\cline{2-3}
			&Local Search \Cref{algo:aloc}& $ s/2 + {\delta}^{-1}\min\left\{ {\lambda_{\max}(\bm{C})},  n{\delta}+{(n-s)\lambda_{\max}(\bm{C})} \right\}$  \\
			\hline
		\end{tabular}%
\end{table}

\section{Conclusion} \label{sec:con}
This paper studies the maximum entropy sampling problem (MESP) and develops  two approximation algorithms with provable performance guarantees. Observing that the objective function of MESP is neither convex nor concave, we derive a new convex integer program for MESP through the Lagrangian dual relaxation and its primal characterization. 
Using the optimal solution of the primal characterization, we develop an efficient sampling algorithm and prove its approximation bound, which improves the best-known bound in literature. %
By developing new mathematical tools for the singular matrices and analyzing the Lagrangian dual of the proposed convex integer program, we further analyze the local search algorithm and prove its first-known approximation bound for MESP. The proof techniques that we developed inspire us an efficient implementation of the local search algorithm. Our numerical study shows that both algorithms work very well, and the local search algorithm performs the best and consistently yields near-optimal solutions. Finally, we extend all analyses to the A-Optimal MESP (A-MESP), develop a new convex integer program and study the volume sampling and local search algorithms with their approximation ratios. Our proposed algorithms are coded and released as open-source software. One possible future direction is to study MESP with general distributions.

\section*{Acknowledgment}
This research has been supported in part by the National Science Foundation grants 2246414 and 2246417. Valuable
comments from Prof. Jon Lee from the University of Michigan, the editors, and anonymous reviewers are gratefully appreciated. The authors also thank Prof. Kurt Anstreicher from the University of Iowa for providing the numerical instances.

\bibliography{reference.bib}

\begin{thebibliography}{72}
\providecommand{\natexlab}[1]{#1}
\providecommand{\url}[1]{\texttt{#1}}
\providecommand{\urlprefix}{URL }

\bibitem[{Abdi(2007)}]{abdi2007eigen}
Abdi H (2007) The eigen-decomposition: Eigenvalues and eigenvectors.
  \emph{Encyclopedia of Measurement and Statistics} 304--308.

\bibitem[{Alarifi et~al.(2019)Alarifi, AlZubi, Al-Maitah, \protect\BIBand{}
  Al-Kasasbeh}]{alarifi2019optimal}
Alarifi A, AlZubi AA, Al-Maitah M, Al-Kasasbeh B (2019) An optimal sensor
  placement algorithm (o-spa) for improving tracking precision of human
  activity in real-world healthcare systems. \emph{Computer Communications}
  148:9--16.

\bibitem[{Alon \protect\BIBand{} Spencer(2016)}]{alon2016probabilistic}
Alon N, Spencer JH (2016) \emph{The probabilistic method} (John Wiley \& Sons).

\bibitem[{Anstreicher(2018)}]{anstreicher2018maximum}
Anstreicher KM (2018) Maximum-entropy sampling and the boolean quadric
  polytope. \emph{Journal of Global Optimization} 72(4):603--618.

\bibitem[{Anstreicher(2020)}]{anstreicher2020efficient}
Anstreicher KM (2020) Efficient solution of maximum-entropy sampling problems.
  \emph{Operations Research} 68(6):1826--1835.

\bibitem[{Anstreicher et~al.(1996)Anstreicher, Fampa, Lee, \protect\BIBand{}
  Williams}]{anstreicher1996continuous}
Anstreicher KM, Fampa M, Lee J, Williams J (1996) Continuous relaxations for
  constrained maximum-entropy sampling. \emph{International Conference on
  Integer Programming and Combinatorial Optimization}, 234--248 (Springer).

\bibitem[{Anstreicher et~al.(1999)Anstreicher, Fampa, Lee, Williams
  et~al.}]{anstreicher1999using}
Anstreicher KM, Fampa M, Lee J, Williams J, et~al. (1999) Using continuous
  nonlinear relaxations to solve. constrained maximum-entropy sampling
  problems. \emph{Mathematical Programming} 85(2):221--240.

\bibitem[{Anstreicher \protect\BIBand{} Lee(2004)}]{anstreicher2004masked}
Anstreicher KM, Lee J (2004) A masked spectral bound for maximum-entropy
  sampling. \emph{mODa 7—Advances in Model-Oriented Design and Analysis},
  1--12 (Springer).

\bibitem[{Arellano-Valle et~al.(2013)Arellano-Valle, Contreras-Reyes,
  \protect\BIBand{} Genton}]{arellano2013shannon}
Arellano-Valle RB, Contreras-Reyes JE, Genton MG (2013) Shannon entropy and
  mutual information for multivariate skew-elliptical distributions.
  \emph{Scandinavian Journal of Statistics} 40(1):42--62.

\bibitem[{Bellman(1997)}]{bellman1997introduction}
Bellman R (1997) \emph{Introduction to matrix analysis}, volume~19 (Siam).

\bibitem[{Ben-Tal \protect\BIBand{} Nemirovski(2001)}]{ben2001lectures}
Ben-Tal A, Nemirovski A (2001) \emph{Lectures on modern convex optimization:
  analysis, algorithms, and engineering applications}, volume~2 (Siam).

\bibitem[{Ben-Tal \protect\BIBand{} Nemirovski(2012)}]{ben2012optimization}
Ben-Tal A, Nemirovski A (2012) Optimization iii: Convex analysis, nonlinear
  programming theory, nonlinear programming algorithms. \emph{Lecture Notes}
  34.

\bibitem[{Bertsekas(1982)}]{bertsekas1982constrained}
Bertsekas DP (1982) \emph{Constrained optimization and Lagrange multiplier
  methods} (Academic press).

\bibitem[{Broida \protect\BIBand{} Williamson(1989)}]{broida1989comprehensive}
Broida JG, Williamson SG (1989) \emph{A comprehensive introduction to linear
  algebra} (Addison-Wesley Redwood City, CA).

\bibitem[{Bueso et~al.(1998)Bueso, Angulo, \protect\BIBand{}
  Alonso}]{bueso1998state}
Bueso M, Angulo J, Alonso F (1998) A state-space model approach to optimum
  spatial sampling design based on entropy. \emph{Environmental and Ecological
  Statistics} 5(1):29--44.

\bibitem[{Burer \protect\BIBand{} Lee(2007)}]{burer2007solving}
Burer S, Lee J (2007) Solving maximum-entropy sampling problems using factored
  masks. \emph{Mathematical Programming} 109(2-3):263--281.

\bibitem[{Charikar et~al.(2000)Charikar, Guruswami, Kumar, Rajagopalan,
  \protect\BIBand{} Sahai}]{charikar2000combinatorial}
Charikar M, Guruswami V, Kumar R, Rajagopalan S, Sahai A (2000) Combinatorial
  feature selection problems. \emph{Proceedings 41st Annual Symposium on
  Foundations of Computer Science}, 631--640 (IEEE).

\bibitem[{Christodoulou(2015)}]{christodoulou2015smarting}
Christodoulou S (2015) Smarting up water distribution networks with an
  entropy-based optimal sensor placement strategy. \emph{Journal of Smart
  Cities} 1(1):47--58.

\bibitem[{{\c{C}}ivril \protect\BIBand{}
  Magdon-Ismail(2009)}]{ccivril2009selecting}
{\c{C}}ivril A, Magdon-Ismail M (2009) On selecting a maximum volume sub-matrix
  of a matrix and related problems. \emph{Theoretical Computer Science}
  410(47-49):4801--4811.

\bibitem[{Civril \protect\BIBand{} Magdon-Ismail(2013)}]{civril2013exponential}
Civril A, Magdon-Ismail M (2013) Exponential inapproximability of selecting a
  maximum volume sub-matrix. \emph{Algorithmica} 65(1):159--176.

\bibitem[{Cover \protect\BIBand{} Thomas(2012)}]{cover2012elements}
Cover TM, Thomas JA (2012) \emph{Elements of information theory} (John Wiley \&
  Sons).

\bibitem[{de~Aguiar et~al.(1995)de~Aguiar, Bourguignon, Khots, Massart,
  \protect\BIBand{} Phan-Than-Luu}]{de1995d}
de~Aguiar PF, Bourguignon B, Khots M, Massart D, Phan-Than-Luu R (1995)
  D-optimal designs. \emph{Chemometrics and Intelligent Laboratory Systems}
  30(2):199--210.

\bibitem[{Derezinski \protect\BIBand{} Warmuth(2017)}]{derezinski2017unbiased}
Derezinski M, Warmuth MK (2017) Unbiased estimates for linear regression via
  volume sampling. \emph{Advances in Neural Information Processing Systems},
  3084--3093.

\bibitem[{Dey et~al.(2022)Dey, Mazumder, \protect\BIBand{} Wang}]{dey2022using}
Dey SS, Mazumder R, Wang G (2022) Using $\ell_1$-relaxation and integer
  programming to obtain dual bounds for sparse pca. \emph{Operations Research}
  70(3):1914--1932.

\bibitem[{Fan(1949)}]{fan1949theorem}
Fan K (1949) On a theorem of weyl concerning eigenvalues of linear
  transformations i. \emph{Proceedings of the National Academy of Sciences of
  the United States of America} 35(11):652.

\bibitem[{Freund \protect\BIBand{} Grigas(2016)}]{freund2016new}
Freund RM, Grigas P (2016) New analysis and results for the frank--wolfe
  method. \emph{Mathematical Programming} 155(1-2):199--230.

\bibitem[{Gilmore(1996)}]{gilmore1996maximum}
Gilmore CJ (1996) Maximum entropy and bayesian statistics in crystallography: a
  review of practical applications. \emph{Acta Crystallographica Section A:
  Foundations of Crystallography} 52(4):561--589.

\bibitem[{Guruswami \protect\BIBand{} Sinop(2012)}]{guruswami2012optimal}
Guruswami V, Sinop AK (2012) Optimal column-based low-rank matrix
  reconstruction. \emph{Proceedings of the twenty-third annual ACM-SIAM
  symposium on Discrete Algorithms}, 1207--1214 (SIAM).

\bibitem[{Hardy et~al.(1952)Hardy, Littlewood, \protect\BIBand{}
  Polya}]{hardy1952inequalities}
Hardy G, Littlewood J, Polya G (1952) Inequalities cambridge univ. \emph{Press,
  Cambridge} (1988).

\bibitem[{Harville(1998)}]{harville1998matrix}
Harville DA (1998) Matrix algebra from a statistician's perspective.

\bibitem[{Hazimeh \protect\BIBand{} Mazumder(2020)}]{hazimeh2020fast}
Hazimeh H, Mazumder R (2020) Fast best subset selection: Coordinate descent and
  local combinatorial optimization algorithms. \emph{Operations Research}
  68(5):1517--1537.

\bibitem[{Hoch et~al.(2014)Hoch, Maciejewski, Mobli, Schuyler,
  \protect\BIBand{} Stern}]{hoch2014nonuniform}
Hoch JC, Maciejewski MW, Mobli M, Schuyler AD, Stern AS (2014) Nonuniform
  sampling and maximum entropy reconstruction in multidimensional nmr.
  \emph{Accounts of Chemical Research} 47(2):708--717.

\bibitem[{Hoffman et~al.(2001)Hoffman, Lee, \protect\BIBand{}
  Williams}]{hoffman2001new}
Hoffman A, Lee J, Williams J (2001) New upper bounds for maximum-entropy
  sampling. \emph{mODa 6—Advances in Model-Oriented Design and Analysis},
  143--153 (Springer).

\bibitem[{Horn(1954)}]{horn1954doubly}
Horn A (1954) Doubly stochastic matrices and the diagonal of a rotation matrix.
  \emph{American Journal of Mathematics} 76(3):620--630.

\bibitem[{Hou(1998)}]{hou1998classroom}
Hou SH (1998) Classroom note: A simple proof of the leverrier--faddeev
  characteristic polynomial algorithm. \emph{SIAM Review} 40(3):706--709.

\bibitem[{Hwang \protect\BIBand{} Rothblum(1993)}]{hwang1993majorization}
Hwang FK, Rothblum UG (1993) Majorization and schur convexity with respect to
  partial orders. \emph{Mathematics of Operations Research} 18(4):928--944.

\bibitem[{Jaynes(1957)}]{jaynes1957information}
Jaynes ET (1957) Information theory and statistical mechanics. \emph{Physical
  Review} 106(4):620.

\bibitem[{Kelmans \protect\BIBand{}
  Kimelfeld(1983)}]{kelmans1983multiplicative}
Kelmans AK, Kimelfeld B (1983) Multiplicative submodularity of a matrix's
  principal minor as a function of the set of its rows and some combinatorial
  applications. \emph{Discrete Mathematics} 44(1):113--116.

\bibitem[{Ko et~al.(1995)Ko, Lee, \protect\BIBand{} Queyranne}]{ko1995exact}
Ko CW, Lee J, Queyranne M (1995) An exact algorithm for maximum entropy
  sampling. \emph{Operations Research} 43(4):684--691.

\bibitem[{Lee(1998)}]{lee1998constrained}
Lee J (1998) Constrained maximum-entropy sampling. \emph{Operations Research}
  46(5):655--664.

\bibitem[{Lee \protect\BIBand{} Williams(2003)}]{lee2003linear}
Lee J, Williams J (2003) A linear integer programming bound for maximum-entropy
  sampling. \emph{Mathematical Programming} 94(2-3):247--256.

\bibitem[{Lemar{\'e}chal \protect\BIBand{}
  Renaud(2001)}]{lemarechal2001geometric}
Lemar{\'e}chal C, Renaud A (2001) A geometric study of duality gaps, with
  applications. \emph{Mathematical Programming} 90(3):399--427.

\bibitem[{Lewis(1995)}]{lewis1995convex}
Lewis A (1995) The convex analysis of unitarily invariant matrix functions.
  \emph{Journal of Convex Analysis} 2(1/2):173--183.

\bibitem[{Li et~al.(2012)Li, Cui, Weng, Negi, Franchetti, \protect\BIBand{}
  Ilic}]{li2012information}
Li Q, Cui T, Weng Y, Negi R, Franchetti F, Ilic MD (2012) An
  information-theoretic approach to pmu placement in electric power systems.
  \emph{IEEE Transactions on Smart Grid} 4(1):446--456.

\bibitem[{Lin \protect\BIBand{} Trudinger(1994)}]{lin1994some}
Lin M, Trudinger NS (1994) On some inequalities for elementary symmetric
  functions. \emph{Bulletin of the Australian Mathematical Society}
  50(2):317--326.

\bibitem[{Madan et~al.(2019)Madan, Singh, Tantipongpipat, \protect\BIBand{}
  Xie}]{madan2019combinatorial}
Madan V, Singh M, Tantipongpipat U, Xie W (2019) Combinatorial algorithms for
  optimal design. \emph{Conference on Learning Theory}, 2210--2258.

\bibitem[{Magnus(1985)}]{magnus1985differentiating}
Magnus JR (1985) On differentiating eigenvalues and eigenvectors.
  \emph{Econometric Theory} 1(2):179--191.

\bibitem[{Meyer(1973)}]{meyer1973generalized}
Meyer CD Jr (1973) Generalized inversion of modified matrices. \emph{SIAM
  Journal on Applied Mathematics} 24(3):315--323.

\bibitem[{Moreno-Salinas et~al.(2013)Moreno-Salinas, Pascoal, \protect\BIBand{}
  Aranda}]{moreno2013sensor}
Moreno-Salinas D, Pascoal A, Aranda J (2013) Sensor networks for optimal target
  localization with bearings-only measurements in constrained three-dimensional
  scenarios. \emph{Sensors} 13(8):10386--10417.

\bibitem[{Nikolov(2015)}]{nikolov2015randomized}
Nikolov A (2015) Randomized rounding for the largest simplex problem.
  \emph{Proceedings of the forty-seventh annual ACM symposium on Theory of
  computing}, 861--870 (ACM).

\bibitem[{Nikolov et~al.(2019)Nikolov, Singh, \protect\BIBand{}
  Tantipongpipat}]{nikolov2019proportional}
Nikolov A, Singh M, Tantipongpipat UT (2019) Proportional volume sampling and
  approximation algorithms for a-optimal design. \emph{Proceedings of the
  Thirtieth Annual ACM-SIAM Symposium on Discrete Algorithms}, 1369--1386
  (SIAM).

\bibitem[{O'Flynn et~al.(2010)O'Flynn, Regan, Lawlor, Wallace, Torres,
  \protect\BIBand{} O'Mathuna}]{o2010experiences}
O'Flynn B, Regan F, Lawlor A, Wallace J, Torres J, O'Mathuna C (2010)
  Experiences and recommendations in deploying a real-time, water quality
  monitoring system. \emph{Measurement Science and Technology} 21(12):124004.

\bibitem[{Overton \protect\BIBand{} Womersley(1995)}]{overton1995second}
Overton ML, Womersley RS (1995) Second derivatives for optimizing eigenvalues
  of symmetric matrices. \emph{SIAM Journal on Matrix Analysis and
  Applications} 16(3):697--718.

\bibitem[{Pedregosa et~al.(2018)Pedregosa, Askari, Negiar, \protect\BIBand{}
  Jaggi}]{pedregosa2018step}
Pedregosa F, Askari A, Negiar G, Jaggi M (2018) Step-size adaptivity in
  projection-free optimization. \emph{arXiv preprint arXiv:1806.05123} .

\bibitem[{Pukelsheim(2006)}]{pukelsheim2006optimal}
Pukelsheim F (2006) \emph{Optimal design of experiments} (SIAM).

\bibitem[{Rigau et~al.(2003)Rigau, Feixas, \protect\BIBand{}
  Sbert}]{rigau2003entropy}
Rigau J, Feixas M, Sbert M (2003) Entropy-based adaptive sampling.
  \emph{Graphics Interface}, volume~2, 79--87.

\bibitem[{Rockafellar(1970)}]{rockafellar1970convex}
Rockafellar RT (1970) \emph{Convex analysis}, volume~28 (Princeton university
  press).

\bibitem[{Sagnol et~al.(2015)Sagnol, Harman et~al.}]{sagnol2015computing}
Sagnol G, Harman R, et~al. (2015) Computing exact $ d $-optimal designs by
  mixed integer second-order cone programming. \emph{The Annals of Statistics}
  43(5):2198--2224.

\bibitem[{Schmieder et~al.(1993)Schmieder, Stern, Wagner, \protect\BIBand{}
  Hoch}]{schmieder1993application}
Schmieder P, Stern AS, Wagner G, Hoch JC (1993) Application of nonlinear
  sampling schemes to cosy-type spectra. \emph{Journal of Biomolecular NMR}
  3(5):569--576.

\bibitem[{Sharma et~al.(2015)Sharma, Kapoor, \protect\BIBand{}
  Deshpande}]{sharma2015greedy}
Sharma D, Kapoor A, Deshpande A (2015) On greedy maximization of entropy.
  \emph{International Conference on Machine Learning}, 1330--1338.

\bibitem[{Shewry \protect\BIBand{} Wynn(1987)}]{shewry1987maximum}
Shewry MC, Wynn HP (1987) Maximum entropy sampling. \emph{Journal of Applied
  Statistics} 14(2):165--170.

\bibitem[{Singh \protect\BIBand{} Xie(2018)}]{singh2018approximate}
Singh M, Xie W (2018) Approximate positive correlated distributions and
  approximation algorithms for d-optimal design. \emph{Proceedings of the
  Twenty-Ninth Annual ACM-SIAM Symposium on Discrete Algorithms}, 2240--2255
  (Society for Industrial and Applied Mathematics).

\bibitem[{Singh \protect\BIBand{} Xie(2020)}]{singh2020approximation}
Singh M, Xie W (2020) Approximation algorithms for d-optimal design.
  \emph{Mathematics of Operations Research} 45(4):1512--1534.

\bibitem[{Song \protect\BIBand{} Li{\`o}(2010)}]{song2010new}
Song Y, Li{\`o} P (2010) A new approach for epileptic seizure detection: sample
  entropy based feature extraction and extreme learning machine. \emph{Journal
  of Biomedical Science and Engineering} 3(06):556.

\bibitem[{Summa et~al.(2014)Summa, Eisenbrand, Faenza, \protect\BIBand{}
  Moldenhauer}]{summa2014largest}
Summa MD, Eisenbrand F, Faenza Y, Moldenhauer C (2014) On largest volume
  simplices and sub-determinants. \emph{Proceedings of the twenty-sixth annual
  ACM-SIAM symposium on Discrete algorithms}, 315--323 (SIAM).

\bibitem[{Thompson(1977)}]{thompson1977singular}
Thompson RC (1977) Singular values, diagonal elements, and convexity.
  \emph{SIAM Journal on Applied Mathematics} 32(1):39--63.

\bibitem[{Tsing et~al.(1994)Tsing, Fan, \protect\BIBand{}
  Verriest}]{tsing1994analyticity}
Tsing NK, Fan MK, Verriest EI (1994) On analyticity of functions involving
  eigenvalues. \emph{Linear Algebra and its Applications} 207:159--180.

\bibitem[{Vershynin(2018)}]{vershynin2018high}
Vershynin R (2018) \emph{High-dimensional probability: An introduction with
  applications in data science}, volume~47 (Cambridge university press).

\bibitem[{Wang et~al.(2019)Wang, Zhang, \protect\BIBand{}
  Chen}]{wang2019optimal}
Wang Y, Zhang L, Chen G (2019) Optimal sensor placement for obstacle detection
  of manipulator based on relative entropy. \emph{2019 14th IEEE Conference on
  Industrial Electronics and Applications (ICIEA)}, 702--707 (IEEE).

\bibitem[{Xu \protect\BIBand{} Do{\u{g}}an{\c{c}}ay(2017)}]{xu2017optimal}
Xu S, Do{\u{g}}an{\c{c}}ay K (2017) Optimal sensor placement for 3-d
  angle-of-arrival target localization. \emph{IEEE Transactions on Aerospace
  and Electronic Systems} 53(3):1196--1211.

\bibitem[{Yao \protect\BIBand{} Wang(2019)}]{yao2019optimal}
Yao Y, Wang H (2019) Optimal subsampling for softmax regression.
  \emph{Statistical Papers} 60(2):235--249.

\bibitem[{Zilly et~al.(2017)Zilly, Buhmann, \protect\BIBand{}
  Mahapatra}]{zilly2017glaucoma}
Zilly J, Buhmann JM, Mahapatra D (2017) Glaucoma detection using entropy
  sampling and ensemble learning for automatic optic cup and disc segmentation.
  \emph{Computerized Medical Imaging and Graphics} 55:28--41.

\end{thebibliography}
\newpage
\titleformat{\section}{\large\bfseries}{\appendixname~\thesection .}{0.5em}{}
\begin{appendices}

\section{Proofs}\label{proofs}
\subsection{Proof of  \Cref{pro:obj}} \label{proof_pro_obj}
\proobj*
	\begin{proof}\textbf{Part (i).} The discrete-submodularity has been proved by \cite{kelmans1983multiplicative}.
	
	We show the other three properties using the following example.
	
	\begin{example} \label{eg:prop}
		For MESP \eqref{eq_obj}, let $n=d=2$, $\bm{v}_1 = (\sqrt{a}, 0)^{\top}$ and $\bm{v}_2 = (0, \sqrt{b})^{\top}$. 
	\end{example}
	\noindent	\textbf{Part (ii) \& Part (iv).} 	In \Cref{eg:prop}, when $a=2$ and $b= 1/4$, we have
	\begin{align*}
	\log \det^1 \left (\bm{v}_1\bm{v}_1^\top \right ) = \log 2 \ge \log \det^2 \left (\bm{v}_1\bm{v}_1^\top + \bm{v}_2\bm{v}_2^\top\right ) = \log \frac{1}{2} < 0,
	\end{align*}
	which proves that the objective function of MESP is not monotonic and is not always nonnegative.
	
	\noindent	\textbf{Part (iii).} 			 In \Cref{eg:prop}, let us consider two feasible solutions $\bm{x}^1=(1, 0)^{\top}$ and $\bm{x}^2=(0, 1)^{\top}$ with $s=1$. If $a=1$ and $b=1$, then we have
	\begin{align*}
	\frac{1}{2}\log \det^1 \left ( \bm{v}_1\bm{v}_1^\top \right )+\frac{1}{2}\log \det^1 \left ( \bm{v}_2\bm{v}_2^\top \right ) = 0 \ge \log \det^1 \bigg (\sum_{i\in[n]} \frac{x_i^1+x_i^2}{2} \bm{v}_i\bm{v}_i^\top \bigg ) =  \log \frac{1}{2}  ,
	\end{align*}
	which disproves the concavity.		
	
	If $a=16$ and $b=1$, then we have
	\begin{align*}
	\frac{1}{2}\log \det^1 \left ( \bm{v}_1\bm{v}_1^\top \right )+\frac{1}{2}\log \det^1 \left ( \bm{v}_2\bm{v}_2^\top \right ) = \log 4 \le \log \det^1 \bigg (\sum_{i\in[n]} \frac{x_i^1+x_i^2}{2} \bm{v}_i\bm{v}_i^\top \bigg ) =  \log 8  ,
	\end{align*}
	which disproves the convexity. \qed
\end{proof}

\subsection{Proof of \Cref{lem:ldmax}} \label{proof_lem_ldmax}
Before proving \Cref{lem:ldmax}, we  first show the following technical lemma.
\begin{restatable}{lemma}{lemmin} \label{lem:min} 
	Given $\lambda_1 \ge \cdots \ge \lambda_d\ge 0$ and $0\le \beta_1 \le \cdots \le \beta_d$, we have
	\begin{enumerate}[(i)]
		\item 		\begin{align}
		\bm{\lambda} :=	\argmin_{\begin{subarray}{c}\bm{\theta}\in \Re_{+}^d,\\
			\theta_1 \ge \cdots \ge \theta_d
			\end{subarray}} \Bigg\{ \sum_{i \in [d]} \theta_{i} \beta_{i} :\sum_{i\in [t]} \theta_{i} \le \sum_{i\in [t]} \lambda_{i} , \forall t \in [d-1], \sum_{i\in [d]} \theta_{i} = \sum_{i\in [d]} \lambda_{i}\Bigg\} 
		, \label{min_lambda}
		\end{align}
		\item 
		\begin{align}
		\bm{\beta} :=	\argmin_{\begin{subarray}{c}\bm{\theta}\in \Re_{+}^d,\\
			\theta_1 \le \cdots \le \theta_d
			\end{subarray}} \Bigg\{ \sum_{i \in [d]} \theta_{i} \lambda_{i} :\sum_{i\in [t+1,d]} \theta_{i} \le \sum_{i\in [t+1,d]} \beta_{i} , \forall t \in [d-1], \sum_{i\in [d]} \theta_{i} = \sum_{i\in [d]} \beta_{i}\Bigg\} . \label{min_beta}
		\end{align}
	\end{enumerate}
\end{restatable}
\begin{proof}
	To prove Part(i), it needs to show that the vector $\bm{\lambda}\in\Re_{+}^d$ is an optimal solution to the minimization problem in the right-hand size of \eqref{min_lambda}.	We use the induction to prove this result.
	\begin{enumerate}[(a)]
		\item When $d=1$, clearly, there is only one optimal solution, which is $\theta_{1}^* = \lambda_1$.
		\item Suppose that the result holds for any $d < \hat{d}$ where $\hat{d}\geq 1$. Now let us consider the case that $d=\hat{d}$. Since the feasible region of the minimization problem in the right-hand size of \eqref{min_lambda} does not contain a ray, one of its optimal solutions must be an extreme point, which is denoted by $\hat{\theta }$. Then $\hat{\theta }$, as an extreme point, must satisfy at least $d$ binding constraints. There are two cases to be discussed:
		\begin{itemize}
			\item If there exists an integer $\hat t\in [d-1]$ such that $\sum_{i\in [\hat t]} \hat{\theta}_{i} = \sum_{i\in [\hat t]} \lambda_{i} $, then problem \eqref{min_lambda}  can be lower bounded by the sum of the following two minimization problems:
			\begin{align*}
			&\min_{\bm{ \theta}} \left\{ \sum_{i \in [\hat t]}  \theta_{i} \beta_{i} :\sum_{i\in [t]}  \theta_{i} \le \sum_{i\in [t]} \lambda_{i} , \forall t \in [\hat t-1], \sum_{i\in [\hat t]}  \theta_{i} = \sum_{i\in [\hat t]} \lambda_{i}, \theta_{1} \ge \cdots \ge  \theta_{\hat t}\right\},\\
			&\min_{\bm{\theta}} \left\{ \sum_{i \in [\hat t+1,d]}  \theta_{i} \beta_{i} :\sum_{i\in [\hat t+1,t]}  \theta_{i} \le \sum_{i\in [\hat t+1,t]} \lambda_{i} ,  \forall t \in [\hat t+1,d], \sum_{i\in [\hat t+1,d]}  \theta_{i} = \sum_{i\in [\hat t+1, d]} \lambda_{i}, \theta_{\hat t+1} \ge \cdots \ge  \theta_{d}\right\}.
			\end{align*}
			According to the induction, there exists an optimal solution of each minimization problem such that $ \theta_i^*=\lambda_i$ for any $i \in [d]$, which is feasible to the original problem \eqref{min_lambda} and thus is optimal.
			
			\item If there does not exist an integer $\hat t\in [d-1]$ such that $\sum_{i\in [\hat t]} \hat{\theta}_{i} = \sum_{i\in [\hat t]} \lambda_{i} $, then the extreme point $\hat{\bm  \theta}$ must satisfy $\hat{\theta}_{1} = \cdots = \hat{\theta}_{d}=\frac{\sum_{i \in [d]} \lambda_i }{d}$. Given $0\leq\beta_1 \le \cdots \le \beta_d$, we have
			\[\sum_{i \in [d]} \lambda_{i} \beta_{i} \le \frac{\sum_{i \in [d]} \lambda_i }{d} \sum_{i \in [d]}\beta_{i} .\]
		\end{itemize} 
		Therefore, when $d=\hat{d}$, $\bm{\theta}^*=\bm{\lambda}$ is also an optimal solution. 
	\end{enumerate}
	The proof of Part (ii) directly follows from the above if we consider $\bm{\beta} = (\lambda_d, \lambda_{d-1}, \cdots, \lambda_1)^{\top}$, $\bm{\lambda} = (\beta_d, \beta_{d-1}, \cdots, \beta_1 )^{\top}$ and $\bm{\theta} = (\theta_d, \theta_{d-1}, \cdots, \theta_1 )^{\top}$  in Part (i). \qed
\end{proof}

{Now let us  prove \Cref{lem:ldmax}.}

\lemldmax*
\begin{subequations}
	\begin{proof}
	{For any $d \times d $ matrix $\bm{X} \succeq 0$}, suppose that $\bm{\lambda}$ is the vector of its eigenvalues satisfying $\lambda_1 \ge \cdots \ge \lambda_d \ge 0$, and according to the eigendecomposition \citep{abdi2007eigen}, there exists an orthonormal matrix $\bm{Q} $ such that ${\bm{X}} = \bm{Q} \Diag(\bm{\lambda}) \bm{Q}^{\top}$. Then the objective function in the left-hand side of \eqref{ldmax} is equivalent to
	\begin{align*} 
	\log \det^s (\bm{X}) -\tr (\bm{X} \bm{\Lambda}) =  \log \bigg ( \prod_{i\in [s]} \lambda_{i} \bigg )- \tr(\Diag(\bm{\lambda}) \bm{Q}^{\top} \bm{\Lambda} \bm{Q} ) = \log \bigg( \prod_{i\in [s]} \lambda_{i} \bigg ) - \sum_{i \in [d] } \theta_{i} \lambda_{i},
	\end{align*}
	where let ${\bm{\theta}} = \diag(\bm{Q}^{\top} \bm{\Lambda} \bm{Q})$. Thus, the left-hand side of \eqref{ldmax} becomes
	\begin{align*}
\max_{\begin{subarray}{c}
		\bm{\lambda} \in \Re_{+}^d,\\
		\lambda_1 \ge \cdots \ge \lambda_d \ge 0
		\end{subarray}} \Bigg \{ \log \bigg( \prod_{i\in [s]} \lambda_{i} \bigg ) - 	\min_{\bm{Q}, \bm{\theta} \in \Re_{+}^d}  \bigg \{ \sum_{i \in [d] } \theta_{i} \lambda_i :  {\bm{\theta}} = \diag(\bm{Q}^{\top} \bm{\Lambda} \bm{Q}), \bm{Q} \textrm{ is orthonormal}\bigg \} \Bigg \}.
	\end{align*}
	
Since any permutation matrix is orthonormal,  for any fixed $\lambda_1 \ge \cdots \ge \lambda_d$, to maximize $-\sum_{i \in [d] } \theta_{i} \lambda_{i}$, we must have $\theta_1\le \cdots \le \theta_d$ based on the rearrangement inequality \citep{hardy1952inequalities}. Thus, the left-hand side of \eqref{ldmax} is further reduced to
	\begin{align}
\max_{\begin{subarray}{c}
	\bm{\lambda} \in \Re_{+}^d,\\
	\lambda_1 \ge \cdots \ge \lambda_d \ge 0
	\end{subarray}} \Bigg \{ \log \bigg( \prod_{i\in [s]} \lambda_{i} \bigg ) - 		\min_{\begin{subarray}{c}\bm{Q}, \bm{\theta} \in \Re_{+}^d\\
	\theta_1\le \cdots \le \theta_d
	\end{subarray}}  \bigg \{ \sum_{i \in [d] } \theta_{i} \lambda_i :  {\bm{\theta}} = \diag(\bm{Q}^{\top} \bm{\Lambda} \bm{Q}), \bm{Q} \textrm{ is orthonormal}\bigg \} \Bigg \}. 
\label{eq_lhs}
\end{align}
	
{Let $\bm\beta$ denote the vector of eigenvalues of $\bm{\Lambda}$ such that $\beta_1 \le \cdots \le \beta_d$ and let ${\bm{\Lambda}} = \bm{P} \Diag(\bm{\beta}) \bm{P}^{\top}$ with an orthonormal matrix $\bm{P}$. Since $\bm{Q}$ is orthonormal, the eigenvalues of $\bm{Q}^{\top} \bm{\Lambda} \bm{Q}$ are also equal to $\bm{\beta}$. According to the well-known majorization inequalities between eigenvalues $\bm{\beta}$ and diagonal entries  $\bm{\theta}$ (see, e.g., \citealt{horn1954doubly,thompson1977singular}), the inner minimization problem in \eqref{eq_lhs} can be lower bounded by}
\begin{align*}
	\min_{\begin{subarray}{c}\bm{\theta}\in \Re_{+}^d,\\
	\theta_1 \le \cdots \le \theta_d
	\end{subarray}} \Bigg\{ \sum_{i \in [d]} \theta_{i} \lambda_{i} :\sum_{i\in [t+1,d]} \theta_{i} \le \sum_{i\in [t+1,d]} \beta_{i} , \forall t \in [d-1], \sum_{i\in [d]} \theta_{i} = \sum_{i\in [d]} \beta_{i}\Bigg\}
\end{align*}
Applying Part (i) in \Cref{lem:min}, an optimal solution to the minimization problem is $\bm{\theta}^*=\bm{\beta}$. Thus, the optimal value of the relaxed minimization problem is $\sum_{i \in [d]} \lambda_{i} \beta_{i}$, which is achieved by letting $\bm{Q}^*=\bm{P}$ and $\bm\theta^*=\bm\beta$ for the inner optimization problem in \eqref{eq_lhs} and is thus optimal. 

Plugging this optimal solution into the inner maximization problem in \eqref{eq_lhs}, we can obtain
	\begin{align}
\max_{\begin{subarray}{c}
	\bm{\lambda} \in \Re_{+}^d,\\
	\lambda_1 \ge \cdots \ge \lambda_d \ge 0
	\end{subarray}} \bigg \{ \log \bigg( \prod_{i\in [s]} \lambda_{i} \bigg ) - 	\sum_{i \in [d]}\beta_{i}\lambda_i \bigg \},
\end{align}
{which can be solved by $\lambda^*_{i}= \frac{1}{\beta_{i}}$ for all $i\in [s]$ and 0 otherwise. Therefore, we have}
	\[\max_{\bm{X}\succeq 0} \left \{	\log \det^s (\bm{X}) -\tr (\bm{X} \bm{\Lambda}) \right\}  =-\log \det_s (\bm{\Lambda}) -s.\]
	This completes the proof. \qed
\end{proof}
\end{subequations}

\subsection{Proof of \Cref{lem:ldmax2}} \label{proof_lem_ldmax2}
\lemldmaxtwo*
\begin{subequations}
	\begin{proof}\noindent\textbf{Part (i).} 		Suppose $\bm{\Lambda}$ has eigenvalues $0<\beta_1 \le \cdots \le \beta_d$ and ${\bm{\Lambda}} = \bm{P} \Diag(\bm{\beta}) \bm{P}^{\top}$ with an orthonormal matrix $\bm{P}$. Then the objective function in the left-hand side of \eqref{ldmin} is equal to
	\begin{align*}
	-\log \det_s (\bm{\Lambda}) +\tr (\bm{X} \bm{\Lambda})= -\log \bigg ( \prod_{i\in [s]} \beta_{i} \bigg) +\tr \left(\bm{P}^{\top} \bm{X} \bm{P} \Diag({\bm \beta}) \right) =  -\log  \bigg ( \prod_{i\in [s]} \beta_{i}  \bigg) +\sum_{i \in [d]} \theta_{i} \beta_{i} ,
	\end{align*}
	where ${\bm{\theta}} = \diag(\bm{P}^{\top} \bm{X} \bm{P})$.

	For any fixed $\beta_1 \le \cdots \le \beta_d$, according to the rearrangement inequality \citep{hardy1952inequalities}, to minimize $\sum_{i \in [d] } \theta_{i} \beta_{i}$, we must have $\theta_1\ge \cdots \ge \theta_d$. Thus, the left-hand side of \eqref{ldmin} becomes
	\begin{align}
	\min_{\begin{subarray}{c}\bm{\beta}\in \Re_{+}^d,\\
		0<\beta_1 \le \cdots \le \beta_d
		\end{subarray}}  \Bigg \{ -\log  \bigg( \prod_{i\in [s]} \beta_{i}  \bigg) + \min_{\begin{subarray}{c}
		\bm{P}, \bm{\theta} \in \Re_{+}^d \\
		\theta_1\ge \cdots \ge \theta_d
		\end{subarray}}  \bigg \{ \sum_{i \in [d] } \theta_{i} \beta_i :  {\bm{\theta}} = \diag(\bm{P}^{\top} \bm{X} \bm{P}), \bm{P} \textrm{ is orthonormal}  \bigg \}  \Bigg\}.\label{eq_ldmin1}
	\end{align}
	As $\bm{P}$ is orthonormal, thus the eigenvalues of $\bm{P}^{\top} \bm{X} \bm{P}$ are also equal to $\bm{\lambda}$. Then the inner minimization problem in \eqref{eq_ldmin1} can be lower bounded by
	\begin{align*}
	\min_{\bm{\theta}} \Bigg\{ \sum_{i \in [d]} \theta_{i} \beta_{i} :\sum_{i\in [t]} \theta_{i} \le \sum_{i\in [t]} \lambda_{i} , \forall t \in [d-1], \sum_{i\in [d]} \theta_{i} = \sum_{i\in [d]} \lambda_{i},\theta_{1} \ge \cdots \ge \theta_{d}\Bigg\} .  
	\end{align*}
According to Part (ii) in \Cref{lem:min}, the optimal value of the inner minimization problem in \eqref{eq_ldmin1} is $\sum_{i \in [d]} \lambda_{i} \beta_{i}$, which is achieved by letting $\bm{P}^*=\bm{Q}$ and $\bm\theta^*=\bm\lambda$. This proves the identity \eqref{ldmin}.

	\noindent\textbf{Part (ii).} Let us introduce an additional variable $\tau$ to differentiate the first $s$ smallest $\bm\beta$ elements and simplify the order constraint in the left-hand problem \eqref{ldmin1} as
	\begin{align}
\min_{\bm{\beta} \in \Re_{+}^d, \tau} \left\{-\sum_{i\in [s]}\log (\beta_i) + \sum_{i \in [d]} \lambda_{i} \beta_i :
	\beta_i \le \tau,  \forall i\in [s],\beta_i \ge \tau, \forall i \in [s+1, d]\right\} . \label{ldmineq2}
	\end{align}
	Let $\bm{\mu} \in \Re^d$ denote the Lagrangian multipliers and the Lagrangian function is
	$$L(\bm{\mu},\bm{\beta},\tau )=-\sum_{i\in [s]}\log (\beta_i) + \sum_{i \in [d]} \lambda_{i} \beta_i +\sum_{i\in [s]}\mu_i(\beta_i - \tau)+\sum_{i\in [s+1,d]}\mu_i(\tau-\beta_i ).$$
	Clearly, as the constraints in the convex program \eqref{ldmineq2} are linear, the relaxed Slater condition holds. Let $(\bm{\mu}^*,\bm{\beta}^*, \tau^* )$ denote the pair of optimal primal and dual solutions. Then the KKT conditions of the convex program \eqref{ldmineq2}  are
	\begin{align*}
	&\frac{\partial L}{\partial \beta_i}(\bm{\mu}^*,\bm{\beta}^*, \tau^* ) = -\frac{1}{\beta_i^*} + \lambda_{i} + \mu_i ^*=0 , \forall  i \in [s],\frac{\partial L}{\partial \beta_i} (\bm{\mu}^*,\bm{\beta}^*, \tau^* )= \lambda_{i}- \mu_i^* =0 , \forall i \in [s+1, d],\\
	&\frac{\partial L}{\partial \tau} (\bm{\mu}^*,\bm{\beta}^*, \tau^* )= \sum_{i\in [s]} \mu_i^* -\sum_{i \in [s+1, d]} \mu_i ^*=0,  \mu_i^*(\beta_i ^*- \tau^* )=0,  \forall  i \in [s],\mu_i^*(\tau^* -\beta_i^*)=0, \forall  i \in [s+1, d],\\
	&\beta_i^* \le \tau^*,  \forall  i \in [s], \beta_i ^*\ge \tau^*, \forall   i \in [s+1, d],\mu_i^* \ge 0 , \forall   i \in [d],
	\end{align*}
	which are necessary and sufficient optimality conditions (see theorem 3.2.4 in \citealt{ben2012optimization}). Recall that matrix $\bm{X}$ has rank $r$ and its eigenvalues are sorted such that $\lambda_1 \ge \cdots \ge \lambda_s \ge \cdots \ge \lambda_r > \lambda_{r+1}=\cdots = \lambda_d =0$. Additionally, according to the KKT conditions, {the optimal solution $\{\beta_{i}\}_{i\in [s]}$ must be sorted in an ascending order,} i.e., $\beta_1 \le \cdots \le \beta_s$. {Thus, let integer $k\in [0,s]$ denote the largest index such that $\beta_i^*<\tau^*$ (by convention, we let $\beta_0^*=0,\lambda_0=\infty$).} Then the above KKT conditions can be simplified as
	\begin{align*}
	&\beta_i^*=\frac{1}{\lambda_i},\mu_i^*=0,\forall i\in [k]; \beta_i^*=\tau^*,\mu_i^*=\frac{1}{\tau^*}-\lambda_i\geq 0, \forall  i \in [k+1,s]; \\
	&\mu_i^*=\lambda_{i}>0 , \beta_i^*=\tau^*,\forall i \in [s+1, r];\mu_i^*=\lambda_{i}=0 , \beta_i^*\geq \tau^*,\forall i \in [r+1, d];\\
	& \sum_{i\in [s]} \mu_i^* -\sum_{i \in [s+1, d]} \mu_i ^*=0.
	\end{align*}
	This implies that all pairs of the optimal primal and dual solutions are characterized by the following set
	\begin{align*}
	\Omega = &\Bigg\{ (\bm{\mu}, \bm{\beta}, \tau):  \tau = \frac{s-k}{\sum_{i\in[k+1,d]} \lambda_{i} }, \beta_i = \frac{1}{\lambda_{i}}, \forall  i \in [k],  \beta_i = \tau, \forall  i=[k+1, r],  \beta_i \ge \beta_r, \forall  i\in [r+1,d], \\
	& \mu_i = 0,  \forall  i =[k], \mu_i= \frac{1}{\tau}- \lambda_{i} ,\forall  i = [k+1, r], \mu_i = 0 , \forall   i = [r+1, d] \Bigg \}.
	\end{align*}
	
	Consequently, any optimal solution for problem \eqref{ldmineq2} satisfies
	\begin{align*} 
	\beta_{i}^* = \frac{1}{\lambda_{i}}, \forall  i \in [k],  \beta_{i}^* = \frac{s-k}{\sum_{i\in [k+1,d]} \lambda_{i}}, \forall  i\in [k+1,r],  \beta_{i}^* \ge \frac{s-k}{\sum_{i\in [k+1,d]} \lambda_{i}}, \forall  i \in [r+1, d], 
	\end{align*}
	which is feasible to the minimization problem in \eqref{ldmin1} and thus is optimal.
	
	Then the optimal value of the minimization problem in \eqref{ldmin1} is equal to
	\begin{align*}
	-\sum_{i\in [s]}\log (\beta_{i}^*) + \sum_{i\in [d]} \lambda_{i} \beta_{i}^* =\sum_{i\in [k]} \log (\lambda_{i}) + (s-k) \log \bigg(\frac{\sum_{i \in [k+1, d]} \lambda_{i} }{s-k} \bigg)+s
	=\Gamma_s (\bm{X})+s,
	\end{align*}
	where the second equality is due to \Cref{def:objPC} of $\Gamma_s (\bm{X})$.
	This completes the proof.\qed
\end{proof}
\end{subequations}

\subsection{Proof of \Cref{primal}} \label{proof_primal}
\primal*
	\begin{proof}
	It is sufficient to prove that for any feasible solution $\bm{x}$ to MESP \eqref{eq_p}, we must have
	$$\log \det^s \bigg (\sum_{i\in[n]} x_i \bm{v}_i\bm{v}_i^\top \bigg )=\Gamma_s\bigg(\sum_{i\in [n]}x_i \bm{v}_i\bm{v}_i^{\top}\bigg).$$
	
	Given a solution $\bm{x}$, we let $\bm{X}= \sum_{i\in [n]}x_i \bm{v}_i\bm{v}_i^{\top}$ with rank $r$ and let $\bm{\lambda}$ denote its eigenvalues such that $\lambda_1 \ge \cdots \ge \lambda_d \ge 0$. Since the rank of matrix $\bm{X}$ satisfies $r \le s$, there are two cases to be discussed regarding whether $r=s$ holds or not.
	\begin{enumerate}[(i)]
		\item If $r< s$, then clearly, we have $\log \overset{s}{\det} \left (\bm{X} \right )=-\infty$. On the other hand, by the choice of $k$ in Lemma \ref{lem:kappa}, it is evident that $k=r$ such that $\frac{1}{s-k} \sum_{i \in [k+1,d]} \lambda_{i} =0$. It follows that $\Gamma_s\left(\bm{X}\right)=-\infty=\log \overset{s}{\det} \left (\bm{X} \right )$.
		\item {If $r=s$, there must exist an integer $\ell$ such that $\lambda_{1} \ge \cdots \ge \lambda_{\ell} > \lambda_{\ell+1}=\cdots=  \lambda_{s}> \lambda_{s+1}=\cdots=\lambda_d=0$.} By the uniqueness of $k$, we must have {$k=\ell$}. Thus, from \Cref{def:objPC}, the objective value is equal to
		\begin{align*}
		\Gamma_s\left(\bm{X}\right)=\log \bigg(\prod_{i\in [k]} \lambda_{i}\bigg ) + (s-k) \log \bigg(\frac{1}{s-k} \sum_{i\in [k+1,d]} \lambda_{i} \bigg ) =\log \bigg(\prod_{i\in [s]} \lambda_{i} \bigg ) =\log \det^s \left ( \bm{X} \right ). 
		\end{align*} 
		\qed
	\end{enumerate}
\end{proof}

\subsection{Proof of \Cref{prop:mesp}}\label{proof_prop_mesp}
\propmesp*

\proof
We  show the three special cases separately.
\begin{enumerate}[(i)] \setlength{\itemsep}{0pt}
	\item Suppose that $\bm{C}$ is diagonal. Without loss of generality, assume that $\bm{C}=\Diag(\bm\lambda)$ with a nonnegative vector $\bm\lambda$ such that $\lambda_1\geq \cdots\geq\lambda_d>\lambda_{d+1}=\cdots=\lambda_n=0$, then we have $\bm{v}_i = \sqrt{\lambda_{i}} \bm{e}_i$ for each $i \in [n]$ and $\bm{C}=\bm{V}^\top\bm{V} $. Clearly, the optimal solution of MESP \eqref{eq_obj} is $x_i^* = 1$ for each $i\in [s]$ and 0 otherwise. Thus, $z^*= \log \overset{s}{\det} \left(\prod_{i \in [n]} x_i^* \bm{v}_i \bm{v}_i^{\top}\right) = \log \left( \prod_{i \in [s]} \lambda_i \right)$. 
	
	Let $\bm{X} = \sum_{i\in[n]} {x}_i^* \bm{v}_i\bm{v}_i^\top$, then we construct the  feasible solution to LD \eqref{eq_dual} as
	\begin{align*}
	\bm{\Lambda}^* = \frac{1}{\lambda_s} (\bm{I}_d -\bm{X}^{\dag} \bm{X})+\bm{X}^{\dag},  \nu^*=1, \mu_i^*=0, \forall i \in [n] . 
	\end{align*}
	It is easy to see that $(\bm{\Lambda}^*, \nu^*, \bm{\mu}^*)$ is feasible to LD \eqref{eq_dual} with the objective value
	\begin{align*}
	z^{LD} \leq -\log \det\limits_s (\bm{\Lambda^*}) + s v^*+\sum_{i \in [n]} {\mu_i^*}-s= \sum_{i \in [s]}\log(\lambda_i) =z^*  \le z^{LD} ,
	\end{align*}
	where the first inequality is by feasibility of $(\bm{\Lambda}^*, \nu^*, \bm{\mu}^*)$ and the second one is from the weak duality.

	\item Suppose that $s=1$. Given any feasible solution $\bm{x}$ to PC \eqref{eq_pcont}, assume that matrix $\bm{X}=\sum_{i \in [n]}x_i \bm{v}_i\bm{v}_i^{\top}$ has the eigenvalue vector $\bm{\lambda}$ such that $\lambda_1\geq\cdots\geq \lambda_d$. By Lemma \ref{lem:kappa}, as $k < s$, we must have $k=0$. Thus, the objective value of PC \eqref{eq_pcont} becomes
	\begin{align*}
	\Gamma_s ( \bm{X} ) = (s-k) \log \bigg(\frac{1}{s-k} \sum_{i\in [k+1,d]} \lambda_i \bigg ) = \log \bigg( \sum_{i\in [d]} \lambda_i\bigg) = \log \bigg( \sum_{i \in [n]} x_i  \bm{v}_i^{\top} \bm{v}_i \bigg ).
	\end{align*}
	Therefore, in this case, we have
	\begin{align*}
	z^{LD}&=\max_{\bm{x}}\bigg\{\log \bigg( \sum_{i \in [n]} x_i \bm{v}_i^{\top} \bm{v}_i \bigg ):\sum_{i \in [n]} x_i =1,\bm{x}\in [0,1]^n\bigg\}=\max_{i\in [n]}\left\{\log (\bm{v}_i^{\top} \bm{v}_i) \right\}=z^*.
	\end{align*}
	
	\item Suppose that $s=n$. In this case,  the only feasible solution of PC \eqref{eq_pcont} or MESP \eqref{eq_p} is $x_{i}=1$ for each $ i \in [n]$ and clearly, PC \eqref{eq_pcont} and MESP \eqref{eq_p} are equivalent. \qed
\end{enumerate}
\endproof

\subsection{Proof of \Cref{lem_hessian}}\label{proof_lem_hessian}
\lemhessian*
\begin{subequations}
\begin{proof}
	\noindent%
	We split the proof into four steps.
	
	\noindent\textbf{Step (i)- An Equivalent Statement.}
	For any $\bm{x},\bm{y} \in \textrm{relint}(\mathbb{D})$, let $\bm{X} = \sum_{i \in [n]} {x}_i \bm{v}_i \bm{v}_i^{\top}$ and $\bm{Y} = \sum_{i \in [n]} {y}_i \bm{v}_i \bm{v}_i^{\top}$, clearly, matrices $\bm{X}$ and $\bm{Y}$ are positive-definite and non-singular. Let us define a function $h(t) = \Gamma_s(\bm{X} + t (\bm{Y} - \bm{X}))$ with $t \in [0,\epsilon]$ for some sufficiently small positive number $\epsilon$. Let $\bm\lambda \in \Re_{++}^d$ denote the vector of eigenvalues of $\bm{X}$ and $\lambda_1 \ge \cdots\lambda_d >0$. Since $$\Gamma_s(\bm{X}) =F(\bm\lambda):= \log \bigg (\prod_{i\in [k]} \lambda_{i}\bigg ) + (s-k) \log \bigg(\frac{1}{s-k} \sum_{i\in [k+1,d]} \lambda_{i} \bigg),$$ and 
	$F(\bm\lambda)$ is symmetric and analytic at $\Re_{++}^d$, thus according to theorem 2.1 in \cite{tsing1994analyticity}, $\Gamma_s(\bm{X})$ is analytic and is thus continuous differentiable.
	Since the positive-definite matrices with distinct eigenvalues are dense in the space of all the positive-definite matrices, without loss of generality, we can assume that $\bm{X}$ has eigenvalues $\lambda_1 > \cdots > \lambda_d>0$  and their corresponding eigenvectors are $\bm{q}_1,\cdots, \bm{q}_d$. Suppose that the eigenvalues and their corresponding eigenvectors of $\bm{X} + t (\bm{Y} - \bm{X})$ are $\lambda_1(t),\cdots, \lambda_d(t)$ and $\bm{q}_1(t),\cdots, \bm{q}_d(t)$. As $\epsilon$ is sufficiently small, thus, we still have $\lambda_1(t)>\cdots>\lambda_d(t)$ and according to \Cref{lem:kappa}, $\bm{\lambda}$ and $\bm{\lambda}(t)$ share the same integer $k$ for all $t\in [0,\epsilon]$. Since all the eigenvalues are distinct, the eigenvalues $\{\lambda_i(t)\}_{i\in [d]}$ and eigenvectors $\{\bm{q}_i(t)\}_{i\in [d]}$  are continuous in the range of $[0,\epsilon]$ (see, e.g.,  \citealt{magnus1985differentiating,overton1995second}).

	As stated in Proposition \ref{supdiff}, function $\Gamma_s(\hat{\bm{X}})$ is differentiable if matrix $\hat{\bm{X}}$ is positive-definite. Thus, for any $t\in (0,\epsilon)$, we have
	$$h'(t) = \frac{\mathrm{d}}{\mathrm{d} t} h(t) =  \left<\nabla \Gamma_s(\bm{X}+ t (\bm{Y} - \bm{X})), \bm{Y}-\bm{X} \right>,$$
	which implies that
	\begin{align*}
	h''(0)=\frac{\mathrm{d}^2}{\mathrm{d} t^2} h(t) \Big{|}_{t=0} = \bigg< \frac{\mathrm{d}}{\mathrm{d} t} \nabla \Gamma_s(\bm{X}+ t (\bm{Y} - \bm{X})) \Big{|}_{t=0} , \bm{Y}-\bm{X}\bigg>. 
	\end{align*}
	Therefore, to prove the inequality \eqref{lips}, it is sufficient to show that
	\begin{align}
	h''(0)\geq - \frac{ \lambda_{\max}^2(\bm{C})}{\delta^2} \|\bm{x}-\bm{y}\|_2^2.\label{deriv0}
	\end{align}
	
	\noindent\textbf{Step (ii)- A Representation of $h''(0)$.}
	
	By Proposition \ref{supdiff}, we have
	$$\nabla \Gamma_s(\bm{X}+ t (\bm{Y} - \bm{X}))= \sum_{i \in [k]} \frac{1}{\lambda_i(t)} \bm{q}_i (t)\bm{q}_i(t)^{\top}+ \sum_{i \in [k+1,d]} \frac{s-k}{\sum_{j \in [k+1,d]}\lambda_j(t)} \bm{q}_i (t)\bm{q}_i(t)^{\top}.$$
	
	For the notational convenience, let us define a vector $\bm{
		\beta}\in \Re_+^d$ such that 
$$\beta_i = \lambda_i, \forall i \in [k], \beta_{i} = \frac{1}{s-k}\sum_{j\in [k+1,d]}\lambda_j, \forall i \in [k+1,d].$$
  Taking the derivative of eigenvalues and eigenvectors over $t$ separately, we obtain
	\begin{align*}
	\frac{\mathrm{d}}{\mathrm{d} t} \nabla \Gamma_s(\bm{X}+ t (\bm{Y} - \bm{X})) \Big{|}_{t=0} =& \underbrace{-\sum_{i \in [k]} \frac{1}{\beta_i^2}  \frac{d \lambda_i(t)}{d t}\Big{|}_{t=0} \bm{q}_i \bm{q}_i^{\top}-\sum_{i \in [k+1,d]} \frac{1}{(s-k)\beta_{i}^2} \frac{d \lambda_i(t)}{d t}\Big{|}_{t=0} \bm{q}_i \bm{q}_i^{\top}}_{:=A}\nonumber \\
	&+ \underbrace{\sum_{i \in [d]}\frac{1}{\beta_i} \frac{d \bm{q}_i(t)}{d t} \Big{|}_{t=0}\bm{q}_i^{\top}+\sum_{i \in [d]}\frac{1}{\beta_i} \bm{q}_i \Big{(}\frac{d \bm{q}_i(t)}{d t} \Big{|}_{t=0}\Big{)}^{\top}}_{:=B}.
	\end{align*}
	
	It follows that
	\begin{align}
	h''(0)= \big< A , \bm{Y}-\bm{X}\big>+\big< B , \bm{Y}-\bm{X}\big>.	\label{deriv}
	\end{align}
	
	Thus, to prove \eqref{deriv0}, we need to find lower bounds of $\big< A , \bm{Y}-\bm{X}\big>$ and $\big< B , \bm{Y}-\bm{X}\big>$ separately.
	
	\noindent\textbf{Step (iii)- Lower Bounds of $\big< A , \bm{Y}-\bm{X}\big>$ and $\big< B , \bm{Y}-\bm{X}\big>$.}

	Before we proceed, let us first prove the following claim.
	\begin{claim}\label{claim1}
For any $\ell \in [s-1]$, we have
		$$ \min_{\bm{u} \in \mathbb{D}} \bigg\{ \frac{1}{s-\ell} \sum_{i \in [\ell+1, d]}\lambda_i(\bm{U}): \bm{U}= \sum_{i \in [n]} {u}_i \bm{v}_i \bm{v}_i^{\top}\bigg\} \geq \min_{S\in [n], |S|=s} \lambda_s\Big(\sum_{j \in S} \bm{v}_j \bm{v}_j^{\top}\Big):=\delta,$$
		where for a symmetric matrix $\bm{X}$, we let $\lambda_i(\bm{X})$ denotes its $i$-th largest eigenvalue.
	\end{claim}
	\begin{proof}For a $d \times d$ positive-semidefinite matrix $\bm{U}$, the function $\sum_{i \in [\ell+1, d]}\lambda_i(\bm{U})$ is concave \citep{fan1949theorem}. On the other hand, it is known that for the concave minimization problem, the optimum can be achieved by one of the extreme points of the feasible region. Thus, 
\begin{align*}
		\inf_{\bm{u} \in \mathbb{D}} \bigg\{ \frac{1}{s-\ell}\sum_{i \in [\ell+1, d]}\lambda_i(\bm{U}): \bm{U}= \sum_{i \in [n]} {u}_i \bm{v}_i \bm{v}_i^{\top}\bigg\} &=
		\frac{1}{s-\ell} \min_{S\in [n], |S|=s} \sum_{i \in [\ell+1, d]}\lambda_i\Big(\sum_{j \in S} \bm{v}_j \bm{v}_j^{\top}\Big)\\
		&= \frac{1}{s-\ell} \min_{S\in [n], |S|=s} \sum_{i \in [\ell+1, s]}\lambda_i\Big(\sum_{j \in S} \bm{v}_j \bm{v}_j^{\top}\Big) \\
		&\geq  \min_{S\in [n], |S|=s} \lambda_s\Big(\sum_{j \in S} \bm{v}_j \bm{v}_j^{\top}\Big) ,
		\end{align*}
		where the second equation is due to the fact that rank of $\sum_{j \in S} \bm{v}_j \bm{v}_j^{\top}$ is equal to $s$, and the first inequality is because $\lambda_s\Big(\sum_{j \in S} \bm{v}_j \bm{v}_j^{\top}\Big)$ is the smallest positive eigenvalues of matrix $\sum_{j \in S} \bm{v}_j \bm{v}_j^{\top}$.
		\qedA
	\end{proof}
	
	Now we are ready to show the lower bounds of $\big< A , \bm{Y}-\bm{X}\big>$ and $\big< B , \bm{Y}-\bm{X}\big>$.
	\begin{enumerate}[(a)]
		\item According to \cite{overton1995second}, we have
		$$\frac{d \lambda_i(t)}{d t} \Big{|}_{t=0} = \bm{q}_i^{\top}\frac{d (\bm{X}+ t (\bm{Y} - \bm{X}))}{d t} \Big{|}_{t=0} \bm{q}_i = \bm{q}_i^{\top} (\bm{Y}-\bm{X})\bm{q}_i.$$ 
		Therefore, $\big< A , \bm{Y}-\bm{X}\big>$ is equivalent to
		\begin{align}
		\big< A , \bm{Y}-\bm{X}\big>&=
		-\sum_{i \in [k]} \frac{1}{\beta_i^2} \left( \bm{q}_i^{\top} (\bm{Y}-\bm{X})\bm{q}_i \right)^2- \frac{1}{\left(s-k\right)\beta_i^2} \sum_{i \in [k+1,d]} \left( \bm{q}_i^{\top} (\bm{Y}-\bm{X})\bm{q}_i \right)^2 \notag\\
		&\geq 
		-\frac{(s-k)^2}{(\sum_{j \in [k+1,d]}\lambda_j)^2}\sum_{i \in [k]} \left( \bm{q}_i^{\top} (\bm{Y}-\bm{X})\bm{q}_i \right)^2- \frac{s-k}{(\sum_{j \in [k+1,d]}\lambda_j)^2} \sum_{i \in [k+1,d]} \left( \bm{q}_i^{\top} (\bm{Y}-\bm{X})\bm{q}_i \right)^2 \notag\\
		&\geq 
		- \frac{1}{\delta^2} \sum_{i \in [d]}\left( \bm{q}_i^{\top} (\bm{Y}-\bm{X})\bm{q}_i \right)^2,  \label{eq_hess1}
		\end{align}
		where the first inequality is due to the fact that $\lambda_1 \ge \cdots \lambda_k > \frac{\sum_{j \in [k+1, d]}\lambda_j}{ s-k}$, the second inequality is because of Claim \ref{claim1}, and $s-k\geq 1$.
		\item According to the result from \cite{magnus1985differentiating} that 
		$\frac{d \bm{q}_i(t)}{d t} \Big{|}_{t=0} = (\lambda_i \bm{I}_d-\bm{X})^{\dag} \frac{d (\bm{X}+ t (\bm{Y} - \bm{X}))}{d t} \Big{|}_{t=0} \bm{q}_i = (\lambda_i \bm{I}_d-\bm{X})^{\dag} (\bm{Y}-\bm{X}) \bm{q}_i$, where
		$$(\lambda_i \bm{I}_d-\bm{X})^{\dag}=\sum_{j\in [d],j\neq i} \frac{1}{\lambda_i-\lambda_j} \bm{q}_j \bm{q}_j^{\top}.$$
		
		Thus, $\big< B , \bm{Y}-\bm{X}\big>$ is equivalent to
		\begin{align}
		\big< B , \bm{Y}-\bm{X}\big>=&\sum_{i \in [d]}\frac{1}{\beta_i} \sum_{j\in [d],j\neq i} \frac{1}{\lambda_i-\lambda_j}  \Big(\bm{q}_j^{\top}(\bm{Y}-\bm{X}) \bm{q}_i \Big)^2+\sum_{j \in [d]}\frac{1}{\beta_j}  \sum_{i\in [d],i\neq j} \frac{1}{\lambda_j-\lambda_i}  \Big(\bm{q}_j^{\top}(\bm{Y}-\bm{X}) \bm{q}_i \Big)^2\nonumber\\
		=& \sum_{i \in [d]} \sum_{j \in [d],j\neq i} \Big(\frac{1}{\beta_i} \frac{1}{\lambda_i-\lambda_j}+\frac{1}{\beta_j} \frac{1}{\lambda_j-\lambda_i}\Big)  \Big(\bm{q}_j^{\top}(\bm{Y}-\bm{X}) \bm{q}_i \Big)^2. \label{eq_cross}
		\end{align}
	Above, we can split the summations in the right-hand side of \eqref{eq_cross} into four cases and also by plugging the values of $\bm{\beta}$, we can rewrite $\big< B , \bm{Y}-\bm{X}\big>$ as
		\begin{align}
		\big< B , \bm{Y}-\bm{X}\big>=& \sum_{i \in [k]} \sum_{j \in [k],j\neq i} \frac{1}{\lambda_i\lambda_j}   \Big(\bm{q}_j^{\top}(\bm{Y}-\bm{X}) \bm{q}_i \Big)^2+\sum_{i \in [k+1,d]} \sum_{j \in [k+1,d],j\neq i} 0\notag\\
		&+ \sum_{i \in [k]} \sum_{j \in [k+1,d],j\neq i}\Big( \frac{1}{\lambda_i} \frac{1}{\lambda_i-\lambda_j} + \frac{s-k}{\sum_{\ell \in[k+1,d]} \lambda_{\ell}} \frac{1}{\lambda_j-\lambda_i}\Big)  \Big(\bm{q}_j^{\top}(\bm{Y}-\bm{X}) \bm{q}_i \Big)^2\notag\\
		&+\sum_{i \in [k+1,d]} \sum_{j \in [k],j\neq i}\Big(  \frac{s-k}{\sum_{\ell \in[k+1,d]}  \lambda_{\ell}}  \frac{1}{\lambda_i-\lambda_j} + \frac{1}{\lambda_j} \frac{1}{\lambda_j-\lambda_i}\Big) \Big(\bm{q}_j^{\top}(\bm{Y}-\bm{X}) \bm{q}_i \Big)^2 \notag\\
		&\geq -\frac{1}{\delta^2}\sum_{i \in [d]} \sum_{j\in [d],j\neq i} \Big(\bm{q}_j^{\top}(\bm{Y}-\bm{X}) \bm{q}_i \Big)^2,\label{eq_cross1}
		\end{align}
		where the inequality is because $\lambda_i > \frac{\sum_{\ell \in[k+1,d]} \lambda_{\ell}}{s-k} \ge \lambda_j$ for each pair $(i,j) \in [k]\times [k+1,d]$, and $\frac{\sum_{\ell \in[k+1,d]} \lambda_{\ell}}{s-k} \geq \delta$ by Claim~\ref{claim1}.

	\end{enumerate}	
	
	\noindent\textbf{Step (iv)- Combining All the Pieces Together.}
	According to the results \eqref{deriv}, \eqref{eq_hess1}, and \eqref{eq_cross1}, we can derive that
	\begin{align*}
	h''(0)& \ge - \frac{1}{\delta^2}\sum_{i \in [d]} \sum_{j\in [d]} \Big(\bm{q}_j^{\top}(\bm{Y}-\bm{X}) \bm{q}_i \Big)^2= - \frac{1}{\delta^2}\tr(((\bm{Y}-\bm{X})\bm{Q})^2) \\
	&\ge- \frac{1}{\delta^2} \|\bm{Y}-\bm{X}\|^2_2\\
	&\ge- \frac{1}{\delta^2} \lambda_{\max}^2(\bm C) \|\bm{y}-\bm{x}\|_2^2,
	\end{align*}
	where the second inequality is due to Cauchy-Schwartz inequality and that matrix $\bm Q$ is orthonormal, and the third inequality stems from the fact that $\|\bm{Y}-\bm{X}\|^2_2 =\Vert \bm{V} \Diag(\bm{y}-\bm{x})\bm{V}^{\top} \Vert_2^2  \le  \lambda_{\max}^2(\bm C) \|\bm{y}-\bm{x}\|^2_2$. 
	\qed
\end{proof}
\end{subequations}

\subsection{Proof of \Cref{lem:rankupd}}\label{proof_lem_rankupd}
\lemrankupd*
\begin{subequations}
 \begin{proof} 	\noindent\textbf{Part (i).} Let $\bm{X}_{-i}= \bm{Q} \Diag(\bm{\lambda}) \bm{Q}^{\top}$ denote its eigendecomposition. Since the rank of $\bm{X}_{-i}$ is $\tau-1$, without loss of generality, we assume that its eigenvalues satisfy $\lambda_1 \ge \cdots \lambda_{\tau-1} > \lambda_{\tau}=\cdots = \lambda_d=0$.
 	
 	For any $\epsilon > 0$, %
 	we have
 	\begin{align*}
 	\det \left( \bm{X}  + \epsilon \bm{I}_d \right) &=  \det (\bm{X}_{-i} + \epsilon \bm{I}_d) \left(1+ \bm{v}_i^{\top} (\bm{X}_{-i} + \epsilon \bm{I}_d)^{-1} \bm{v}_i \right)  \\
 	& =  \epsilon^{n-\tau+1} \prod_{i \in [\tau-1]} (\lambda_{i} + \epsilon)  \left (1+ \bm{v}_i^{\top} (\bm{X}_{-i} + \epsilon \bm{I}_d)^{-1} \bm{v}_i \right)\\
 	& =   \epsilon^{n-\tau} \prod_{i \in [\tau-1]} (\lambda_{i} + \epsilon)  \left( \epsilon + \bm{v}_i^{\top} \bm{Q} \Diag(\bm{\beta}(\epsilon)) \bm{Q}^{\top} \bm{v}_i \right), 
 	\end{align*}
 	where the first equality is from the Matrix Determinant lemma \citep{harville1998matrix} and in the third equality, we let $\bm{\beta}(\epsilon) = (\frac{\epsilon}{{\lambda_1}+\epsilon}, \cdots, \frac{\epsilon}{{\lambda_{\tau-1}}+\epsilon},1, \cdots, 1)^{\top}$ denote the eigenvalues of $\epsilon(\bm{X}_{-i} + \epsilon \bm{I}_d)^{-1}$.
 	As $\overset{\tau}{\det} (\bm{X}) = \lim_{\epsilon \to 0} \epsilon^{-(n-\tau)} \det \left(\bm{X} + \epsilon \bm{I}_d \right)$, thus
 	\begin{align*}
 	\det^{\tau} (\bm{X}) &= \lim_{\epsilon \to 0} \frac{ \det \left( \bm{X} + \epsilon \bm{I}_d\right) }{\epsilon^{n-\tau}} = \lim_{\epsilon \to 0}  \prod_{i \in [\tau-1]} (\lambda_{i} + \epsilon)  \left( \epsilon + \bm{v}_i^{\top} \bm{Q} \Diag(\bm{\beta}(\epsilon)) \bm{Q}^{\top} \bm{v}_i \right) \\
 	&=\lim_{\epsilon \to 0}  \prod_{i \in [\tau-1]} (\lambda_{i} + \epsilon)\lim_{\epsilon \to 0}  \left( \epsilon + \bm{v}_i^{\top} \bm{Q} \Diag(\bm{\beta}(\epsilon)) \bm{Q}^{\top} \bm{v}_i \right)\\
 	& = \det^{\tau-1} (\bm{X}_{-i}) \left( \bm{v}_i^{\top} \bm{Q} \Diag\left ( \bm{\beta}(0) \right) \bm{Q}^{\top} \bm{v}_i \right) = \det^{\tau-1} (\bm{X}_{-i}) \bm{v}_i^{\top}(\bm{I}_d-
 	\bm{X}_{-i}^{\dag}\bm{X}_{-i})\bm{v}_i,
 	\end{align*}
 	where the third equality is because both limits exist and the last equality is from the fact that the vector of eigenvalues of $(\bm{I}_d-
 	\bm{X}_{-i}^{\dag}\bm{X}_{-i})$ is equal to $\bm{\beta}(0)$ and the corresponding matrix consisting of the eigenvectors is $\bm{Q}$.
 	
 	\noindent The proof of \textbf{Part (ii)} is similar to  \textbf{Part (i)} and is thus omitted here.
 	
 	\noindent \textbf{Part (iii)} and \textbf{Part (iv)} follow directly from theorem 1 and theorem 6 in \cite{meyer1973generalized}.
 	
 	\noindent\textbf{Part (v).} By \textbf{Part (iii)} and the fact that $(\bm{I}_d-\bm{X}_{-i}^{\dag}\bm{X}_{-i})$ is a projection matrix, we have
 	\begin{align*}
 	\bm{v}_i^{\top} \bm{X}^{\dag} \bm{v}_i
 	=&\bm{v}_i^{\top} \bm{X}_{-i}^{\dag} \bm{v}_i - 
 	\frac{\bm{v}_i^{\top} \bm{X}_{-i} \bm{v}_i \bm{v}_i^{\top}(\bm{I}_d-\bm{X}_{-i}^{\dag}\bm{X}_{-i}) \bm{v}_i}{\|(\bm{I}_d-\bm{X}_{-i}^{\dag}\bm{X}_{-i}) \bm{v}_i \|_{2}^{2}} -
 	\frac{\bm{v}_i^{\top}( \bm{I}_d-\bm{X}_{-i}^{\dag}\bm{X}_{-i}) \bm{v}_i \bm{v}_i^{\top} \bm{X}_{-i} \bm{v}_i}{\|(\bm{I}_d-\bm{X}_{-i}^{\dag}\bm{X}_{-i}) \bm{v}_i \|_{2}^{2}} \\
 	& + \frac{(1+\bm{v}_i^{\top} \bm{X}_{-i} \bm{v}_i)\bm{v}_i^{\top}( \bm{I}_d-\bm{X}_{-i}^{\dag}\bm{X}_{-i}) \bm{v}_i \bm{v}_i^{\top} (\bm{I}_d-\bm{X}_{-i}^{\dag}\bm{X}_{-i}) \bm{v}_i }{\|(\bm{I}_d-\bm{X}_{-i}^{\dag}\bm{X}_{-i}) \bm{v}_i \|_{2}^{4}}\\
 	=& \bm{v}_i^{\top} \bm{X}_{-i}^{\dag} \bm{v}_i - \bm{v}_i^{\top} \bm{X}_{-i}^{\dag} \bm{v}_i - \bm{v}_i^{\top} \bm{X}_{-i}^{\dag} \bm{v}_i +1 + \bm{v}_i^{\top} \bm{X}_{-i}^{\dag} \bm{v}_i = 1.
 	\end{align*}
 	
 	\noindent\textbf{Part (vi).} Since $\bm{X}= \bm{X}_{-i}+ \bm{v}_i \bm{v}_i^{\top}$, then we have 
 	\begin{align*}
 	\bm{v}_i^{\top} (\bm{I}_d- \bm{X}^{\dag}\bm{X}) = \bm{v}_i^{\top} - \bm{v}_i^{\top} \bm{X}^{\dag}\bm{X}_{-i}  - \bm{v}_i^{\top} \bm{X}^{\dag} \bm{v}_i  \bm{v}_i^{\top} = - \bm{v}_i^{\top} \bm{X}^{\dag}\bm{X}_{-i},
 	\end{align*}
 	where the second equality is from the fact that $\bm{v}_i^{\top} \bm{X}^{\dag} \bm{v}_i = 1$ in \textbf{Part (v)}.
 	
 	To compute $\bm{v}_i^{\top} \bm{X}^{\dag}  \bm{X}_{-i}$, using the result in \textbf{Part (iii)} and the facts that $(\bm{I}_d-\bm{X}_{-i}^{\dag}\bm{X}_{-i})\bm{X}_{-i}^{\dag}=0$ and $(\bm{I}_d-\bm{X}_{-i}^{\dag}\bm{X}_{-i})$ is a projection matrix, we then obtain %
 	\begin{align*}
 	\bm{v}_i^{\top} \bm{X}^{\dag}  \bm{X}_{-i} 
 	=&\bm{v}_i^{\top} \bm{X}_{-i}^{\dag} \bm{X}_{-i}^{\dag}  - 
 	\frac{\bm{v}_i^{\top} \bm{X}_{-i}^{\dag} \bm{v}_i \bm{v}_i^{\top}(\bm{I}_d-\bm{X}_{-i}^{\dag}\bm{X}_{-i})\bm{X}_{-i}^{\dag} }{\|(\bm{I}_d-\bm{X}_{-i}^{\dag}\bm{X}_{-i}) \bm{v}_i \|_{2}^{2}} 
 	-\frac{\bm{v}_i^{\top} (\bm{I}_d-\bm{X}_{-i}^{\dag}\bm{X}_{-i}) \bm{v}_i \bm{v}_i^{\top} \bm{X}_{-i}^{\dag} \bm{X}_{-i}^{\dag} }{\|(\bm{I}_d-\bm{X}_{-i}^{\dag}\bm{X}_{-i}) \bm{v}_i \|_{2}^{2}} \\
 	& + \frac{(1+\bm{v}_i^{\top} \bm{X}_{-i}^{\dag} \bm{v}_i)\bm{v}_i^{\top}(\bm{I}_d-\bm{X}_{-i}^{\dag}\bm{X}_{-i}) \bm{v}_i \bm{v}_i^{\top} (\bm{I}_d-\bm{X}_{-i}^{\dag}\bm{X}_{-i})\bm{X}_{-i}^{\dag} }{\|(\bm{I}_d-\bm{X}_{-i}^{\dag}\bm{X}_{-i}) \bm{v}_i \|_{2}^{4}} \\
 	=& \bm{v}_i^{\top} \bm{X}_{-i}^{\dag} \bm{X}_{-i}^{\dag}  -\bm{v}_i^{\top} \bm{X}_{-i}^{\dag} \bm{X}_{-i}^{\dag} =  \bm{0}.
 	\end{align*}
 	Hence, $\bm{v}_i^{\top} (\bm{I}_d- \bm{X}^{\dag}\bm{X}) = - \bm{v}_i^{\top} \bm{X}^{\dag}\bm{X}_{-i}    = \bm{0}$.
 	
 	\noindent\textbf{Part (vii).} According to \textbf{Part (iv)}, we have
 	\begin{align}
 	\bm{X}_{-i}^{\dag} \bm{X} 
 	=& \bm{X}^{\dag}\bm{X} - \frac{\bm{X}^{\dag}\bm{v}_i \bm{v}_i^{\top} \bm{X}^{\dag}}{\|\bm{X}^{\dag} \bm{v}_i\|_2^2}- \frac{\bm{X}^{\dag}\bm{X}^{\dag}\bm{v}_i \bm{v}_i^{\top} \bm{X}^{\dag} \bm{X}}{\|\bm{X}^{\dag} \bm{v}_i\|_2^2} +
 	\frac{\bm{v}_i^{\top} (\bm{X}^{\dag})^3  \bm{v}_i\bm{X}^{\dag}\bm{v}_i \bm{v}_i^{\top} \bm{X}^{\dag} \bm{X}}{\|\bm{X}^{\dag} \bm{v}_i\|_2^4} \label{loceq1},\\
 	\bm{X}_{-i}^{\dag} \bm{v}_i \bm{v}_i^{\top} =&  \bm{X}^{\dag}\bm{v}_i \bm{v}_i^{\top} - \frac{\bm{X}^{\dag}\bm{v}_i \bm{v}_i^{\top} \bm{X}^{\dag}\bm{X}^{\dag} \bm{v}_i \bm{v}_i^{\top}}{\|\bm{X}^{\dag} \bm{v}_i\|_2^2}- \frac{\bm{X}^{\dag}\bm{X}^{\dag}\bm{v}_i \bm{v}_i^{\top} \bm{X}^{\dag} \bm{v}_i \bm{v}_i^{\top}}{\|\bm{X}^{\dag} \bm{v}_i\|_2^2} \nonumber +
 	\frac{\bm{v}_i^{\top} (\bm{X}^{\dag})^3  \bm{v}_i\bm{X}^{\dag}\bm{v}_i \bm{v}_i^{\top} \bm{X}^{\dag} \bm{v}_i \bm{v}_i^{\top}}{\|\bm{X}^{\dag} \bm{v}_i\|_2^4} \nonumber \\
 	=&- \frac{\bm{X}^{\dag}\bm{X}^{\dag}\bm{v}_i  \bm{v}_i^{\top}}{\|\bm{X}^{\dag} \bm{v}_i\|_2^2}+
 	\frac{\bm{v}_i^{\top} (\bm{X}^{\dag})^3  \bm{v}_i\bm{X}^{\dag}\bm{v}_i  \bm{v}_i^{\top}}{\|\bm{X}^{\dag} \bm{v}_i\|_2^4}. \label{loceq2}
 	\end{align}
 	where the third equality is due to $\bm{v}_i^{\top} \bm{X}^{\dag} \bm{v}_i=1$ from \textbf{Part (v)}. 
 	
 	Since $\bm{X}= \bm{X}_{-i}+\bm{v}_i \bm{v}_i^{\top}$, we can obtain
 	\begin{align*}
 	\bm{v}_i^{\top}(\bm{I}_d-\bm{X}_{-i}^{\dag}\bm{X}_{-i})\bm{v}_i =
 	\bm{v}_i^{\top}\bm{v}_i -\bm{v}_i^{\top}\bm{X}_{-i}^{\dag}(\bm{X}-\bm{v}_i \bm{v}_i^{\top})\bm{v}_i = \bm{v}_i^{\top}\bm{v}_i -\bm{v}_i^{\top}\bm{X}_{-i}^{\dag}\bm{X}\bm{v}_i +\bm{v}_i^{\top}\bm{X}_{-i}^{\dag}\bm{v}_i \bm{v}_i^{\top}\bm{v}_i.
 	\end{align*}
 	Applying the identities in \eqref{loceq1} and \eqref{loceq2}, we further have
 	\begin{align*}
 	\bm{v}_i^{\top}(\bm{I}_d-\bm{X}_{-i}^{\dag}\bm{X}_{-i})\bm{v}_i 
 	= & \frac{1}{\|\bm{X}^{\dag} \bm{v}_i\|_2^2}+ \frac{\bm{v}_i^{\top} (\bm{X}^{\dag})^3  \bm{v}_i \bm{v}_i^{\top} (\bm{I}_d- \bm{X}^{\dag}\bm{X}) \bm{v}_i}{\|\bm{X}^{\dag} \bm{v}_i\|_2^4}
 	= \frac{1}{\|\bm{X}^{\dag} \bm{v}_i\|_2^2}.
 	\end{align*}
 	where the last equality is due to the fact that $\bm{v}_i^{\top} (\bm{I}_d- \bm{X}^{\dag}\bm{X}) = \bm{0}$ from \textbf{Part (vi)}.\\
 	
 	\noindent\textbf{Part (viii).} There are two cases: whether $\bm{v}_j$ is in the column space of $\bm{X}_{-i}$ or not. 
 	\begin{enumerate}[(a)]
 		\item If $\bm{v}_j \notin \col(\bm{X}_{-i})$, we follow the proof of \textbf{Part (vii)}. Since $\bm{X}= \bm{X}_{-i}+\bm{v}_i \bm{v}_i^{\top}$, we can obtain 
 		\begin{align*}
 		\bm{v}_j^{\top}(\bm{I}_d-\bm{X}_{-i}^{\dag}\bm{X}_{-i})\bm{v}_j 
 		=& \bm{v}_j^{\top}\bm{v}_j -\bm{v}_j^{\top}\bm{X}_{-i}^{\dag}\bm{X}\bm{v}_j +\bm{v}_j^{\top}\bm{X}_{-i}^{\dag}\bm{v}_i \bm{v}_i^{\top}\bm{v}_j = \bm{v}_j^{\top} (\bm{I}_d-\bm{ X}^{\dag} \bm{X}) \bm{v}_j + \frac{(\bm{v}_j^{\top} \bm{ X}^{\dag} \bm{v}_i)^2}{\|\bm{ X}^{\dag}\bm{v}_i\|^2_2} \\
 		&- \bm{v}_i^{\top}(\bm{I}_d - \bm{X}^{\dag}\bm{X})\bm{v}_j \frac{\bm{v}_j^{\top}\bm{X}^{\dag}\bm{X}^{\dag}\bm{v}_i }{\|\bm{X}^{\dag} \bm{v}_i\|_2^2} + \bm{v}_i^{\top}(\bm{I}_d - \bm{X}^{\dag}\bm{X})\bm{v}_j \frac{\bm{v}_i^{\top}(\bm{X}^{\dag})^3\bm{v}_i \bm{v}_j^{\top} \bm{X}^{\dag} \bm{v}_i }{\|\bm{X}^{\dag} \bm{v}_i\|_2^4} \\
 		=& \bm{v}_j^{\top} (\bm{I}_d-\bm{ X}^{\dag} \bm{X}) \bm{v}_j + \frac{(\bm{v}_j^{\top} \bm{ X}^{\dag} \bm{v}_i)^2}{\|\bm{ X}^{\dag}\bm{v}_i\|^2_2}, 
 		\end{align*}
 		where the second equality is due to the identites in \eqref{loceq1} and \eqref{loceq2}, and the last equality is because $\bm{v}_i^{\top} (\bm{I}_d- \bm{X}^{\dag}\bm{X}) = \bm{0}$ from \textbf{Part (vi)}.
 		\item 	Second, if $\bm{v}_j \in \col(\bm{X}_{-i})$, then we rewrite $\bm{v}_{j} = \sum_{\ell \in \hat{S}\setminus \{i\}} a_{\ell} \bm{v}_l$, which stems from the fact that the vectors $\{\bm{v}_{\ell}, \ell \in \hat{S}\setminus \{i\}\}$ span the column space of $\bm{X}_{-i}$. Then it follows that 
 		\begin{align*}
 		\bm{v}_j^{\top}(\bm{I}_d-\bm{X}_{-i}^{\dag}\bm{X}_{-i})\bm{v}_j = \sum_{\ell \in S\setminus \{i\}} a_{\ell} \bm{v}_{\ell}^{\top}  (\bm{I}_d-\bm{X}_{-i}^{\dag}\bm{X}_{-i})\bm{v}_j  = 0,
 		\end{align*}
 		where the second equality is because $\bm{v}_{\ell}^{\top} (\bm{I}_d- \bm{X}_{-i}^{\dag}\bm{X}_{-i}) = \bm{0}$ for all $\ell \in \hat{S}\setminus \{i\}$ from \textbf{Part (vi)}.
 	\end{enumerate}
 	\qed
 \end{proof}
\end{subequations}

\subsection{Proof of \Cref{lem:rankupd2}} \label{proof_lem_rankupdtwo}
\lemrankupdtwo*
\begin{subequations}
\begin{proof}
	For each pair $(i,j)\in \hat{S} \times ([n]\setminus \hat{S})$, the stopping criterion of Algorithm \ref{algo} implies that 
	\begin{align} \label{loc_opt}
	\det \limits^{s} ({\bm{X}}_{-i}+\bm{v}_i\bm{v}_i^{\top}) \ge \det \limits^{s} ({\bm{X}}_{-i} +\bm{v}_j\bm{v}_j^{\top}), 
	\end{align}
	and $\{\bm{v}_{{\ell}}\}_{{\ell}\in \hat{S}}$ are linearly independent. There are two cases to be considered: whether $\bm{v}_j$ is in the column space of $\bm{X}_{-i}$ or not.
	\begin{enumerate} [(i)]\setlength{\itemsep}{0pt}
		\item  If $\bm{v}_j \notin \col (\bm{X}_{-i})$, then by Parts (i) and (ii) in Lemma \ref{lem:rankupd} and the fact that $\det \limits^{s-1} ({\bm{X}}_{-i})>0$, the local optimality condition \eqref{loc_opt} is equivalent to
		\begin{align}
		\bm{v}_i^{\top}(\bm{I}_d-{\bm{ X}}_{-i}^{\dag}{\bm{X}}_{-i})\bm{v}_i \ge \bm{v}_j^{\top} (\bm{I}_d-{\bm{X}}_{-i}^{\dag}{\bm{X}}_{-i})\bm{v}_j . \label{eq_s}
		\end{align}
		Plugging the results of Parts (vii) and (viii) in Lemma \ref{lem:rankupd},  the above inequality is further reduced to
		\begin{align}
		1 \ge \left (\bm{v}_i^{\top}\bm{X}^{\dag} \bm{X}^{\dag}\bm{v}_i\right) \bm{v}_j^{\top}  (\bm{I}_d - \bm{X}^{\dag} \bm{X}) \bm{v}_j + \bm{v}_j^{\top} \bm{X}^{\dag} \bm{v}_i\bm{v}_i^{\top} \bm{X}^{\dag} \bm{v}_j.\label{eq_s2}
		\end{align}
		
		\item If $\bm{v}_j\in \col (\bm{X}_{-i})$, then we must have $\bm{v}_j\in \col (\bm{X})$. According to Part (vi) in Lemma \ref{lem:rankupd}, we have
		$$ \left (\bm{v}_i^{\top}\bm{X}^{\dag} \bm{X}^{\dag}\bm{v}_i\right) \bm{v}_j^{\top}  (\bm{I}_d - \bm{X}^{\dag} \bm{X}) \bm{v}_j + \bm{v}_j^{\top} \bm{X}^{\dag} \bm{v}_i\bm{v}_i^{\top} \bm{X}^{\dag} \bm{v}_j = (\bm{v}_i^{\top} \bm{X}^{\dag} \bm{v}_j)^2.$$
		Using Part (iii) in Lemma \ref{lem:rankupd}, we have
		\begin{align*}
		\bm{v}_i^{\top} \bm{X}^{\dag} \bm{v}_j=&\bm{v
		}_i^{\top} \bm{X}_{-i}^{\dag} \bm{v}_j - 
		\frac{\bm{v}_i^{\top} \bm{X}_{-i} \bm{v}_i \bm{v}_i^{\top}(\bm{I}_d-\bm{X}_{-i}^{\dag}\bm{X}_{-i}) \bm{v}_j}{\|(\bm{I}_d-\bm{X}_{-i}^{\dag}\bm{X}_{-i}) \bm{v}_i \|_{2}^{2}} -
		\frac{\bm{v}_i^{\top}( \bm{I}_d-\bm{X}_{-i}^{\dag}\bm{X}_{-i}) \bm{v}_i \bm{v}_i^{\top} \bm{X}_{-i} \bm{v}_j}{\|(\bm{I}_d-\bm{X}_{-i}^{\dag}\bm{X}_{-i}) \bm{v}_i \|_{2}^{2}} \\
		& + \frac{(1+\bm{v}_i^{\top} \bm{X}_{-i} \bm{v}_i)\bm{v}_i^{\top}( \bm{I}_d-\bm{X}_{-i}^{\dag}\bm{X}_{-i}) \bm{v}_i \bm{v}_i^{\top} (\bm{I}_d-\bm{X}_{-i}^{\dag}\bm{X}_{-i}) \bm{v}_j }{\|(\bm{I}_d-\bm{X}_{-i}^{\dag}\bm{X}_{-i}) \bm{v}_i \|_{2}^{4}}\\
		=&\bm{v}_i^{\top} \bm{X}_{-i}^{\dag} \bm{v}_j -\frac{\bm{v}_i^{\top}( \bm{I}_d-\bm{X}_{-i}^{\dag}\bm{X}_{-i}) \bm{v}_i \bm{v}_i^{\top} \bm{X}_{-i} \bm{v}_j}{\|(\bm{I}_d-\bm{X}_{-i}^{\dag}\bm{X}_{-i}) \bm{v}_i \|_{2}^{2}}=0,
		\end{align*}
		where the first equality is due to Part (iii) in Lemma \ref{lem:rankupd}, the second equality is due to Part (vi) in Lemma \ref{lem:rankupd} and $\bm{v}_j$ is a linear combination of $\{\bm{v}_\ell\}_{ \ell \in \hat{S}\setminus \{i\} }$, and the last equality is because $(\bm{I}_d-\bm{X}_{-i}^{\dag}\bm{X}_{-i})$ is a projection matrix.
		
		Thus, clearly, we arrive at
		$$ \left (\bm{v}_i^{\top}\bm{X}^{\dag} \bm{X}^{\dag}\bm{v}_i\right) \bm{v}_j^{\top}  (\bm{I}_d - \bm{X}^{\dag} \bm{X}) \bm{v}_j + \bm{v}_j^{\top} \bm{X}^{\dag} \bm{v}_i\bm{v}_i^{\top} \bm{X}^{\dag} \bm{v}_j = (\bm{v}_i^{\top} \bm{X}^{\dag} \bm{v}_j)^2=0\leq 1.$$\qed
	\end{enumerate}

\end{proof}
\end{subequations}

\subsection{Proof of \Cref{them1}}\label{proof_them1}
\themlocal*
\begin{subequations}
\begin{proof}
	We split the proof into three steps.
	
	\noindent\textbf{Step 1. Constructing Solution of Dual Variable $\bm{ \Lambda}$.}\par
	Given the output $\hat{S}$ of the local search Algorithm \ref{algo}, %
	let us denote $\bm{X} = \sum_{i \in \hat{S}} \bm{v}_i\bm{v}_i^{\top}$ and let ${\bm{X}}_{-i} = \bm{X} - \bm{v}_i \bm{v}_i^{\top}$ for each $ i \in \hat{S}$.

	We first construct $\bm{ \Lambda}$ of LD \eqref{eq_dual} as below
	\begin{align}
	\bm{ \Lambda} = \frac{1}{t}\left[\tr(\bm{X}^{\dag})(\bm{I}_d - \bm{ X}^{\dag} \bm{ X}) + \bm{X}^{\dag}\right],%
	\label{eq-construc_lambda}
	\end{align}
	where $t>0$ is a scaling factor and will be specified later. Accordingly, the identity \eqref{eq-construc_lambda} leads to that $\log \underset{s}{\det}  (\bm{\Lambda}) = \log \overset{s}{\det} (\bm{X})+s\log t$.
	
	\noindent\textbf{Step 2. Constructing Solution of the Other Dual Variables $({\nu}, {\bm{\mu}})$ with $\bm{ \Lambda}$ in \eqref{eq-construc_lambda}.}\par	
	Next, to construct the solution of the other two dual variables $({\nu}, {\bm{\mu}})$, we need to check the feasibility of constraints in LD \eqref{eq_dual}, i.e., 
	\begin{align}
	\bm{v}_i^{\top} \bm{ \Lambda} \bm{v}_i\le \nu+\mu_i,\forall i \in [n].\label{eq_possible_nu_mu}
	\end{align}
	We consider the following two cases: (i) for each $ i \in \hat{S}$ and (ii) for each  $j \in [n]\setminus \hat{S}$. 
	
	\begin{enumerate}[(i)]  \setlength{\itemsep}{0pt}
		\setlength{\parskip}{0pt}
		\item For each $ i \in \hat{S}$, we have
		\begin{align} \label{loc_constr}
		\bm{v}_i^{\top} \bm{\Lambda} \bm{v}_i = \frac{1}{t}\left[\tr(\bm{ X}^{\dag}) \bm{v}_i^{\top} (\bm{I}_d - \bm{ X}^{\dag} \bm{ X}) \bm{v}_i+ \bm{v}_i^{\top} \bm{ X}^{\dag} \bm{v}_i\right]  = \frac{1}{t},
		\end{align}
		where the second equality results from Parts (v) and (vi) in Lemma \ref{lem:rankupd} with $\tau=s$.
		\item For each  $j \in [n]\setminus \hat{S}$, according to Lemma \ref{lem:rankupd2}, we have%
		\begin{align*}
		1 \ge \left (\bm{v}_i^{\top}\bm{X}^{\dag} \bm{X}^{\dag}\bm{v}_i\right) \bm{v}_j^{\top}  (\bm{I}_d - \bm{X}^{\dag} \bm{X}) \bm{v}_j + \bm{v}_j^{\top} \bm{X}^{\dag} \bm{v}_i\bm{v}_i^{\top} \bm{X}^{\dag} \bm{v}_j, \forall i \in \hat{S}.
		\end{align*}
		Summing the above inequality over $i \in \hat{S}$ and using the fact that $\bm{X} = \sum_{i \in \hat{S}} \bm{v}_i \bm{v}_i^{\top}$, we have
		\begin{align}\label{loc_constrjone}
		s \ge \tr(\bm{X}^{\dag}) \bm{v}_j^{\top} (\bm{I}_d - \bm{X}^{\dag} \bm{X}) \bm{v}_j+ \bm{v}_j^{\top} \bm{X}^{\dag} \bm{v}_j = t\bm{v}_j^{\top} \bm{\Lambda} \bm{v}_j.
		\end{align} 
	\end{enumerate}

	By inequalities \eqref{loc_constr} and \eqref{loc_constrjone}, to find the best $(\nu, \bm{\mu})$, it suffices to solve the optimization problem below:
	\begin{align*}
	z^{LD}\leq \min_{t>0}\min_{ \nu, \bm{\mu} \in \Re_{+}^n} \bigg\{ \log \overset{s}{\det} (\bm{X})+s\log (t) +s\nu+\sum_{i \in [n]} \mu_i-s :  \nu+\mu_i \ge \frac{1}{t} , \forall i \in \hat{S} ,\nu+\mu_i \ge \frac{s}{t} , \forall i \in [n]\setminus\hat{S} \bigg\}.
	\end{align*}
	Above, by checking the primal and dual of inner minimization problems, there are following two candidate optimal solutions
	\begin{align*}
	& \nu^a = \frac{s}{t},  \mu_i^a= 0, \forall i \in [n], \\
	& \nu^b = \frac{1}{t}, \mu_i^b = 0 ,\forall i \in \hat{S}, \mu_i^b  = \frac{s-1}{t} ,\forall i \in [n]\setminus \hat{S}.
	\end{align*}

	\noindent\textbf{Step 3. Finding the Best Scaler $t$ and Proving the Approximation Bound.}\par	
	
	Plugging in these two candidate solutions of $(\nu,\bm\mu)$, the right-hand side of the above minimization problem becomes
	\begin{align*}
	z^{LD}\leq\log \overset{s}{\det} (\bm{X})  +\min_{t>0} \min\left\{ s\log (t) + s\left(\frac{s}{t}-1\right), s\log (t) + (n-s)\frac{s-1}{t}+\frac{s}{t}-s \right\}.
	\end{align*}
	By swapping the two minimum operators and optimizing over $t$, the right-hand side of above inequality is further equivalent to
	\begin{align*}
	z^{LD}\leq\log \overset{s}{\det} (\bm{X})  + s \min\left\{ \log (s), \log \left(n-s-\frac{n}{s}+2\right)  \right\}.
	\end{align*}

	According to the weak duality between MESP \eqref{eq_obje} and LD \eqref{eq_dual} and the fact that $\hat{S}$ is feasible to MESP \eqref{mesp}, we have
	\begin{align*}
	\log \det^s \bigg(\sum_{i \in \hat{S}} \bm{v}_i\bm{v}_i^{\top}\bigg) =	\log \det^s (\bm{X}) \le z^* \le z^{LD} \le  \log \overset{s}{\det} (\bm{X}) + s\min\left\{ \log (s), \log \left(n-s-\frac{n}{s}+2\right) \right\},
	\end{align*}
	which completes the proof. \qed
\end{proof}
\end{subequations}

\subsection{Proof of \Cref{prop:degenerate}} \label{proof_prop_deg}
\propdeg*
\begin{proof}
	We construct the following instance.
	{\begin{example}\label{examp2} Given $s\leq d\leq n$, suppose that for each $i\in [n]$, 
			\begin{align*}
			\bm{v}_i =\begin{cases}
			\bm{e}_i, & \textrm{ if } i \in [s],\\
			\sum_{j \in [s]}\bm{e}_j, & \textrm{ otherwise}.
			\end{cases}
			\end{align*}
	\end{example}}
	In the above example, one optimal solution to MESP \eqref{eq_obj} is $S^*=[s]$.
	Suppose in the local search Algorithm \ref{algo}, we start with $\hat{S}=S^*$, then it terminates immediately. We follow \eqref{feas_lambda} to construct a feasible $\bm{\Lambda}$ to LD, which is identical to the one \eqref{eq-construc_lambda} used in Theorem \ref{them1}. According to the proof of Theorem \ref{them1}, we only need to check if the inequalities \eqref{loc_constrjone} are tight, i.e., 
	$$s = \tr(\bm{X}^{\dag}) \bm{v}_j^{\top} (\bm{I}_d - \bm{X}^{\dag} \bm{X}) \bm{v}_j+ \bm{v}_j^{\top} \bm{X}^{\dag} \bm{v}_j = t\bm{v}_j^{\top} \bm{\Lambda} \bm{v}_j,\forall j\in [s+1,n].$$
	In fact,
	\begin{align*}
	\tr(\bm{X}^{\dag}) \bm{v}_j^{\top} (\bm{I}_d - \bm{X}^{\dag} \bm{X}) \bm{v}_j+ \bm{v}_j^{\top} \bm{X}^{\dag} \bm{v}_j &= \tr(\bm{X}^{\dag}) \bigg(\sum_{i \in [s]} \bm{e}_i \bigg)^{\top} (\bm{I}_d - \bm{X}^{\dag} \bm{X})  \bigg(\sum_{i \in [s]} \bm{e}_i \bigg) + \bm{v}_j^{\top} \bm{X}^{\dag} \bm{v}_j \\
	&= \sum_{i \in [s]} \bm{e}_i^{\top} \bm{X}^{\dag} \bm{e}_i  = s, \forall j \in [s+1,n],
	\end{align*} 
	where the second equality is due to Part (vi) in Lemma~\ref{lem:rankupd} with $\tau=s$ and the third one is due to $\bm{X}=\sum_{i \in [s]} \bm{e}_i \bm{e}_i^{\top}$ and $\bm{e}_i^{\top} \bm{X}^{\dag} \bm{e}_{\ell}=0$ for all $i,{\ell}\in [s]$ and $i\neq \ell$.
	\qed
\end{proof}

\subsection{Proof of \Cref{prop_ls}} \label{proof_prop_ls}
\propls*
\begin{proof}	The proof follows directly from Theorem \ref{them1}. Thus, we only sketch the proof for the sake of page limit.\par 
	
	\noindent{\textbf{Step 0.}} Given the output $\hat{S}$ of the local search Algorithm \ref{algo}, 
	let us denote $\bm{X} = \sum_{i \in \hat{S}} \bm{v}_i\bm{v}_i^{\top}$. Let  $\lambda_1  \ge \cdots \ge \lambda_s > \lambda_{s+1} = \cdots = \lambda_d =0$ denote the eigenvalues of $\bm{X}$. Clearly, according to the definition of $\delta$ and Cauchy's Interlacing theorem \citep{bellman1997introduction}, we have $\lambda_{\max}(\bm C)\geq \lambda_1$ and $\lambda_s \geq \delta$.
	
	\noindent{\textbf{Step 1.}} Construct $\bm{ \Lambda} =(\lambda_s t)^{-1}(\bm{I}_d - \bm{ X}^{\dag} \bm{ X}) +t^{-1} \bm{X}^{\dag}$ such that $\log \overset{s}{\det}  (\bm{X}) = -\log \underset{s}{\det} (\bm{\Lambda})+s\log t$.

	\noindent{\textbf{Step 2.}} We can show that $\bm{v}_i^{\top} \bm{\Lambda} \bm{v}_i = 1/t$ for all $i \in \hat{S}$. 
	
	Since the vectors $\{\bm{v}_i\}_{i \in \hat{S}}$ span the column space of $\bm{X}$, the assumption that $\bm{v}_i^{\top} \bm{v}_j = 0$ for each pair $ (i,j) \in \hat{S}\times ([n] \setminus \hat{S})$ implies that $\bm{v}_j$ is orthogonal to the column space of $\bm{X}$. Thus, we have
	\begin{align*}
	\bm{v}_j^{\top} \bm{X} = \bm{0} , \bm{v}_j^{\top} \bm{X}^{\dag} = \bm{0} , \bm{v}_j^{\top} \bm{\Lambda} \bm{v}_j = (\lambda_st)^{-1}\bm{v}_j^{\top} \bm{v}_j, \forall j\in [n]\setminus \hat{S}.
	\end{align*}
	To obtain the upper bound of $\bm{v}_j^{\top} \bm{\Lambda} \bm{v}_j$, according to Lemma~\ref{lem:rankupd2}, we have
	\begin{align*}
	1 \ge \left (\bm{v}_i^{\top}\bm{X}^{\dag} \bm{X}^{\dag}\bm{v}_i\right) \bm{v}_j^{\top}  (\bm{I}_d - \bm{X}^{\dag} \bm{X}) \bm{v}_j + \bm{v}_j^{\top} \bm{X}^{\dag} \bm{v}_i\bm{v}_i^{\top} \bm{X}^{\dag} \bm{v}_j =  \left (\bm{v}_i^{\top}\bm{X}^{\dag} \bm{X}^{\dag}\bm{v}_i\right) \bm{v}_j^{\top} \bm{v}_j , \forall i \in \hat{S},
	\end{align*}
	where the equality is due to $  \bm{v}_j^\top\bm{X}^{\dag}=\bm0$. Summing the above inequalities over $i \in \hat{S}$, then for each $j \in [n]\setminus \hat{S}$, we have
	\begin{align*}
	\bm{v}_j^{\top} \bm{\Lambda} \bm{v}_j = \frac{1}{\lambda_st}\bm{v}_j^{\top} \bm{v}_j \le \frac{1}{\lambda_st} \frac{s}{\tr( \bm{X}^{\dag})} \le \frac{\lambda_1}{\lambda_st}\leq \frac{\lambda_{\max}(\bm C)}{\delta t},
	\end{align*}
	where the second inequality is due to $\lambda_1 \tr( \bm{X}^{\dag})\geq {s}$, and the third inequality is from $\lambda_{\max}(\bm C)\geq \lambda_1$ and $\delta\leq \lambda_s$.

	\noindent{\textbf{Step 3.}} To choose $(\nu,\bm\mu)$ such that $(\bm{\Lambda},\nu,\bm\mu)$ is feasible to LD \eqref{eq_dual},  let us consider the optimization problem below
	\begin{align*}
 &\min_{t>0}\min_{\nu, \bm{\mu} \in \Re_{+}^n}\bigg \{ \log \overset{s}{\det} (\bm{X})+s\log t +s\nu+\sum_{i \in [n]} \mu_i-s :  \nu+\mu_i \ge \frac{1}{t} , \forall i \in \hat{S} ,\nu+\mu_i \ge \frac{\lambda_{\max}(\bm{C})}{\delta t}, \forall i \in [n]\setminus\hat{S} \bigg\},
	\end{align*}
	which provides an upper bound to $z^{LD}$.
	By optimizing the right-hand side, we obtain
	\begin{align*}
	z^{LD}\leq & \log \overset{s}{\det} (\bm{X})+ \min\left\{ s\log\left(\frac{\lambda_{\max}(\bm{C})}{\delta }\right), s\log \left( \frac{\lambda_{\max}(\bm{C})}{s\delta } (n-s) +\frac{2s-n}{s}\right) \right\}.
	\end{align*}
	Invoking the weak duality between MESP \eqref{mesp} and LD \eqref{eq_dual} and the fact that $\hat{S}$ is feasible to MESP \eqref{mesp}, we conclude that
	\begin{align*}
	\log \det^s (\bm{X}) \le z^* \le z^{LD} \le  \log \overset{s}{\det} (\bm{X})+ s\min\left\{ \log\left(\frac{\lambda_{\max}(\bm{C})}{\delta }\right), \log \left( \frac{\lambda_{\max}(\bm{C})}{s\delta } (n-s) -n / s + 2\right) \right\}.
	\end{align*} \qed
\end{proof}

\subsection{Proof of \Cref{prop:aobj}} \label{proof_prop_aobj}
\propaobj*
\begin{proof}
	\noindent \textbf{Part (i).} For any size-$s$ subset $S \subseteq [n]$ with $s\geq 1$, let $\bm{X}=\sum_{\ell \in S}\bm{v}_{\ell} \bm{v}_{\ell}^{\top}$, then for any $i \in [n]\setminus S$, we have
	$$\overset{s+1}{\tr}\left[(\bm{X}+\bm{v}_i\bm{v}_i^{\top})^{\dag}\right] =\overset{s}{\tr} (\bm{X}^{\dag})+ \frac{1+\bm{v}_i^{\top} \bm{X}^{\dag} \bm{v}_i }{\bm{v}_i^{\top} (\bm{I}_n -\bm{X}\bm{X}^{\dag}) \bm{v}_i}\geq \overset{s}{\tr} (\bm{X}^{\dag}),$$
	where the equality is due to Part (iii) in Lemma \ref{lem:rankupd} , and thus proves the monotonicity.
	
	\noindent \textbf{Part (ii).} Consider an instance of $n=3$, $\bm{v}_1=2\bm{e}_1+\bm{e}_2$, $\bm{v}_2=2\bm{e}_1-\bm{e}_2$ and $\bm{v}_3\in \Re^3$. Then we let $S_1=\{1\}, S_2=\{1,2\}$ and $\bm{X}_1=\sum_{i \in S_1}\bm{v}_i\bm{v}_i^{\top}$, $\bm{X}_2=\sum_{i \in S_2}\bm{v}_i\bm{v}_i^{\top}$. In this way, we have
	\begin{align*}
	\bm{X}_1=\left(                
	\begin{array}{ccc}   
	4 & 2 & 0\\  
	2 & 1& 0\\  
	0 & 0 & 0\\
	\end{array}
	\right), \ \  \bm{X}_1^{\dag}=\left(                
	\begin{array}{ccc}   
	0.16 & 0.08 & 0\\  
	0.08 & 0.04& 0\\  
	0 & 0 & 0\\
	\end{array}
	\right), \ \  
	\bm{X}_2=\left(                
	\begin{array}{ccc}   
	8 & 0 & 0\\  
	0 & 2& 0\\  
	0 & 0 & 0\\
	\end{array}
	\right), 
	\ \ 
	\bm{X}_2^{\dag}=\left(                
	\begin{array}{ccc}   
	0.125 & 0 & 0\\  
	0 & 0.5& 0\\  
	0 & 0 & 0\\
	\end{array}
	\right) .
	\end{align*}
	
	If $\bm{v}_3=(40 \  \ 10 \ \ 20)^{\top}$, then
	\begin{align*}
	\overset{2}{\tr}\left[(\bm{X}_1+\bm{v}_3\bm{v}_3^{\top})^{\dag}\right] -\overset{1}{\tr} (\bm{X}_1^{\dag})= \frac{1+324}{480} \ge \frac{1+250}{400} =\overset{3}{\tr}\left[(\bm{X}_2+\bm{v}_3\bm{v}_3^{\top})^{\dag}\right] -\overset{2}{\tr} (\bm{X}_2^{\dag}),
	\end{align*}
	which disproves the discrete-supermodularity.
	
	If $\bm{v}_3=(10 \  \ 10 \ \ 20)^{\top}$, then
	\begin{align*}
	\overset{2}{\tr}\left[(\bm{X}_1+\bm{v}_3\bm{v}_3^{\top})^{\dag}\right] -\overset{1}{\tr} (\bm{X}_1^{\dag})= \frac{1+52}{420} \le \frac{1+62.5}{400} =\overset{3}{\tr}\left[(\bm{X}_2+\bm{v}_3\bm{v}_3^{\top})^{\dag}\right] -\overset{2}{\tr} (\bm{X}_2^{\dag}),
	\end{align*}
	which disproves the discrete-submodularity.
	
	\noindent \textbf{Part (iii).} Let us consider Example \ref{eg:prop} in Proposition \ref{pro:obj}. In this example, we consider two feasible solutions $\bm{x}^1=(1, 0)^{\top}$ and $\bm{x}^2=(0, 1)^{\top}$ of A-MESP \eqref{a_eq_obj} with $s=1$. The following two cases disprove the convexity and concavity: 
	\begin{enumerate}[{Case} 1.]
		\item 
		If $a=1$ and $b=1$, we have
		\begin{align*}
		\frac{1}{2} \overset{1}{\tr} \left[\left ( \bm{v}_1\bm{v}_1^\top \right )^{\dag} \right]+\frac{1}{2} \overset{1}{\tr} \left[ \left ( \bm{v}_2\bm{v}_2^\top \right )^{\dag}\right] = 1 \le \overset{1}{\tr} \bigg[ \bigg (\sum_{i\in[n]} \frac{x_i^1+x_i^2}{2} \bm{v}_i\bm{v}_i^\top \bigg )^{\dag}\bigg] = 2 ,
		\end{align*}
		which disproves the convexity.		
		
		\item	If $a=4$ and $b=1$, then we have
		\begin{align*}
		\frac{1}{2} \overset{1}{\tr} \left[\left ( \bm{v}_1\bm{v}_1^\top \right )^{\dag} \right]+\frac{1}{2} \overset{1}{\tr} \left[ \left ( \bm{v}_2\bm{v}_2^\top \right )^{\dag}\right] = \frac{1}{8}+\frac{1}{2} \ge \overset{1}{\tr} \bigg[ \bigg (\sum_{i\in[n]} \frac{x_i^1+x_i^2}{2} \bm{v}_i\bm{v}_i^\top \bigg )^{\dag}\bigg] = \frac{1}{2} ,
		\end{align*}
		which disproves the concavity. \qed
	\end{enumerate}
\end{proof}

\subsection{Proof of \Cref{lem:aldmax}} \label{proof_lem_aldmax}
\lemaldmax*
\begin{proof}
Following the proof of Lemma \ref{lem:ldmax}, the left-hand side of \eqref{eq_relaxed_A1} can be equivalently written as
	$$\min_{\begin{subarray}{c} \bm{\lambda} \in \Re_{+}^d\\
		\lambda_1\ge \cdots \ge \lambda_d \ge 0
		\end{subarray}} \Bigg\{  \sum_{i \in [s]}\frac{1}{\lambda_i} + \min_{ \begin{subarray}{c} \bm{Q}, \bm{\theta} \in \Re_{+}^d\\ \theta_{1} \le \cdots \le \theta_{d}
		\end{subarray}   }  \bigg\{ \sum_{i \in [d]} \theta_i \lambda_i : \bm{\theta} =\diag(\bm{Q}^{\top} \bm{\Lambda} \bm{Q} ), \bm{Q} \textrm{ is orthonormal} \bigg\} \Bigg \},$$
which can be further reduced to
	$$\min_{\begin{subarray}{c} \bm{\lambda} \in \Re_{+}^d\\
	\lambda_1\ge \cdots \ge \lambda_d \ge 0
	\end{subarray}} \bigg\{  \sum_{i \in [s]}\frac{1}{\lambda_i} +  \sum_{i \in [d]}\beta_i \lambda_i \bigg \}.$$
	Minimizing the inner problem over $\bm{\lambda}$ yields $\lambda_i = \frac{1}{\sqrt{\beta_i}}$ for any $i \in [s]$ and $\lambda_i = 0$ otherwise.
Thus,
	$$\min_{\bm{X} \succeq 0} \left\{\overset{s}{\tr}(\bm{X}^{\dag})+\tr(\bm{X}\bm{\Lambda}) \right \} = 2\sum_{i \in [s]} \sqrt{\beta_i} = 2 \ \underset{s}{\tr}\left(\bm{\Lambda}^{\frac{1}{2}}\right). $$
	\qed
\end{proof}

\subsection{Proof of \Cref{thm:apcont}} \label{proof_them_apcont}
\thmapcont*
\begin{proof}
	For A-LD \eqref{a_eq_dual}, let $\bm{x}\in \Re_{+}^n$ denote the Lagrangian multipliers associated with $\nu + \mu_i \ge \bm{v}_i^\top \bm{\Lambda} \bm{v}_i$ for each $i \in [n]$ and thus its dual is equal to
	\begin{align*}
	z^{LD}_A := \min_{\bm{x}\in \Re_{+}^n}  \max_{\bm{\Lambda} \succeq 0, \nu, \bm{\mu} \in \Re_{+}^n } \bigg\{
	2 \underset{s}{\tr}\left(\bm{\Lambda}^{\frac{1}{2}}\right) - s \nu - \sum_{i \in [n]} {\mu_i} + \sum_{i \in [n]} x_i( \nu + \mu_i - \bm{v}_i^\top \bm{\Lambda} \bm{v}_i) \bigg\},
	\end{align*}
	where according to theorem 3.2.2 in \cite{ben2012optimization}, the strong duality holds since the constraint system satisfies the relaxed Slater condition.
	
	Clearly, the inner maximization can be separated into two parts: maximization over $\bm{\Lambda} \succeq 0$ and maximization over $\nu, \bm{\mu} \in \Re_{+}^n $. 
	\begin{enumerate}[(i)]
		\item Let $\bm{X}=\sum_{i \in [n]}x_i \bm{v}_i\bm{v}_i^{\top}$ and then the inner maximization problem over $\bm{\Lambda} \succeq 0$ becomes
		$$\max_{\bm{\Lambda} \succeq 0}   \bigg\{
		2 \underset{s}{\tr}\left(\bm{\Lambda}^{\frac{1}{2}}\right) - \tr( \bm{\Lambda} \bm{X} ) \bigg\}.$$ 
		Suppose $\bm{\Lambda}$ has eigenvalues $0 \leq\beta_1 \le  \cdots \le \beta_d$ and $\bm{\Lambda} = \bm{P} \Diag(\bm{\beta})\bm{P}^{\top}$ with an orthonormal matrix $\bm{P}$. Let us denote $\bm{\theta} = \diag(\bm{P}^{\top} \bm{X} \bm{P})$ and for $\bm{X}$ with rank $r$, let $\bm{X}=\bm{Q} \Diag(\bm{\lambda}) \bm{Q}^{\top}$ denote its eigendecomposition, where $\lambda_1 \ge \cdots \ge \lambda_r > \lambda_{r+1}= \cdots =\lambda_d= 0$ and $\bm{Q}$ is orthonormal.  Following the similar proof of Lemma \ref{lem:ldmax2}, we can reformulate the above maximization problem as
		\begin{align*}
		\max_{\begin{subarray}{c} \bm{P}, \bm{\theta} \in \Re_{+}^d, \bm{\beta}\in \Re_{+}^d,\\
			0\le \beta_1 \le \cdots \le \beta_d,\\
			\theta_{1}\ge \cdots \ge \theta_{d}\ge 0
			\end{subarray}}  \Bigg \{ 2\sum_{i \in [s]}\sqrt{\beta_i}  - \sum_{i \in [d] } \theta_{i} \beta_i :  {\bm{\theta}} = \diag(\bm{P}^{\top} \bm{X} \bm{P}), \bm{P} \textrm{ is orthonormal}   \Bigg\}= \Phi_s(\bm{X}),
		\end{align*}
		with an optimal solution
		\begin{align*}
		\bm{P}^* = \bm{Q}, \bm{\theta}^* = \bm{\lambda},
		\beta_{i}^* = \frac{1}{\lambda_{i}^2}, \forall i \in [k], \beta_{i}^* = \frac{(s-k)^2}{(\sum_{i \in [k+1,d]} \lambda_i)^2}, \forall i \in [k+1,r], \beta_{i}^* \ge \beta_r^*, \forall i \in [r+1,d].
		\end{align*}
		\item For the maximization with respect to $\nu, \bm{\mu} \in \Re_{+}^n$, we have
		\begin{align*}
		\max_{\nu, \bm{\mu} \in \Re_{+}^n } \bigg\{
		- s \nu - \sum_{i \in [n]} {\mu_i} + \sum_{i \in [n]} x_i( \nu + \mu_i )\bigg\}=\begin{cases}
		0,&\textrm{ if }\sum_{i \in [n]}x_i=s, x_i \leq 1,\\
		\infty, &\textrm{ otherwise}.
		\end{cases}
		\end{align*}
	\end{enumerate}
	Combining Parts (i) and (ii), we arrive at \eqref{a_eq_pcont}. \qed
\end{proof}

\subsection{Proof of \Cref{lem:lowbnd}} \label{proof_lem_lowbnd}
\lemlowbnd*

\begin{proof}
	Without loss of generality, suppose that the eigenvalues of matrix $\sum_{i \in [n]} x_i \bm{v}_i \bm{v}_i^{\top}$ are sorted in a descending order, i.e., $\lambda_1\ge \cdots \ge \lambda_d\ge 0$.  Let us construct a new vector $\bm{\beta}$ as
	$$\beta_i = \lambda_i, \forall i \in [k],  \beta_i = \frac{\sum_{i\in [k+1,d]} \lambda_i}{s-k}, \forall i \in [k+1, s], \beta_i = 0, \forall i \in [s+1,d].$$
	
For any two vectors $\bm{x}, \bm{y} \in \Re^d$, we say that $\bm{x}$ is majorized by $\bm{y}$ if
	\begin{align*}
	\sum_{i \in [t]} x_{i} \le \sum_{i \in [t]} y_{i}, \forall t \in [d-1],
	\sum_{i \in [d]} x_{i} = \sum_{i \in [d]} y_{i}.
	\end{align*}	
	Further, a function $f$ is Schur-convex if $f(\bm{x}) \le f(\bm{y})$ holds for any $\bm{x}, \bm{y} \in dom(f)$ that $\bm{x}$ is majorized by $\bm{y}$ (see, e.g., \citealt{hwang1993majorization}). 
	
	Clearly,  $\bm{\lambda}$ is majorized by $\bm{\beta}$ and thus obtain
	\begin{align*}
	\Phi_s\bigg(\sum_{i \in [n]} x_i \bm{v}_i \bm{v}_i^{\top}\bigg) = \frac{E_{s-1}(\bm{\beta})}{E_{s}(\bm{\beta})} \ge  \frac{E_{s-1}(\bm{\lambda})}{E_{s}(\bm{\lambda})} ,
	\end{align*}
	{where the inequality follows from the Schur-convexity of function $\frac{E_{s-1}(\cdot)}{E_{s}(\cdot)} $ (see theorem 3.1 in \citealt{guruswami2012optimal} and the fact $1/(f(\bm x))$ is Schur-convex if $f(\bm x)$ is Schur-concave)}. \qed
\end{proof}

\subsection{Proof of \Cref{them:asamp}} \label{proof_them_asampling}
\themasampling*
\begin{proof}
	For any positive semidefinite matrix $\bm{X} \succeq 0$, let $\bm{\lambda}(\bm{X})$ denote the vector of its eigenvalues. 
	
	The expected objective value output from Algorithm \ref{alg:volsamp} can be upper bounded by
	\begin{align*}
	\mathbb{E} \bigg[ \overset{s}{\tr} \bigg[ \bigg(\sum_{i \in \tilde{S}}\bm{v}_i \bm{v}_i^{\top} \bigg)^{\dag}\bigg] \bigg] &=\sum_{{S}\in \binom{[n]}{s}} \mathbb{P}[\tilde{S}= S]  \overset{s}{\tr} \left[\left(\bm{V}_{S} \bm{V}_{S}^{\top}\right)^{\dag}\right] \\
	&=\sum_{{S}\in \binom{[n]}{s}} \frac{\prod_{i \in S}\hat{x}_i \overset{s}{\det}(\bm{V}_{S} \bm{V}_{S}^{\top})}{\sum_{\bar{S}\in \binom{[n]}{s}} \prod_{i \in \bar{S}} \hat{x}_i \overset{s}{\det}(\bm{V}_{\bar{S}} \bm{V}_{\bar{S}}^{\top})} \frac{E_{s-1}(\bm{\lambda}(\bm{V}_{S} \bm{V}_{S}^{\top}))}{\overset{s}{\det}(\bm{V}_{S} \bm{V}_{S}^{\top})} \\
	&= \frac{\sum_{S\in \binom{[n]}{s}} \prod_{i \in S}\hat{x}_i \sum_{T\in {S\choose s-1}}E_{s-1}(\bm{\lambda}(\bm{V}_{T} \bm{V}_{T}^{\top}))}{\sum_{\bar{S}\in \binom{[n]}{s}} \prod_{i \in \bar{S}} \hat{x}_i \overset{s}{\det}(\bm{V}_{\bar{S}} \bm{V}_{\bar{S}}^{\top})} \\
	&= \frac{\sum_{T\in \binom{[n]}{s-1}}\sum_{S\in \binom{[n]}{s}, T\subseteq S} \prod_{i \in S}\hat{x}_i  E_{s-1}(\bm{\lambda}(\bm{V}_{T} \bm{V}_{T}^{\top}))}{\sum_{\bar{S}\in \binom{[n]}{s}} \prod_{i \in \bar{S}} \hat{x}_i \overset{s}{\det}(\bm{V}_{\bar{S}} \bm{V}_{\bar{S}}^{\top})} \\
	&= \frac{\sum_{T\in \binom{[n]}{s-1}}(\sum_{i\in [n]\setminus T}\hat{x}_i) \prod_{i \in T}\hat{x}_i  E_{s-1}(\bm{\lambda}(\bm{V}_{T} \bm{V}_{T}^{\top}))}{\sum_{\bar{S}\in \binom{[n]}{s}} \prod_{i \in \bar{S}} \hat{x}_i \overset{s}{\det}(\bm{V}_{\bar{S}} \bm{V}_{\bar{S}}^{\top})} \\
	&\le \min(s,n-s+1) \frac{\sum_{T\in \binom{[n]}{s-1}} \prod_{i \in T}\hat{x}_i  E_{s-1}(\bm{\lambda}(\bm{V}_{T} \bm{V}_{T}^{\top}))}{\sum_{\bar{S}\in \binom{[n]}{s}} \prod_{i \in \bar{S}} \hat{x}_i E_s(\bm{\lambda}(\bm{V}_{\bar{S}} \bm{V}_{\bar{S}}^{\top}))}\\
	&= \min(s,n-s+1) \frac{E_{s-1}(\bm{\lambda}(\sum_{i\in [n]}\hat{x}_i\bm{v}_i\bm{v}_i^\top))}{E_{s}(\bm{\lambda}(\sum_{i\in [n]}\hat{x}_i\bm{v}_i\bm{v}_i^\top))}\\
	& \le  \min(s,n-s+1) \Phi_s\bigg(\sum_{i\in [n]}\hat{x}_i\bm{v}_i\bm{v}_i^\top\bigg) \le  \min(s,n-s+1) z_A^*
	\end{align*}
	where the third equality is due to Cauchy-Binet formula \citep{broida1989comprehensive}, the fourth and fifth equalities are due to interchange of summations and collecting terms, the first inequality stems from the fact that $\sum_{i\in [n]\setminus T}\hat{x}_i \leq \min(s,n-s+1) $ for any size-$(s-1)$ subset $T$, the sixth equality is due to Cauchy-Binet formula \citep{broida1989comprehensive}, the second inequality is from Lemma \ref{lem:lowbnd} and the last inequality results from the weak duality. \qed
\end{proof}

\subsection{Proof of \Cref{lem:arankupd}} \label{proof_lem_arankupd}
\lemarankupd*
\begin{proof}
	Similar to the analysis of Lemma \ref{lem:rankupd2}, for each pair $(i,j)$, there are two cases to be considered, conditional on whether $\bm{v}_j \in \col(\bm{X}_{-i})$ or not. If the rank of $\bm{X}$ is $s$, then $\overset{s}{\tr}(\bm{X}^{\dag}) = \tr (\bm{X}^{\dag})$, thus for notational convenience, we use $\tr(\cdot)$ instead.
	\begin{enumerate}[(i)]
		\item If $\bm{v}_j \notin \col(\bm{X}_{-i})$, according to the local optimality condition,  we have
		\begin{align}
		{\tr} (\bm{X}^{\dag}) & \le {\tr}[(\bm{X}_{-i}
		+\bm{v}_j\bm{v}_j^{\top})^{\dag}]=  {\tr} (\bm{X}_{-i}^{\dag})+ \frac{1+\bm{v}_j^{\top} \bm{X}_{-i}^{\dag} \bm{v}_j }{\bm{v}_j^{\top} (\bm{I}_n -\bm{X}_{-i}\bm{X}_{-i}^{\dag}) \bm{v}_j}\notag\\
		&= {\tr} (\bm{X}^{\dag})-\frac{\bm{v}_i^{\top} (\bm{X}^{\dag})^3 \bm{v}_i}{\bm{v}_i^{\top} (\bm{X}^{\dag})^2 \bm{v}_i}+ \frac{1+\bm{v}_j^{\top} \bm{X}_{-i}^{\dag} \bm{v}_j }{\bm{v}_j^{\top} (\bm{I}_n -\bm{X}_{-i}\bm{X}_{-i}^{\dag}) \bm{v}_j}\notag\\
		&= {\tr} (\bm{X}^{\dag})-\frac{\bm{v}_i^{\top} (\bm{X}^{\dag})^3 \bm{v}_i}{\bm{v}_i^{\top} (\bm{X}^{\dag})^2 \bm{v}_i}+ \frac{1+\bm{v}_j^{\top} \bm{X}_{-i}^{\dag} \bm{v}_j }{\bm{v}_j^{\top} (\bm{I}_n -\bm{X}\bm{X}^{\dag}) \bm{v}_j + (\bm{v}_j^{\top} \bm{X}^{\dag} \bm{v}_i)^2/ \bm{v}_i^{\top} (\bm{X}^{\dag})^2 \bm{v}_i},\label{eq_local_opt_A_MESP}
		\end{align}
		where the equalities follow from Part (iii), Part (iv) and Part (viii) in Lemma \ref{lem:rankupd}, respectively.
		
		Then by Part (iv) in Lemma \ref{lem:rankupd}, we further have
		$$\bm{v}_j^{\top} \bm{X}_{-i}^{\dag} \bm{v}_j = \bm{v}_j^{\top} \bm{X}^{\dag} \bm{v}_j - 2 \frac{\bm{v}_j^{\top} \bm{X}^{\dag} \bm{v}_i \bm{v}_j^{\top} (\bm{X}^{\dag})^2 \bm{v}_i}{\bm{v}_i^{\top} (\bm{X}^{\dag})^2 \bm{v}_i} + \frac{\bm{v}_i^{\top} (\bm{X}^{\dag})^3 \bm{v}_i (\bm{v}_j^{\top} \bm{X}^{\dag} \bm{v}_i)^2 }{(\bm{v}_i^{\top} (\bm{X}^{\dag})^2 \bm{v}_i)^2}.$$
		Plugging the equation above into the local optimality condition \eqref{eq_local_opt_A_MESP}, we can simplify it as
		\begin{align*}
		\bm{v}_i^{\top} (\bm{X}^{\dag})^3 \bm{v}_i \bm{v}_j^{\top} (\bm{I}_n- \bm{X}^{\dag}\bm{X})\bm{v}_j 
		\le \bm{v}_i^{\top} (\bm{X}^{\dag})^2 \bm{v}_i + \bm{v}_i^{\top} (\bm{X}^{\dag})^2 \bm{v}_i\bm{v}_j^{\top} \bm{X}^{\dag} \bm{v}_j- 2\bm{v}_i^{\top} (\bm{X}^{\dag})^2 \bm{v}_j\bm{v}_i^{\top} \bm{X}^{\dag} \bm{v}_j.
		\end{align*}
		\item If $\bm{v}_j \in \col(\bm{X}_{-i})$, we show that $\bm{v}_j^{\top} (\bm{I}_n- \bm{X}^{\dag}\bm{X})\bm{v}_j  = 0$ and $\bm{v}_i^{\top} \bm{X}^{\dag} \bm{v}_j  = 0$ for each $i \in \hat{ S}$ in the proof of Lemma \ref{lem:rankupd2}. Thus, it is clear that
		\begin{align*}
		0=\bm{v}_i^{\top} (\bm{X}^{\dag})^3 \bm{v}_i \bm{v}_j^{\top} (\bm{I}_n- \bm{X}^{\dag}\bm{X})\bm{v}_j 
		&\le \bm{v}_i^{\top} (\bm{X}^{\dag})^2 \bm{v}_i + \bm{v}_i^{\top} (\bm{X}^{\dag})^2 \bm{v}_i\bm{v}_j^{\top} \bm{X}^{\dag} \bm{v}_j\\
		&=\bm{v}_i^{\top} (\bm{X}^{\dag})^2 \bm{v}_i + \bm{v}_i^{\top} (\bm{X}^{\dag})^2 \bm{v}_i\bm{v}_j^{\top} \bm{X}^{\dag} \bm{v}_j- 2\bm{v}_i^{\top} (\bm{X}^{\dag})^2 \bm{v}_j\bm{v}_i^{\top} \bm{X}^{\dag} \bm{v}_j.
		\end{align*} \qed
	\end{enumerate}
\end{proof}

\subsection{Proof of \Cref{them:alocs}}\label{proof_them_alocs}
\themalocs*
\begin{subequations}
\begin{proof}
	Let us denote $\bm{X}= \sum_{i \in \hat{S}}\bm{v}_i\bm{v}_i^{\top}$. Clearly, the rank of $\bm{X}$ is $s$ and suppose that its eigenvalues satisfy $\lambda_1 \ge \cdots \ge \lambda_s > \lambda_{s+1}=\cdots = \lambda_d=0 $. Thus, $\overset{s}{\tr}(\bm{X}^{\dag})=\sum_{i \in[s]} \frac{1}{\lambda_i}= \tr(\bm{X}^{\dag})$.  If the rank of an $n\times n$ positive semi-definite matrix $\bm{Y}$ is $s$, since $\overset{s}{\tr}(\bm{Y}) = \tr (\bm{Y})$, thus for notational convenience, we will use $\tr(\cdot)$ instead.
	
	Similar to the proof in Theorem \ref{them1}, our proof relies on the weak duality of A-LD \eqref{a_eq_dual}. Consider a feasible variable $\bm{ \Lambda}$ of A-LD \eqref{a_eq_dual} as 
	\begin{align*}
	\bm{ \Lambda} = 2t^2  (\bm{X}^{\dag})^{2}+ 2t^2\lambda_s^{-2} (\bm{I}_n -\bm{X}^{\dag}\bm{X}),
	\end{align*}
	where $t>0$ is a scaling factor and will be specified later. 	
	Next, to construct the solution of the other two dual variables $({\nu}, {\bm{\mu}})$, we need to check the feasibility of constraints in A-LD \eqref{a_eq_dual}, i.e., 
	\begin{align*}
	\bm{v}_i^{\top} \bm{ \Lambda} \bm{v}_i\le \nu+\mu_i,\forall i \in [n].%
	\end{align*}
	There are two cases to be considered: (i) for each $ i \in \hat{S}$ and (ii) for each  $j \in [n]\setminus \hat{S}$. 
	
	\begin{enumerate}[(i)]  \setlength{\itemsep}{0pt}
		\setlength{\parskip}{0pt}
		\item For each $ i \in \hat{S}$, we have
		\begin{align} \label{a_ineqi}
		\bm{v}_i^{\top} \bm{\Lambda}\bm{v}_i = 2t^2\bm{v}_i^{\top} (\bm{X}^{\dag})^2\bm{v}_i \le 2t^2\tr(\bm{X}^{\dag}),
		\end{align}
		where the equation is due to Part (vi) in Lemma \ref{lem:rankupd} and the inequality is from $\sum_{i \in \hat{S}} \bm{v}_i^{\top} (\bm{X}^{\dag})^2\bm{v}_i = \tr(\bm{X}^{\dag})$.
		\item 
		For each  $j \in [n]\setminus \hat{S}$, according to Lemma \ref{lem:arankupd}, for each $i\in \hat{S}$, we have
		\begin{align*}
		\bm{v}_i^{\top} (\bm{X}^{\dag})^3 \bm{v}_i \bm{v}_j^{\top} (\bm{I}_n- \bm{X}^{\dag}\bm{X})\bm{v}_j 
		\le \bm{v}_i^{\top} (\bm{X}^{\dag})^2 \bm{v}_i + \bm{v}_i^{\top} (\bm{X}^{\dag})^2 \bm{v}_i\bm{v}_j^{\top} \bm{X}^{\dag} \bm{v}_j- 2\bm{v}_i^{\top} (\bm{X}^{\dag})^2 \bm{v}_j\bm{v}_i^{\top} \bm{X}^{\dag} \bm{v}_j .
		\end{align*}
		Summing up the above inequality over all $i\in \hat{S}$, we can obtain
		\begin{align*}
		\frac{1}{t^2}\bm{v}_j^{\top} \bm{\Lambda}\bm{v}_j&\leq \lambda_s^{-2} \bm{v}_j^{\top} (\bm{I}_d+ \bm{X}^{\dag}\bm{X})\bm{v}_j+{\tr}[(\bm{X}^{\dag})^2] \bm{v}_j^{\top} (\bm{I}_d- \bm{X}^{\dag}\bm{X})\bm{v}_j  + 2\bm{v}_j^{\top} (\bm{X}^{\dag})^2 \bm{v}_j
		\\
		&\le  {\tr}(\bm{X}^{\dag}) + {\tr}(\bm{X}^{\dag}) \bm{v}_j^{\top} \bm{X}^{\dag} \bm{v}_j+\lambda_s^{-2} \bm{v}_j^{\top} (\bm{I}_d- \bm{X}^{\dag}\bm{X})\bm{v}_j,
		\end{align*}
		where the first inequality is due to ${\tr}[(\bm{X}^{\dag})^2]\geq \lambda_s^{-2}$. Above, we can further bound the right-hand side as below
		\begin{align*}
		\frac{1}{t^2}\bm{v}_j^{\top} \bm{\Lambda}\bm{v}_j&\le {\tr}(\bm{X}^{\dag}) + {\tr}(\bm{X}^{\dag}) \bm{v}_j^{\top} \bm{X}^{\dag} \bm{v}_j+\lambda_s^{-2} \bm{v}_j^{\top} (\bm{I}_d- \bm{X}^{\dag}\bm{X})\bm{v}_j\\
		&\le {\tr}(\bm{X}^{\dag}) + \lambda_s^{-1}{\tr}(\bm{X}^{\dag}) \bm{v}_j^{\top} \bm{X}^{\dag}\bm{X} \bm{v}_j+\lambda_s^{-1}{\tr}(\bm{X}^{\dag}) \bm{v}_j^{\top} (\bm{I}_d- \bm{X}^{\dag}\bm{X})\bm{v}_j\\
		&={\tr}(\bm{X}^{\dag}) (1+\lambda_s^{-1}\bm{v}_j^{\top}\bm{v}_j) \le {\tr}(\bm{X}^{\dag}) \left(1+ \frac{\lambda_{\max}(\bm{C})}{\delta }\right),
		\end{align*}
		where the second inequality is because ${\tr}[(\bm{X}^{\dag})]\geq \lambda_s^{-1}$ and $\bm{X}^{\dag}\succeq\lambda_s^{-1}\bm{X}^{\dag}\bm{X}$, and the third inequality is due to the facts that $\bm{v}_{\ell}^{\top}\bm{v}_{\ell} \le \lambda_{\max}(\bm{C})$ for any $\ell \in [n]$ and $\lambda_s \ge \delta$.
		
		Thus, for each  $j \in [n]\setminus \hat{S}$, we must have
		\begin{align}
		\bm{v}_j^{\top} \bm{\Lambda}\bm{v}_j \leq t^2{\tr}(\bm{X}^{\dag}) \left(1+ \frac{\lambda_{\max}(\bm{C})}{\delta }\right).  \label{a_ineqj}
		\end{align}
	\end{enumerate}

	Using inequalities \eqref{a_ineqi} and \eqref{a_ineqj} to construct $(\nu, \bm{\mu})$, it suffices to solve the optimization problem
	\begin{align*}
	z_A^{LD}\geq \max_{t>0}\max_{ \nu, \bm{\mu} \in \Re_{+}^n} &\bigg\{ t2\sqrt{2}\overset{s}{\tr}(\bm{X}^{\dag})-s\nu-\sum_{i \in [n]} \mu_i :  \nu+\mu_i \ge 2t^2{\tr}(\bm{X}^{\dag}) , \forall i \in \hat{S} ,\\
	&\nu+\mu_i \ge t^2{\tr}(\bm{X}^{\dag}) \left(1+ \frac{\lambda_{\max}(\bm{C})}{\delta }\right) , \forall i \in [n]\setminus\hat{S} \bigg\}.
	\end{align*}
	Above, by checking the primal and dual of inner maximization problems, there are following two candidate optimal solutions:	
	\begin{align*}
	& \nu^a =t^2{\tr}(\bm{X}^{\dag}) \left(1+ \frac{\lambda_{\max}(\bm{C})}{\delta }\right),  \mu_i^a= 0, \forall i \in [n], \\
	& \nu^b = 2t^2{\tr}(\bm{X}^{\dag}), \mu_i^b = 0 ,\forall i \in \hat{S}, \mu_i^b  = t^2{\tr}(\bm{X}^{\dag})  \left(\frac{\lambda_{\max}(\bm{C})}{ \delta }-1\right) ,\forall i \in [n]\setminus \hat{S}.
	\end{align*}
	Plugging in these two solutions, the above maximization problem becomes
	\begin{align*}
	z_A^{LD}\geq{\tr}(\bm{X}^{\dag}) \max_{t>0} \max\left\{2\sqrt{2}t-s (1+ \frac{\lambda_{\max}(\bm{C})}{\delta })t^2, 2\sqrt{2}t-\left(2s+(n-s)\left(\frac{\lambda_{\max}(\bm{C})}{ \delta }-1\right) \right)t^2  \right\}.
	\end{align*}
	By swapping the two maximization operators and optimizing over $t$, the right-hand side of above inequality is further equivalent to
	\begin{align*}
	z_A^{LD}\geq{\tr}(\bm{X}^{\dag}) \max\left\{ \frac{2}{s(1+\lambda_{\max}(\bm{C})/\delta )}, \frac{2}{n+s+(n-s)\lambda_{\max}(\bm{C})/\delta } \right\}.
	\end{align*}
	Using the fact that $z_A^{LD}\leq z_A^*$, we obtain the desired approximation ratio. \qed
\end{proof}
\end{subequations}

\section{MESP  \eqref{mesp} using Sample Covariance Matrix}\label{app:estimated}
\begin{subequations}
When the true covariance matrix $\bm C$ is not available, we estimate it using $N$ i.i.d. samples, denoted by $\hat{\bm C}_N$. %
Suppose that the random observations are multi-variate sub-Gaussian (see the formal definition in \citealt{vershynin2018high}).
According to theorem 4.7.1 in \cite{vershynin2018high}, we have the following generalization bound between  the sample covariance matrix and the true one
\begin{align} \label{estimate}
\mathbb{E}\left[\| \bm C - \hat{\bm C}_N\|_2 \right] \le c\left(\sqrt{\frac{n}{N}}+\frac{n}{N}\right),
\end{align}
where $c>0$ is a positive constant depending on the data and for a symmetric matrix $\bm X$, we let $\|\bm X\|_2$ denote its largest eigenvalue.
Then, let $\hat{z}_N$ denote the optimal value of MESP \eqref{mesp} using the sample covariance matrix $\hat{\bm C}_N$, i.e.,
\begin{align}
	\hat{z}_N :=  \max_{S} \left\{ \log \det ((\hat{\bm C}_N)_{S,S}):  S \subseteq [n] , |S|=s \right\}. \label{mesp_approx}
\end{align}
Next, we show that with high probability, one has $|z^*-\hat{z}_N| = O(1/\sqrt{N})$.

\begin{proposition}
Suppose that the random observations are multi-variate sub-Gaussian and the positive constants $\zeta,\bar{\zeta}>0$ denote the minimum of all the smallest eigenvalues among all the $s\times s$ positive definite principal submatrices of $\bm C$ and the maximum of all the largest eigenvalues among all the $s\times s$ principal submatrices of $\bm C$, respectively. If the sample size satisfies $N\ge \max\{n, 16c^2n/(\eta^2\zeta^2) (2\bar{\zeta} /\zeta+1)^{2s-2} \}$ with the constants $c>0$ and $\eta \in (0,1)$, then with  probability  at least $1-\eta$, we have
\begin{align*}
|z^*-\hat{z}_N|  \le s\log\left(1+ \frac{4c}{\eta\zeta}\sqrt{\frac{n}{N}}\right) \leq  \frac{4cs}{\eta\zeta}\sqrt{\frac{n}{N}}.
\end{align*} 
\end{proposition}

\begin{proof}
We split the proof into two steps.\par
\noindent\textbf{Step I.} 
Since $N\geq n$, the inequality \eqref{estimate} implies that
\begin{align} \label{estimate2}
	\mathbb{E}\left[\| \bm C - \hat{\bm C}_N\|_2 \right] \le 2c\sqrt{\frac{n}{N}}.
\end{align}
Together with Markov inequality, we have
\begin{align*}
\mathbb{P}\left[\| \bm C - \hat{\bm C}_N\|_2 \le \frac{2c}{\eta}\sqrt{\frac{n}{N}} \right]=1-	\mathbb{P}\left[\| \bm C - \hat{\bm C}_N\|_2 > \frac{2c}{\eta}\sqrt{\frac{n}{N}}\right]\ge 1-  \frac{\eta}{2c}\sqrt{\frac{N}{n}} {\mathbb{E}\left[\| \bm C - \hat{\bm C}_N\|_2 \right]} \ge 1-\eta.
\end{align*}

Thus with probability at least $1-\eta$, the estimated covariance matrix $\hat{\bm C}_N$ satisfies
\begin{align}\label{eq_dist}
- \frac{2c}{\eta}\sqrt{\frac{n}{N}}\bm I_n \preceq \bm C - \hat{\bm C}_N \preceq \frac{2c}{\eta}\sqrt{\frac{n}{N}}\bm I_n.
\end{align}

\noindent\textbf{Step II.} 
Let $S^*$ and $\hat{S}_N$  denote the optimal solutions of MESP \eqref{mesp} and formulation \eqref{mesp_approx}, respectively. 
Since the optimal principal submatrix $\bm C_{S^*,S^*}$ is positive definite and $N\ge 16c^2n/(\eta^2\zeta^2) (2\bar{\zeta} /\zeta+1)^{2s-2}$, using the bounds of $\hat{\bm C}_N$ in \eqref{eq_dist},  $(\hat{\bm C}_N)_{S^*,S^*}$ is also positive definite with the smallest eigenvalue at least $\zeta/2$ (see the inequalities \eqref{min_eign} below).
Thus, both the optimal principal submatrices $\bm C_{S^*,S^*}$ and $(\hat{\bm C}_N)_{\hat{S}_N,\hat{S}_N}$ must be positive definite and their corresponding optimal values $z^*$ and  $\hat{z}_N$ satisfy
\begin{align*}
z^*- \hat{z}_N &= \log\det(\bm C_{S^*,S^*}) -  \log\det((\hat{\bm C}_N)_{\hat{S}_N,\hat{S}_N}) \\
&=  \log\det(\bm C_{S^*,S^*}) - \log\det((\hat{\bm C}_N)_{S^*,S^*})+ \log\det((\hat{\bm C}_N)_{S^*,S^*})-  \log\det((\hat{\bm C}_N)_{\hat{S}_N,\hat{S}_N}) \\
& \le \log\det(\bm C_{S^*,S^*}) - \log\det((\hat{\bm C}_N)_{S^*,S^*})\\
& \le \log\det\left((\hat{\bm C}_N)_{S^*,S^*} +\frac{2c}{\eta}\sqrt{\frac{n}{N}}\bm I_s\right)- \log\det((\hat{\bm C}_N)_{S^*,S^*}) = \log\det\left(\bm I_s +\frac{2c}{\eta}\sqrt{\frac{n}{N}}(\hat{\bm C}_N)^{-1}_{S^*,S^*}\right)\\
& \le s \log \left(1+ \frac{2c}{\eta}\sqrt{\frac{n}{N}} \frac{1 }{ \lambda_{\min}((\hat{\bm C}_N)_{S^*,S^*})} \right) \le  s \log \left( 1+\frac{4c}{\eta\zeta}\sqrt{\frac{n}{N}}\right)  ,
\end{align*} 
where the first inequality is because of the optimality of $(\hat{\bm C}_N)_{\hat{S}_N,\hat{S}_N}$, the second one is due to inequalities \eqref{eq_dist}, the third one is from the monotonicity of function $\log(\cdot)$, and the last one is due to the fact that if  $N\ge 16c^2n/(\eta^2\zeta^2) (2\bar{\zeta} /\zeta+1)^{2s-2} \geq 16c^2n/(\eta^2\zeta^2)$, we must have
\begin{align} \label{min_eign}
\lambda_{\min}((\hat{\bm C}_N)_{S^*,S^*}) \ge \lambda_{\min}(\bm C_{S^*, S^*}) - \frac{2c}{\eta}\sqrt{\frac{n}{N}}\ge \zeta - \frac{2c}{\eta}\sqrt{\frac{n}{N}} \ge \frac{\zeta}{2}.
\end{align}

Similarly, since $N\ge 16c^2n/(\eta^2\zeta^2) (2\bar{\zeta} /\zeta+1)^{2s-2}$, the matrix $\bm C_{\hat{S}_N,\hat{S}_N}$ must also be positive definite and thus, its smallest eigenvalue must be at least $\zeta$.
Therefore, we can also show that
\begin{align*}
z^* - \hat{z}_N  &= \log\det(\bm C_{S^*,S^*}) -  \log\det((\hat{\bm C}_N)_{\hat{S}_N,\hat{S}_N}) \\
&=  \log\det(\bm C_{S^*,S^*}) - \log\det(\bm C_{\hat{S}_N,\hat{S}_N})  + \log\det(\bm C_{\hat{S}_N,\hat{S}_N})  -  \log\det((\hat{\bm C}_N)_{\hat{S}_N,\hat{S}_N}) \\
& \ge\log\det(\bm C_{\hat{S}_N,\hat{S}_N})  -  \log\det((\hat{\bm C}_N)_{\hat{S}_N,\hat{S}_N})\\
& \ge \log\det(\bm C_{\hat{S}_N,\hat{S}_N})  -  \log\det\left(\bm C_{\hat{S}_N,\hat{S}_N} + \frac{2c}{\eta}\sqrt{\frac{n}{N}}\bm I_s\right)
= - \log\det\left(\bm I_s +\frac{2c}{\eta}\sqrt{\frac{n}{N}}\bm C^{-1}_{\hat{S}_N,\hat{S}_N}\right)\\
& \ge -s \log \left( 1+ \frac{2c}{\eta}\sqrt{\frac{n}{N}}\frac{1}{\lambda_{\min}(\bm C_{\hat{S}_N,\hat{S}_N})} \right) \ge- s\log\left( 1+\frac{4c}{\eta\zeta}\sqrt{\frac{n}{N}}\right). 
\end{align*} 
All the results in Step II hold almost surely conditioning on that $\| \bm C - \hat{\bm C}_N\|_2 \le 2c/\eta\sqrt{n/N}$, where the latter occurs with probability at least $1-\eta$. This completes the proof.
\qed
\end{proof}
\end{subequations}

\section{MISOCP Formulation of MESP} \label{misocp}
\begin{subequations}
In this section, we develop a mixed integer second-order conic programming (MISOCP) formulation for MESP, which is equivalent to the nonlinear convex integer program studied by \cite{anstreicher2020efficient}. The formulation from \cite{anstreicher2020efficient} has the following form
\begin{align}
		\textrm{(MESP)} \  z^{*} := \max_{\bm{x}} \Bigg\{ \frac{1}{2} \log \det \left(\gamma \bm{C} \Diag(\bm{x}) \bm{C} + \bm{I}_n-\Diag(\bm{x}) \right)-\frac{1}{2} s\log(\gamma):  \sum_{i\in [n]} x_i =s,
\bm{x} \in [0,1]^n \Bigg \},\label{eq_linx_bound}
\end{align}
where $\gamma$ is a positive scalar and can be arbitrary. In fact, a good choice of $\gamma$ can improve the continuous relaxation of formulation \eqref{eq_linx_bound}. According to table 1 in \cite{sagnol2015computing}, we can show that the above formulation \eqref{eq_linx_bound} is equivalent to
\begin{equation}\label{eq_linx_bound2}
\begin{aligned}
\textrm{(MESP)} \quad z^*:= \max_{\bm{x} ,\bm{Z}_1,\cdots, \bm{Z}_{2n} \bm{t}, \bm{J}} \ \ & \frac{n}{2}\log \bigg(\prod_{j=1}^{n} (\bm{J}_{j,j})^{1/n} \bigg)  - \frac{1}{2}s \log(\gamma)\\
\text{s.t.  } &\sum_{i\in [n]} \sqrt{\gamma} \bm{C}_i \bm{Z}_i + \sum_{k\in [n+1, 2n]} \bm{e}_{k-n} \bm{Z}_k = \bm{J},  \bm{J} \text{ is lower triangular,}\\
&\|\bm{Z}_i \bm{e}_j \|^{2} \le \bm{t}_{ij}x_i, \forall i\in [2n], \forall j\in [n],\\
&\sum_{i\in [2n]} \bm{t}_{ij} \le \bm{J}_{j,j}, \forall j\in [n],\\
&\bm{t}_{i,j} \ge 0, \forall i\in [2n], \forall j\in[n],\\
&1-x_i=x_{n+i}, \forall i\in [n],\\
&\sum_{i\in [n]} x_i=s,\\
&\bm{x} \in \{0,1\}^{n},
\end{aligned}
\end{equation}
where $\bm{C}_i$ denotes the $i$-th column vector of matrix $\bm{C}$. Note that according to chapter 2.3 in \cite{ben2001lectures}, we can equivalently represent the objective function in the formulation \eqref{eq_linx_bound2} as a second order conic program. 
\end{subequations}
\end{appendices}

\end{document}